\relax
\documentclass[letterpaper]{article} % DO NOT CHANGE THIS
\usepackage{aaai20}  % DO NOT CHANGE THIS
\usepackage{times}  % DO NOT CHANGE THIS
\usepackage{helvet} % DO NOT CHANGE THIS
\usepackage{courier}  % DO NOT CHANGE THIS
\usepackage[hyphens]{url}  % DO NOT CHANGE THIS
\usepackage{graphicx} % DO NOT CHANGE THIS
\usepackage{graphicx}  % DO NOT CHANGE THIS

\usepackage{algorithm}
\usepackage{algorithmicx}
\usepackage{algpseudocode}
\usepackage{amsfonts}
\usepackage{amsmath}
\usepackage{amsthm}
\usepackage{booktabs}
\usepackage{subfigure}
\usepackage{float}
\usepackage{color}
\newtheorem{Def}{Definition}
\newtheorem{The}{Theorem}
\newtheorem{Rem}{Remark}
\newtheorem*{DP}{Differential Privacy}

\title{Weighted Distributed Differential Privacy ERM: Convex and Non-convex}
\author{\Large \textbf{Yilin Kang, Yong Liu, Weiping Wang}\\ % All authors must be in the same font size and format. Use \Large and \textbf to achieve this result when breaking a line
}

\begin{document}

\maketitle

\begin{abstract}
Distributed machine learning is an approach allowing different parties to learn a model over all data sets without disclosing their own data.
In this paper, we propose a weighted distributed differential privacy (WD-DP) empirical risk minimization (ERM) method to train a model in distributed setting, considering different weights of different clients.
We guarantee differential privacy by gradient perturbation, adding Gaussian noise, and advance the state-of-the-art on gradient perturbation method in distributed setting.
By detailed theoretical analysis, we show that in distributed setting, the noise bound and the excess empirical risk bound can be improved by considering different weights held by multiple parties.
Moreover, considering that the constraint of convex loss function in ERM is not easy to achieve in some situations, we generalize our method to non-convex loss functions which satisfy Polyak-Lojasiewicz condition.
Experiments on real data sets show that our method is more reliable and we improve the performance of distributed differential privacy ERM, especially in the case that data scale on different clients is uneven.
\end{abstract}

\section{Introduction}
\noindent In recent years, machine learning has been widely used in many fields such as data mining and pattern recognition \cite{1,24,25,23}.
Because of the need of data for training machine learning algorithms, tremendous data is collected by individuals and companies.
As a result, sensitive information disclosure is becoming a huge problem.
In addition to data itself, model parameters trained by data can reveal sensitive information in an undirect way as well.

To solve the problems mentioned above, differential privacy \cite{3} is proposed to preserve privacy in machine learning and has been applied to principal component analysis (PCA) \cite{19,28,35}, regression \cite{4,20,8}, boosting \cite{18,29}, deep learning \cite{5,6,27} and other fields.

There are mainly three methods to achieve differential privacy: output perturbation \cite{9,7,16,13}, objective perturbation \cite{4,14} and gradient perturbation \cite{16,6,15}.
Among them, gradient perturbation is the most popular method because it can be applied to any gradient descent method, which makes it general, and it not only protects the results but also the gradients, which makes it more reliable.

With the development of organizations' corporation, multiple parties' desire to train models by combining data is becoming stronger, such as biomedicine and financial fraud detection.
In these situations, different parties want to use all the data without disclosing their own data, which brings more press on privacy preserving.
Moreover, the number of data instances owned by different parties always varies greatly, making the performance degrade significantly.

\begin{table*}[t]
\centering
\caption{Comparison between our method and other methods on noise bound and excess empirical risk bound}\smallskip
\resizebox{0.95\textwidth}{!}{
\begin{tabular}{c|c|c|c|c}
\toprule
 & Gaussian Noise Bound & Excess Empirical Risk Bound & Distributed & Non-convex\\
\midrule
\citeauthor{7} \shortcite{7} & None & $O\left(\frac{(m-1)^2(\lambda+1)}{n_{(1)}^2\lambda^2}+\frac{p^2(\lambda+1)}{n_{(1)}^2\epsilon^2\lambda^2}\log^2(\frac{p}{\delta})+
\frac{p(m-1)(\lambda+1)}{n_{(1)}^2\epsilon\lambda^2}\log(\frac{p}{\delta})\right)$ & Yes & No \\
\citeauthor{8} \shortcite{8} & $O\left(\frac{G^2T\log(1/\delta)}{m^2n_{(1)}^2\epsilon^2}\right)$ & $O\left(\frac{pG^2L\log^2(mn_{(1)})\log(1/\delta)}{m^2n_{(1)}^2\lambda^2\epsilon^2}\right)$ & Yes & No \\
\citeauthor{13} \shortcite{13} & $O\left(\frac{pG^2\log(n/\delta)\log(1/\delta)+G^2\epsilon^2}{\epsilon^2}\right)$ & $O\left(\frac{G\sqrt{p\log(n/\delta)\log(1/\delta)}D}{n\epsilon}\right)$ & No & Yes \\
\citeauthor{11} \shortcite{11} & $O\left(\frac{G^2T\ln(1/\delta)}{n^2\epsilon^2}\right)$ & $O\left(\frac{pG^2\log^2(n)\ln(1/\delta)}{n^2\epsilon^2}\right)$ & No & Yes \\
Our Method WD-DP & $O\left(\frac{G^2T\ln(1/\delta)}{{\color{blue}\boldsymbol{n^2}}\epsilon^2}\right)$ & $O\left(\frac{pG^2{\color{blue}\boldsymbol{\log(n)}}\ln(1/\delta)}{{\color{blue}\boldsymbol{n^2}}\epsilon^2}\right)$ & Yes & Yes \\
\bottomrule
\end{tabular}
}
\label{table1}
\end{table*}

Distributed machine learning is an approach to solve multiple-party learning problem.
Among many distributed learning strategies, \textit{divide and conquer} is simple and effective.
It preserves privacy by minimizing information communications, which has caused widespread concern of researchers.
\citeauthor{7} \shortcite{7} proposed the first distributed differential privacy machine learning method, whose privacy is preserved by output perturbation.
\citeauthor{8} \shortcite{8} introduced differential privacy distributed methods using output perturbation and gradient perturbation, achieving better performance.
\citeauthor{36} \shortcite{36} proposed federated learning method to address distributed machine learning problem, without privacy preserving.
Based on \cite{36}, \citeauthor{15} \shortcite{15} proposed a method to guarantee that whether a client participants in federated learning cannot be inferred, preserving privacy on federated learning to some extent.
\citeauthor{37} \shortcite{37} proposed a user-level differential privacy method on training LSTM, applied on language models.
However, among the work mentioned above, only federated learning based methods consider about weights of parties when aggregating parameters.
Considering that in real scenarios, data scale on different clients is always uneven, simply averaging without weights leads worse performance.
Moreover, most work assumes that loss function is convex in theoretical analysis, but not considers about the non-convex condition.

To address the problems above, in this paper, we propose Weighted Distributed Differential Privacy (WD-DP) ERM method based on divide and conquer distributed method, applying gradient perturbation by adding Gaussian noise to guarantee $(\epsilon,\delta)$-differential privacy.
We consider about different weights owned by different parties instead of simply averaging when aggregating models' parameters to reduce the negative impact caused by uneven data scale, which leads a better noise bound and excess empirical risk bound theoretically.
Experiments on real data sets show that the performance of our method is much better than the method proposed in \cite{8}, the best method in distributed differential privacy ERM we know.
Moreover, considering the fact that most previous theoretical analysis on differential privacy ERM is based on convex functions and this constraint is not easy to guarantee in some situations, first, we improve the proof process of the excess empirical risk bound proposed by \citeauthor{11} \shortcite{11} in centralized setting and then generalize our method to non-convex functions which satisfy the Polyak-Lojasiewicz condition in distributed setting.

The rest of the paper is organized as follows.
We introduce some related work on distributed differential privacy ERM methods and centralized differential privacy ERM under non-convex condition in Section 2.
We propose our method WD-DP in detail and then analyze the $(\epsilon,\delta)$-differential privacy of our method in Section 3.
We give the theoretical analysis of the excess empirical risk bound of our method on both convex and non-convex conditions in Section 4.
We present the experimental results in Section 5.
Finally, we conclude the paper in Section 6.

\section{Related Work}
\noindent In this section, we first introduce some related work over distributed differential privacy machine learning.
Then, we introduce some work on centralized differential privacy ERM under non-convex condition.

\subsection{Distributed Setting}
\citeauthor{7} \shortcite{7} proposed a distributed privacy preserving protocol whose objective function has a regularization term $\lambda N(\theta)$, where $\theta$ is the model with $p$ parameters.
Different parties train models locally and interact with curator to construct additive shares of a perturbed aggregated model.
This work guarantees differential privacy by output perturbation, adding Laplace noise.
In this work, parameters' delivery relies on homomorphic encryption \cite{2}, which is expensive on computation.

\citeauthor{8} \shortcite{8} introduced a distributed learning method, combining privacy with secure multi-party computation (SMC) \cite{17}.
This work guarantees differential privacy by output perturbation and gradient perturbation, adding noise within a SMC.
The noise bound and excess empirical risk bound are better than in \cite{7} by assuming the loss function $\ell(\cdot)$ is $G$-Lipschitz and $L$-smooth.
Particularly, in this method, parties aggregate parameters by simply averaging.
As a result, if the number of data instances on parties is not even, the performance will decrease rapidly.
Unfortunately, in real scenarios, data scale on clients is always uneven.

\citeauthor{36} \shortcite{36} proposed a decentralized method to solve the problem of distributed machine learning, Federated Learning, and applied it on deep networks.
This method leaves training data distributed on different parties and learns a shared model by aggregating local models.
This work considers about different weights of different parties when aggregating models, but does not consider much about privacy preserving, without theoretical analysis on privacy or utility.

In the method WD-DP, proposed by this paper, by considering about different weights held by different parties when aggregating parameters, we achieve better performance both theoretically and practically, no matter data scale is even or not.
So, our method is more general and adapt to most scenarios.
The comparison between our method and other methods mentioned above on noise bound and excess empirical risk bound is given in Table 1.

It can be observed in Table 1 that without considering about regularization term, noise bound and excess empirical risk bound of our method are better than the best distributed method we have known, proposed by \citeauthor{8} \shortcite{8}, by a factor of $\frac{(mn_{(1)})^2}{n^2}$ and $\frac{(mn_{(1)})^2\log(n)}{(\log(mn_{(1)})n)^2}$, respectively, where $m$ is the number of parties, $n_{(1)}$ denotes the smallest size of data set owned by parties and $n$ represents the total number of data instances over all data sets.
Particularly, our method is much better than the method proposed by \citeauthor{8} \shortcite{8} when data scale is uneven on clients and remains the same performance under even data scale.
In other words, the method mentioned above is a special case of our method WD-DP under average setting.
And obviously, the excess empirical risk bound of our method is much tighter than which in \cite{7}.
It is worth emphasizing that although our method is proposed under distributed setting, it achieves almost the same theoretical performance as centralized methods.

\subsection{Non-convex ERM}
\citeauthor{13} \shortcite{13} proposed Random Round Private SGD, guaranteeing $(\epsilon,\delta)$-differential privacy over non-convex function.
It is the first theoretical result on centralized non-convex differentially private ERM problem.
In this method, the excess empirical risk bound is proportional to $D$, the upper bound of the $\ell_2$ norm of the model's parameters.

\citeauthor{11} \shortcite{11} gave theoretical analyses on noise bound and excess empirical risk bound of gradient perturbation under non-convex condition in centralized setting, assuming the iteration number is $T$.
However, the proof process on the excess empirical risk bound of this method can be better, leading a tighter excess empirical risk bound.

\citeauthor{21} \shortcite{21} studied the centralized differential privacy ERM problem with non-convex loss functions and gave upper bounds for the utility.
This work considers the problem in both low and high dimensional space and shows that for some special non-convex loss functions, the utility can be improved to a level similar to convex ones.

In this paper, first, we improve the proof process on excess empirical risk bound in \cite{11}.
Then, considering there is no previous theoretical analysis over distributed non-convex differential private ERM, we extend this method to distributed setting in which loss function is not constrained convex.
The comparison between our method and these centralized methods under non-convex condition on noise bound and excess empirical risk bound is given in Table 1.

It can be observed in Table 1 that by improving the proof process proposed by \citeauthor{11} \shortcite{11}, our excess empirical risk bound is tighter than before by a factor of $\log(n)$.
Meanwhile, considering the parameter $D$ is hard to control, our method is more reliable than which proposed by \citeauthor{13} \shortcite{13}, with a tighter noise bound.

\section{WD-DP: Weighted Distributed Differential Privacy Empirical Risk Minimization}
In this section, we first introduce some basic definitions and the Empirical Risk Minimization in distributed setting.
Then, we propose our method WD-DP in detail and give theoretical analysis of $(\epsilon,\delta)$-DP over our algorithm.

Given $d$-dimensional vector $\mathbf{x}$=$[x_1,x_2,...,x_d]^\top$, denote its $\ell_2$-norm as $\left\| \mathbf{x}  \right\|$=$(\sum_{i=1}^{d} \vert x_i \vert^2)^{\frac{1}{2}}$.
$\tilde{O}(\cdot)$ is similar to $O(\cdot)$, but hiding factors polynomial in $\log n$ and $\log(1/\delta)$.
Denote the probability distribution of data as $\mathcal{D}^n$, for two databases $D,D' \in \mathcal{D}^n$ differing by one single element, they are denoted as $D \sim D'$, called \textit{adjacent databases}.

\begin{Def}
\cite{9} A randomized function $\mathcal{A} : \mathcal{D}^n \rightarrow \mathbb{R}^p$ is $(\epsilon,\delta)$-differential privacy ($(\epsilon,\delta)$-DP) if
\begin{equation*}
\mathbb{P}\left[\mathcal{A}(D) \in S\right] \leq e^\epsilon \mathbb{P}[\mathcal{A}(D') \in S] + \delta,
\end{equation*}
where $S \in$ range($\mathcal{A}$) and $p$ is the number of parameters.
\end{Def}
According to the definition, differential privacy requires that data sets $D,D'$ lead to similar distributions on the output of a randomized algorithm $\mathcal{A}$.
This implies that an adversary will draw essentially the same conclusions about an individual whether or not that individual’s data was used even if many records are known a priori to the adversary.

The centralized ERM objective function is defined as:
\begin{equation*}
L_D(\theta)=\frac{1}{n}\sum_{i=1}^{n} \ell(\theta,x_i,y_i),
\end{equation*}
where $(x_i,y_i)$ denotes data instance, $\ell$ is the loss function.

\subsection{Distributed Differential Privacy}
Suppose there are $m$ parties $P_1,P_2,...,P_m$, owning data sets $D_1,D_2,...,D_m$ with size $n_1,n_2,...,n_m$ respectively.
Parties train their own model $\theta^{(1)},\theta^{(2)},...,\theta^{(m)}$ locally to prevent data disclosing (in this paper, we denote model by parameters), and then their models are aggregated by a trusted third party (called server).

So, in distributed setting, considering all the parties, the objective function is:
\begin{equation}
L_D(\theta)=\frac{1}{m}\sum_{j=1}^{m}\frac{1}{n_j}\sum_{i=1}^{n_j}\ell(\theta,x_i^{(j)},y_i^{(j)}),
\end{equation}
where party $j$'s data instances are denoted as $(x_i^{(j)},y_i^{(j)})$.

By equation (1), when it comes to gradient perturbation, considering about round $t$, with learning rate $\eta$, we have the updating criteria on party $j$:
\begin{equation*}
\theta_{t+1}^{(j)}=\theta_t^{(j)}-\eta(\nabla L_{D_j}(\theta_t^{(j)})+z_t),
\end{equation*}
and the updating criteria on server after $T$ local iterations is:
\begin{equation}
\theta^{(c)}=\frac{1}{m}\sum_{j=1}^{m}\theta_{T}^{(j)},
\end{equation}
where $L_{D_j}(\theta)$ represents the objective function over party $j$, $z_t\sim\mathcal{N}(0,\sigma^2I_p)$ is Gaussian noise guaranteeing differential privacy and $\theta^{(c)}$ denotes the aggregated model on server.

\subsection{Weighted Distributed Differential Privacy}
Traditional methods use equation (2) to aggregate parameters by simply averaging.
However, this method pays more attention on data instances in small data sets, which leads worse noise bound and excess empirical risk bound.
Considering that data scale on clients in real scenarios is always uneven, simply averaging leads worse performance.

So, to solve the problem mentioned above, instead of simply averaging the parameters when aggregating models, we consider the weights of different parties related to their data sets' size, which leads updating criteria on the server to:
\begin{equation*}
\theta^{(c)}=\sum_{j=1}^{m}\frac{n_j}{n}\theta_{T}^{(j)}.
\end{equation*}

When considering about weights of different parties, data instances in different parties are paid same attention, which reduces the negative impact caused by a single bad data instance, rare but special high noise generated for guaranteeing differential privacy or uneven data scale.

Our method is detailed in Algorithm 1.
Note that in Algorithm 1, we assume that the size of data sets $n_j$ are public knowledge, like in \cite{36}.

\begin{algorithm}
    \caption{Weighted Distributed Differential Privacy ERM Method: WD-DP}
    \begin{algorithmic}[1]
        \Require $m$ parties indexed by $j$, number of local iteration rounds $T$, learning rate $\eta$
        \Function {DistributedLearning}{$m,T,\eta$}
        \State First, $m$ parties download random $\theta^{(c)}$ from the server as initialization.
        \State \textbf{Party $j$ ($j \in \{1,2,...,m\}$) executes at round $t$}:
        \State \quad $\theta_{t+1}^{(j)}=\theta_t^{(j)}-\eta(\nabla L_{D_j}(\theta_t^{(j)})+z_t)$,
        \State \quad where $z_t \sim \mathcal{N}(0,\sigma^2I_p)$.
        \State \textbf{Server executes after $T$ local rounds}:
        \State \quad $\theta^{(c)}=\sum_{j=1}^{m}\frac{n_j}{n}\theta_{T}^{(j)}$,
        \State \quad where $n=\sum_{i=1}^{m}n_i$ is the total number of data.
        \State return $\theta^{(c)}$.
        \EndFunction
    \end{algorithmic}
\end{algorithm}

\begin{DP}
\rm{
In this paper, we guarantee $(\epsilon,\delta)$-DP using Gaussian Mechanism proposed by \citeauthor{9} \shortcite{9} and moments accountant introduced by \citeauthor{6} \shortcite{6}.
}
\end{DP}

\begin{The}
In \textbf{Algorithm 1}, for $\epsilon,\delta \ge 0$, if $\ell(\theta,x,y)$ is G-Lipschitz over $\theta$ and
\begin{equation}
\sigma^2=c\frac{G^2T\ln(1/\delta)}{n^2\epsilon^2},
\end{equation}
it is $(\epsilon, \delta)$-DP for some constant $c$.
\end{The}

\begin{proof}
Consider the $t^{th}$ query which may disclose privacy:
\begin{equation}
\begin{aligned}
M_t&=\sum_{j=1}^{m}\frac{n_j}{n}\left[\frac{1}{n_j}\sum_{i=1}^{n_j}\nabla \ell(\theta_t^{(c)},x_i^{(j)},y_i^{(j)})+\mathcal{N}(0,\sigma^2I_p)\right] \\
&=\frac{1}{n}\sum_{i=1}^{n}\nabla \ell(\theta_t^{(c)},x_i,y_i)+\mathcal{N}(0,\sigma^2I_p),
\end{aligned}
\end{equation}
where $\theta_t^{(c)}$ represents $\theta^{(c)}$ after $t$ local rounds.

In moments accountant method proposed by \citeauthor{6} \shortcite{6}, the $\lambda^{th}$ ($\lambda \geq 1$) moment $\alpha_M(\lambda;D,D')$ on mechanism $M$ is defined as:
\begin{equation}
\alpha_M(\lambda;D,D')=\log\mathbb{E}_{o\sim M(D)}\left[\exp(\lambda c(o;M,D,D'))\right],
\end{equation}
where $c(o;M,D,D')$ is privacy loss at output $o$, defined as:
\begin{equation}
c(o;M,D,D')=\log\frac{\mathbb{P}\left[M(D)=o\right]}{\mathbb{P}\left[M(D')=o\right]}.
\end{equation}

In order to preserve privacy, it is necessary to bound all possible $\alpha_M(\lambda;D,D')$. So, $\alpha_M(\lambda)$ is defined as:
\begin{equation*}
\alpha_M(\lambda)=\max_{D,D'}\alpha_M(\lambda;D,D').
\end{equation*}

Denote probability distributions on adjacent databases $D$ and $D'$ over mechanism $M_t$ as $P$ and $Q$:
\begin{gather*}
P=\nabla L_D(\theta_t^{(c)})+\mathcal{N}(0,\sigma^2I_p)=\mathcal{N}(\nabla L_D(\theta_t^{(c)}),\sigma^2I_p), \\
Q=\nabla L_{D'}(\theta_t^{(c)})+\mathcal{N}(0,\sigma^2I_p)=\mathcal{N}(\nabla L_{D'}(\theta_t^{(c)}),\sigma^2I_p).
\end{gather*}

By Definition 2.1 in \cite{12}, define $D_\alpha$ as:
\begin{equation}
D_{\alpha}(P\Vert Q)=\frac{1}{\alpha-1}\log\left(\mathbb{E}_{x\sim P}\left[\left(\frac{P(x)}{Q
(x)}\right)^{\alpha-1}\right]\right).
\end{equation}

By equations (5), (6), (7) and definitions $P$, $Q$, we have equations below over mechanism $M_t$:
\begin{equation*}
\begin{aligned}
\alpha_{M_t}(\lambda)&=\log\mathbb{E}_{o\sim P}\left[\exp\left(\lambda\log(\frac{P}{Q})\right)\right] \\
&=\log\mathbb{E}_{o\sim P}\left[\left(\frac{P}{Q}\right)^\lambda\right] \\
&=\lambda D_{\lambda+1}(P\Vert Q).
\end{aligned}
\end{equation*}

By Lemma 2.5 in \cite{12}, we have:
\begin{equation*}
\lambda D_{\lambda+1}(P \Vert Q)=\frac{\lambda(\lambda+1)\left\|\nabla L_D(\theta_t^{(c)})-\nabla L_{D'}(\theta_t^{(c)}) \right\|^2}{2\sigma^2}.
\end{equation*}

Note that $\ell$ is $G$-Lipschitz (denoted as $G$ below), and there is only one single element different between $D$ and $D'$, suppose it is the $n^{th}$ one, we have:
\begin{equation*}
\begin{aligned}
&\nabla L_D(\theta_t^{(c)})-\nabla L_{D'}(\theta_t^{(c)}) \\
&=\frac{1}{n}\left(\sum_{i=1}^{n-1}\nabla\ell(\theta_t^{(c)},x_i,y_i)+\nabla\ell(\theta_t^{(c)},x_n,y_n)\right) \\
&\quad-\frac{1}{n}\left(\sum_{i=1}^{n-1}\nabla\ell(\theta_t^{(c)},x_i,y_i)+\nabla\ell(\theta_t^{(c)},x_n',y_n')\right) \\
&=\frac{1}{n}\left(\nabla\ell(\theta_t^{(c)},x_n,y_n)-\nabla\ell(\theta_t^{(c)},x_n',y_n')\right) \\
&\overset{(G)}{\leq}\frac{2G}{n}.
\end{aligned}
\end{equation*}

Thus,
\begin{equation*}
\alpha_{M_t}(\lambda)=\lambda D_{\lambda+1}(P \Vert Q) \leq \frac{2G^2\lambda(\lambda+1)}{\sigma^2n^2}.
\end{equation*}

By Theorem 2.1 in \cite{6}, we have:
\begin{equation*}
\alpha_M(\lambda) \leq \sum_{t=1}^{T}\alpha_{M_t}(\lambda).
\end{equation*}

Then, note that $\lambda \geq 1$, we have:
\begin{equation*}
\alpha_M(\lambda) \leq \sum_{t=1}^{T}\alpha_{M_t}(\lambda)=2\lambda(\lambda+1)\frac{G^2T}{\sigma^2n^2} \leq 4\lambda^2\frac{G^2T}{\sigma^2n^2}.
\end{equation*}

Taking $\sigma^2=c\frac{G^2T\ln(1/\delta)}{n^2\epsilon^2}$ for some constant $c$, we can guarantee that:
\begin{equation*}
\alpha_M(\lambda) \leq 4\lambda^2\frac{G^2T}{\sigma^2n^2} \leq \frac{\lambda\epsilon}{2},
\end{equation*}
and as a result, we have:
\begin{equation*}
\frac{4\lambda^2G^2T\epsilon^2}{cG^2T\ln(1/\delta)} \leq \frac{\lambda\epsilon}{2},
\end{equation*}
which means:
\begin{equation*}
\delta \leq \exp(\frac{-\lambda\epsilon}{2}),
\end{equation*}
means $(\epsilon,\delta)$-DP due to Theorem 2.2 in \cite{6}.
\end{proof}

In Theorem 1, $\ell$ is $G$-Lipschitz, but not constrained \textit{convex}.
Thus, Theorem 1 is general in both convex and non-convex conditions.

Although Algorithm 1 considers distributed setting, equation (4) is not related to the number of parties $m$.
As a result, Gaussian noise guaranteeing $(\epsilon,\delta)$-DP is not related to multiple parties, but has the same form as in centralized setting.

Moreover, we consider \textit{moments accountant} method and different weights held by parties when aggregating models, so the bound is tighter than which introduced by \citeauthor{8} \shortcite{8} by a factor of $\frac{(mn_{(1)})^2}{n^2}$, where $n_{(1)}$ is the smallest size of data sets owned by all the parties. When data scale on clients is not even, our method is much better.

\section{Theoretical Analysis over Convex and Non-convex Conditions}
In this section, first we give the analysis of excess empirical risk of our method WD-DP under convex condition and then generalize it to non-convex functions which satisfy the Polyak-Lojasiewicz condition.
To our knowledge, this is the first theoretical analysis of excess empirical risk bound on non-convex distributed differential privacy ERM.

\begin{figure*}[htb]
\centering
\subfigure[KDDCup99]{\includegraphics[width=0.3\textwidth]{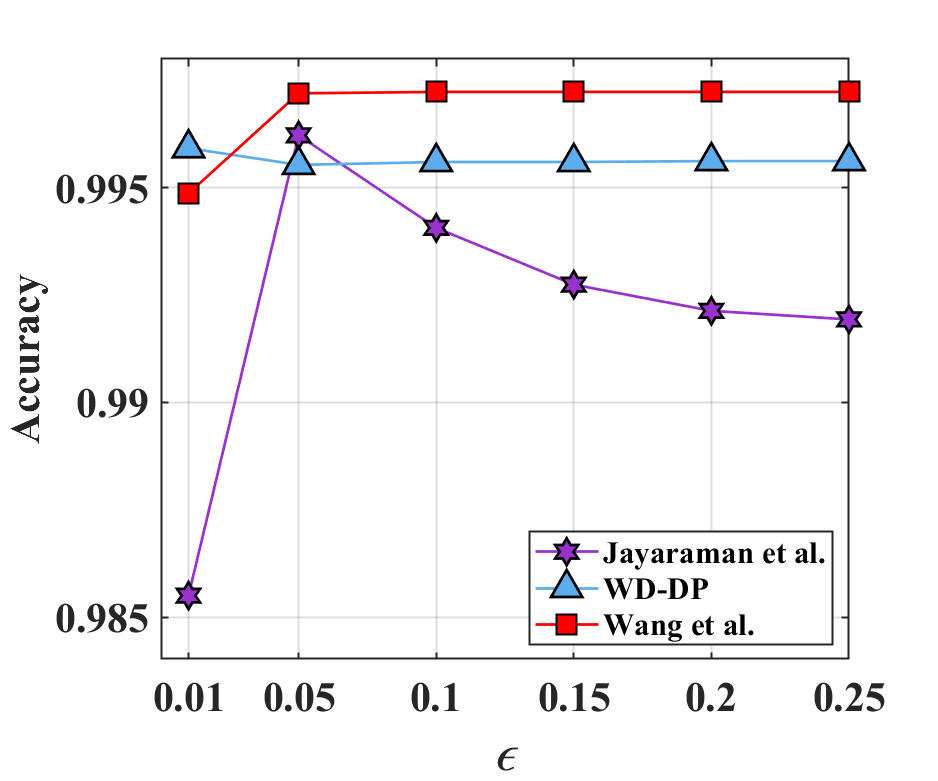}}
\subfigure[Adult]{\includegraphics[width=0.3\textwidth]{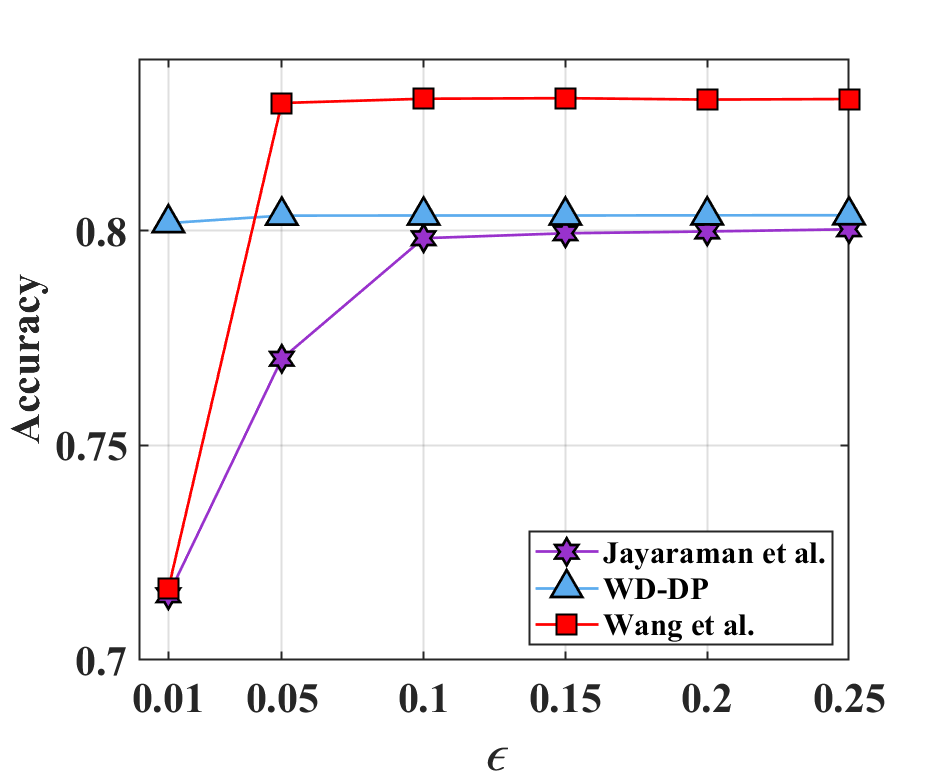}}
\subfigure[Bank]{\includegraphics[width=0.3\textwidth]{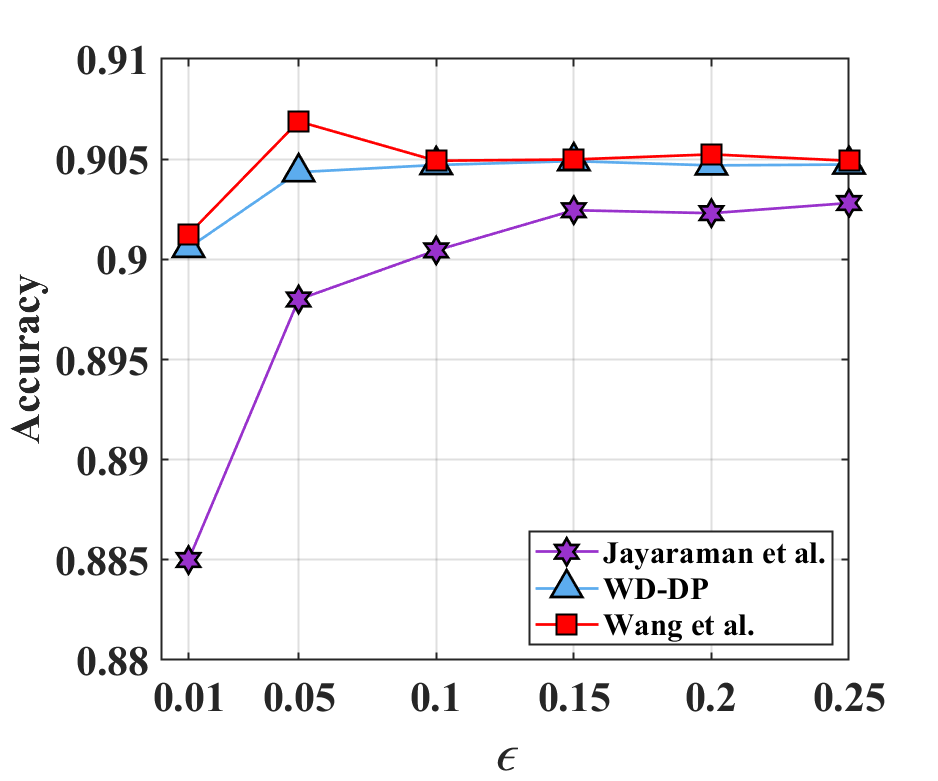}}
\subfigure[Breast Cancer]{\includegraphics[width=0.3\textwidth]{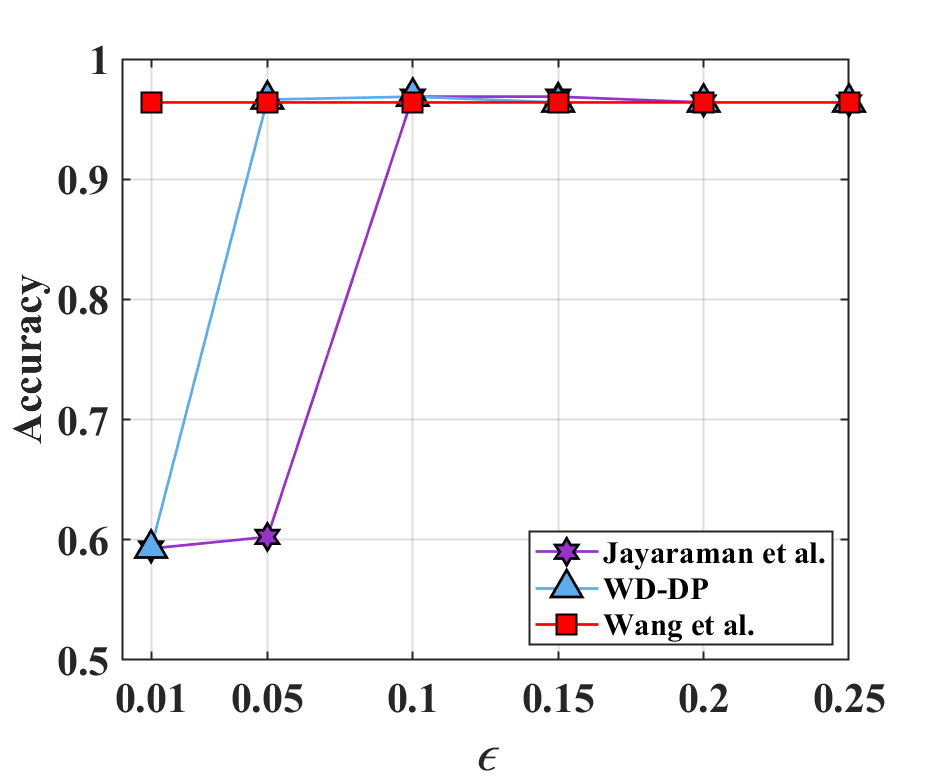}}
\subfigure[Credit Card Fraud]{\includegraphics[width=0.3\textwidth]{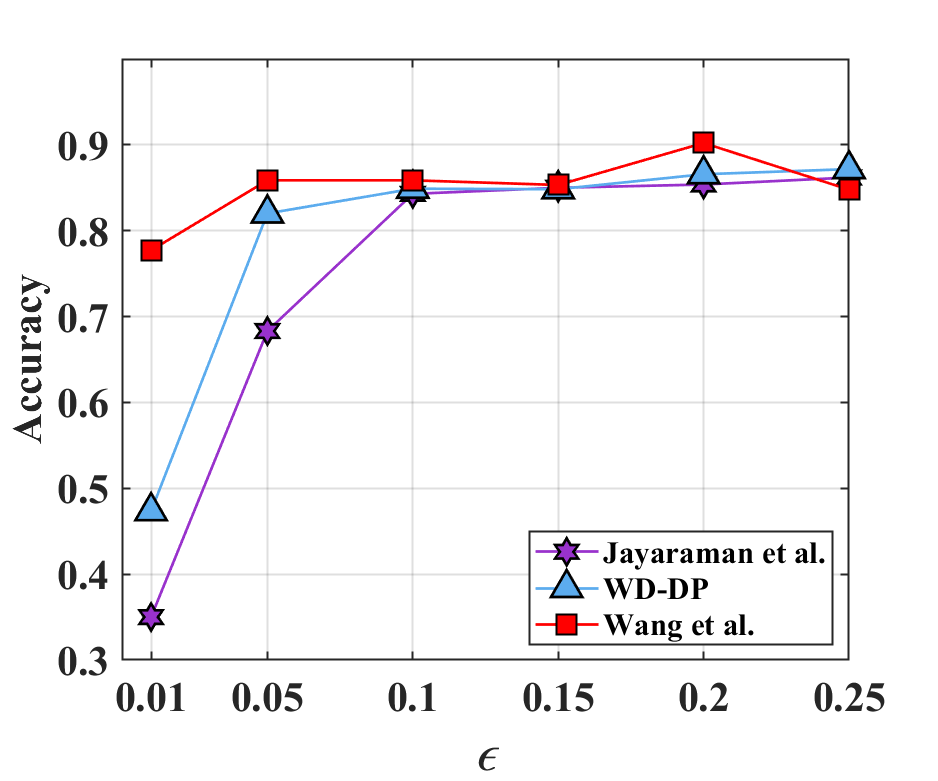}}
\caption{Accuracy on data sets over privacy budget $\epsilon$. $m=32$, data instances owned by each client is not the same.}
\end{figure*}

\subsection{Convex}
In this part, we give the theoretical analysis of excess empirical risk when objective function $L(\cdot)$ is $\lambda$-strongly convex.

\begin{The}
Suppose that $\ell(\theta,x,y)$ is G-Lipschitz and L-smooth over $\theta$, $L_D(\theta)$ is $\lambda$-strongly convex and differentiable, with $\sigma$ is the same as in (3) and learning rate $\eta=\frac{1}{L}$. We have:
\begin{equation*}
\mathbb{E}\left[L_D(\theta_T)\right]-L_D^* \leq O\left(\frac{pG^2\ln(1/\delta)\log(n)}{n^2\epsilon^2}\right),
\end{equation*}
where $T=\tilde{O}\left(\log(\frac{n^2\epsilon^2}{pG^2\ln(1/\delta)})\right)$, $L_D^*=\min_{\theta \in \mathbb{R}^p}L_D(\theta)$ and $p$ is the number of parameters.
\end{The}

\begin{proof}
According to updating criteria of gradient descent:
\begin{equation*}
\theta_{t+1}-\theta_{t}=-\eta(\nabla L_D(\theta_t)+z_t)=-\frac{1}{L}(\nabla L_D(\theta_t)+z_t).
\end{equation*}

Function $\ell$ is $L$-smooth (denoted as $L$ below) and $L_D(\theta)$ is differentiable (denoted as $d$ below), we have:
\begin{align}
\begin{split}
&\mathbb{E}_{z_t}\left[L_D(\theta_{t+1})-L_D(\theta_t)\right] \\
&\overset{(L,d)}{\leq}\mathbb{E}_{z_t}\left[\left< \nabla L_D(\theta_t), \theta_{t+1}-\theta_t \right>+\frac{L}{2}\left\|\theta_{t+1}-\theta_t \right\|^2\right] \\
&\leq -\frac{1}{L}\left\|\nabla L_D(\theta_t) \right\|^2-\frac{1}{L}\left<\nabla L_D(\theta_t),z_t\right> \\
&\quad+\frac{1}{2L}\left\|\nabla L_D(\theta_t) \right\|^2+\frac{1}{2L}\mathbb{E}_{z_t}\left\|z_t \right\|^2+\frac{1}{L}\left<\nabla L_D(\theta_t),z_t\right> \\
&= -\frac{1}{2L}\left\|\nabla L_D(\theta_t) \right\|^2+\frac{1}{2L}\mathbb{E}_{z_t}\left\|z_t \right\|^2.
\end{split}
\end{align}

$L_D(\theta)$ is $\lambda$-strongly convex and differentiable, from \cite{26}, we have:
\begin{align}
\begin{split}
\left\|\nabla L_D(\theta_t) \right\|^2 \geq 2\lambda(L_D(\theta_t)-L_D^*).
\end{split}
\end{align}

For random variable $X$, we have:
\begin{align}
\begin{split}
\mathbb{E}(X^2)=\mathbb{E}^2(X)+v(X),
\end{split}
\end{align}
where $v(X)$ denotes variance of $X$.

So, by (9) and (10), inequality (8) can be transferred to:
\begin{align*}
\begin{split}
&\mathbb{E}_{z_t}\left[L_D(\theta_{t+1})-L_D(\theta_t)\right] \\
&\leq -\frac{\lambda}{L}(L_D(\theta_t)-L_D^*)+\frac{1}{2L}\left(\mathbb{E}_{z_t}^2\left\|z_t\right\|+v(\left\|z_t\right\|)\right) \\
&= -\frac{\lambda}{L}(L_D(\theta_t)-L_D^*)+\frac{p\sigma^2}{2L}.
\end{split}
\end{align*}

Then, summing over $T$ iterations, we have:
\begin{equation}
\begin{aligned}
&\mathbb{E}\left[L_D(\theta_T)\right]-L_D^* \\
&\leq (1-\frac{\lambda}{L})^T(L_D(\theta_0)-L_D^*) \\
&\quad+\frac{p\sigma^2}{2L}\left((1-\frac{\lambda}{L})^0+
(1-\frac{\lambda}{L})^1+...+(1-\frac{\lambda}{L})^{T-1}\right) \\
&=(1-\frac{\lambda}{L})^T(L_D(\theta_0)-L_D^*)+\frac{p\sigma^2}{2L}\frac{L}{\lambda}\left(1-(1-\frac{\lambda}{L})^T\right) \\
&\leq(1-\frac{\lambda}{L})^T(L_D(\theta_0)-L_D^*)+\frac{p\sigma^2}{2\lambda}.
\end{aligned}
\end{equation}

Taking $T=\tilde{O}\left(\log(\frac{n^2\epsilon^2}{pG^2\ln(1/\delta)})\right)$, we have:
\begin{equation*}
\mathbb{E}\left[L_D(\theta_T)\right]-L_D^* \leq O\left(\frac{pG^2\ln(1/\delta)\log(n)}{n^2\epsilon^2}\right).
\end{equation*}
\end{proof}

\begin{Rem}
In \cite{11}, inequality (11) is simply scaling to:
\begin{equation}
\mathbb{E}\left[L_D(\theta_T)\right]-L_D^* \leq (1-\frac{\lambda}{L})^T(L_D(\theta_0)-L_D^*)+\frac{Tp\sigma^2}{2L}.
\end{equation}
Obviously, equation (11) is tighter than equation (12).
In this way, we improve the proof process, leading a better excess empirical risk bound by a factor of log(n).
\end{Rem}

It can be observed that our method is better than which in distributed setting \cite{8} and centralized setting \cite{11} by a factor of $\frac{(mn_{(1)})^2\log(n)}{(\log(mn_{(1)})n)^2}$ and $\log(n)$, respectively.
Intuitively, giving weights to parties means data instances owned by all parties are of the same importance, similar to centralized setting.
Conversely, simply averaging gives more weight to data instances in smaller data sets, making it more \textit{distributed}.

\begin{figure*}[htb]
\centering
\subfigure[KDDCup99]{\includegraphics[width=0.3\textwidth]{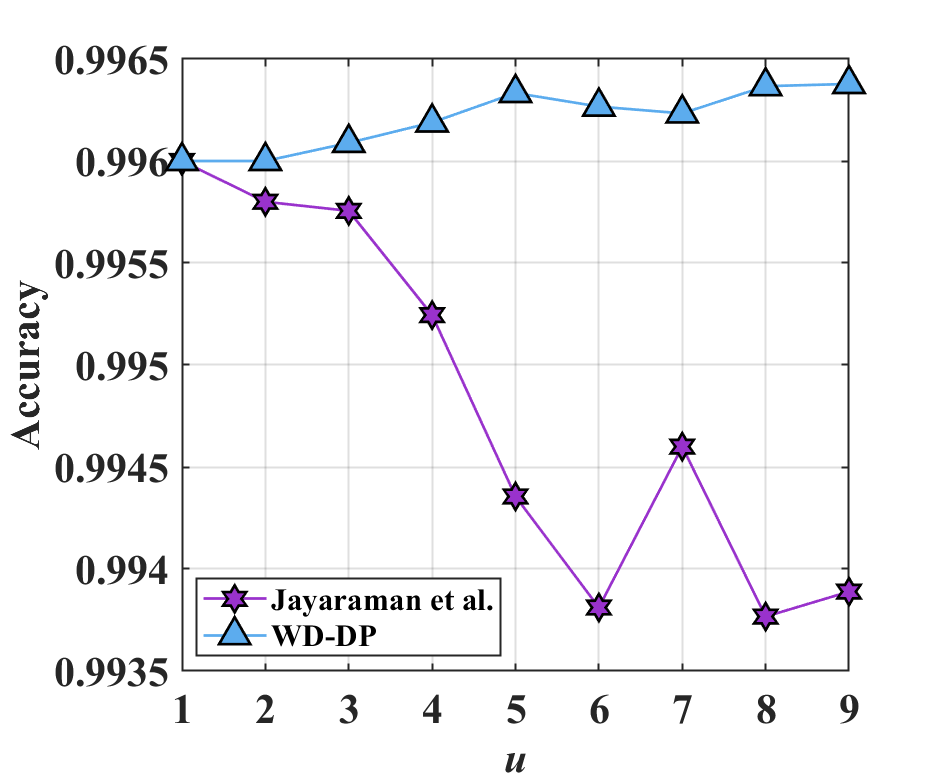}}
\subfigure[Adult]{\includegraphics[width=0.3\textwidth]{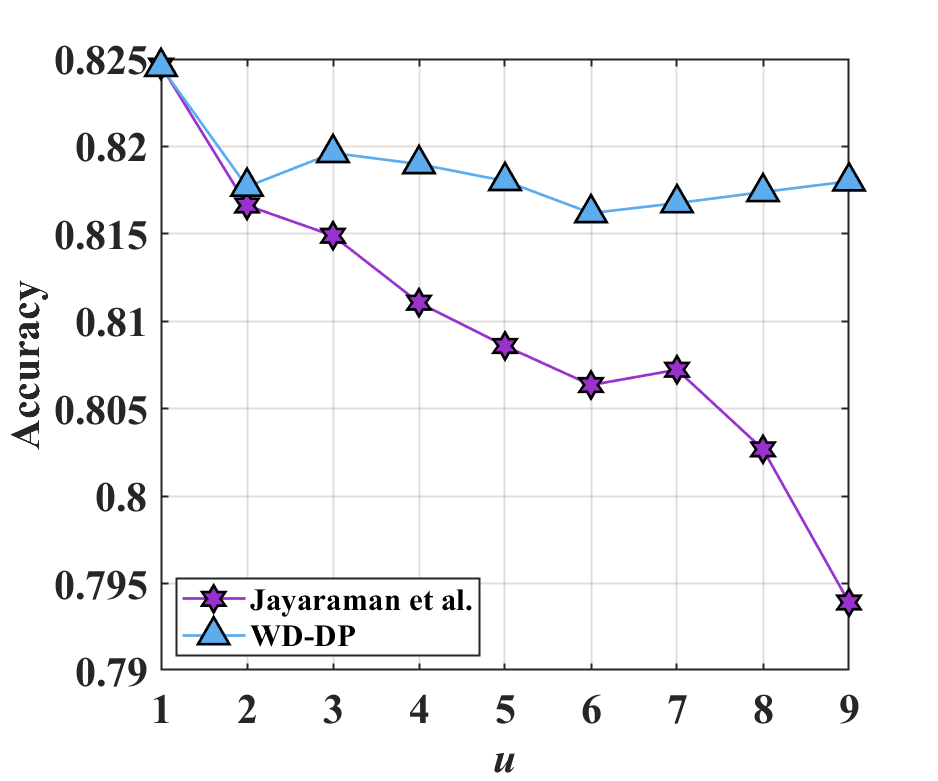}}
\subfigure[Bank]{\includegraphics[width=0.3\textwidth]{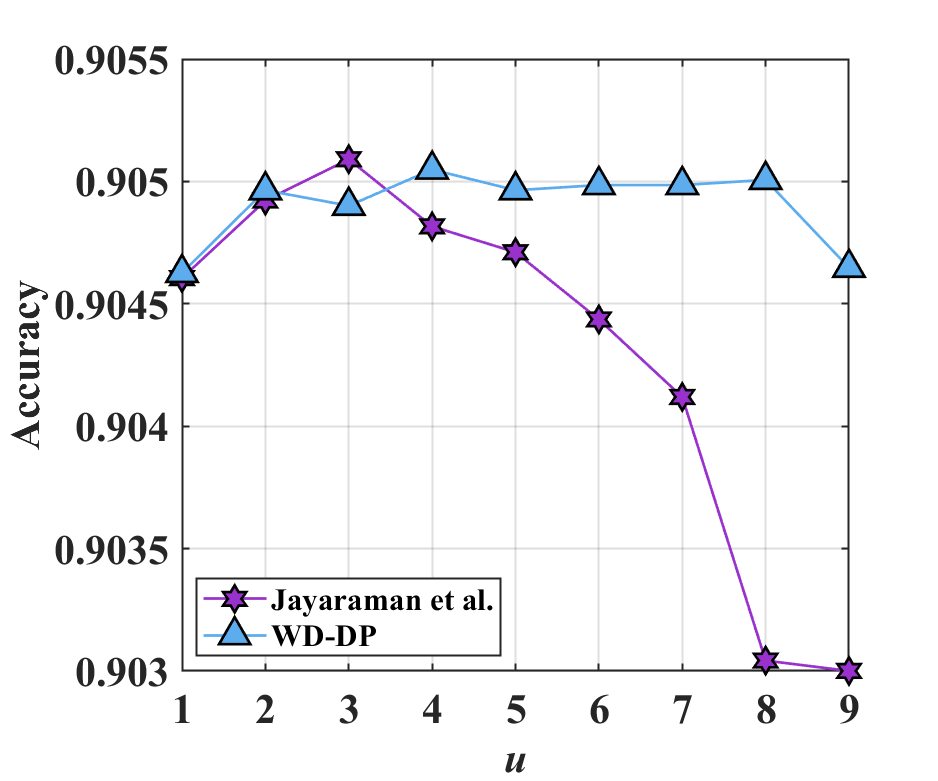}}
\subfigure[Breast Cancer]{\includegraphics[width=0.3\textwidth]{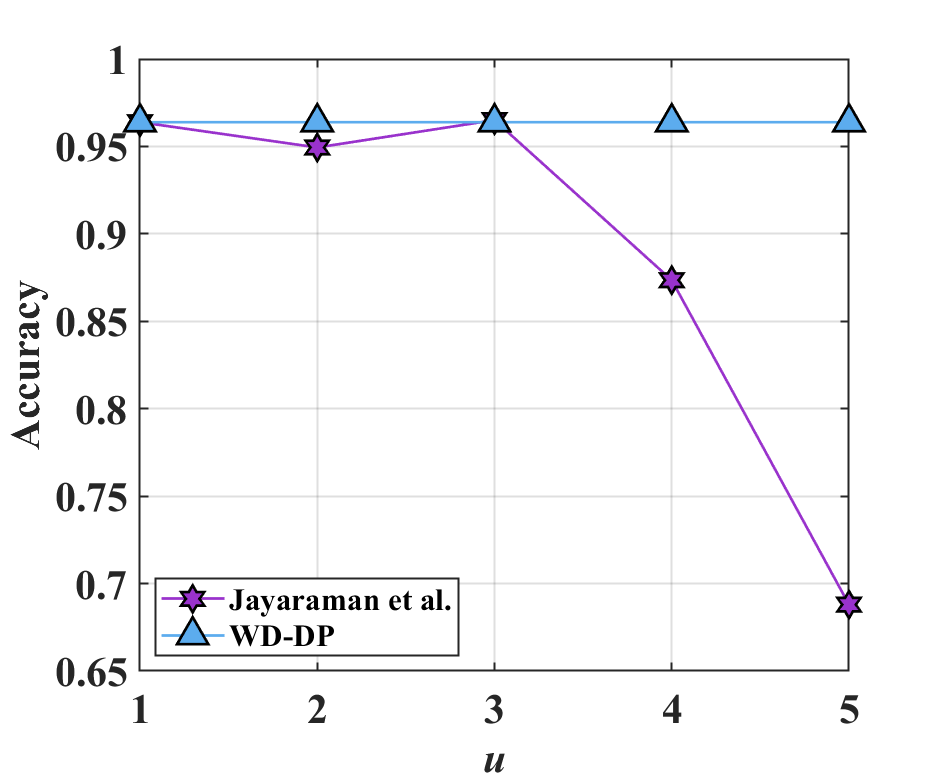}}
\subfigure[Credit Card Fraud]{\includegraphics[width=0.3\textwidth]{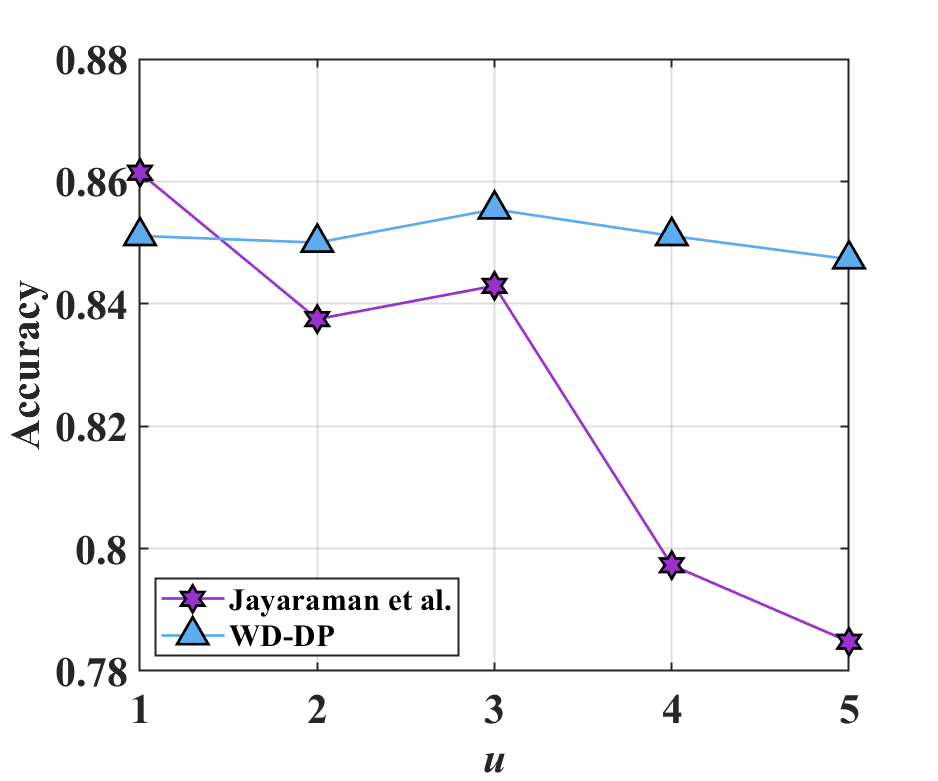}}
\caption{Accuracy on data sets over the level of non-average $u$. $m=16$, $\epsilon=0.05$, clients are divided into 2 groups, the number of data instances owned by each client in the same group is the same.}
\end{figure*}

\subsection{Non-convex}
In this part, we generalize the analysis above to non-convex $L(\cdot)$ which satisfies the Polyak-Lojasiewicz condition.

\begin{Def}
For function $L$, denotes $L^*=\min _{\theta \in \mathbb{R}^p}L(\theta)$, if there exists $\mu > 0$ and for every $\theta$,
\begin{equation}
\left\| \nabla L(\theta) \right\|^2 \geq 2\mu(L(\theta)-L^*),
\end{equation}
then function $L$ satisfies the Polyak-Lojasiewicz condition.
\end{Def}

Obviously, convex functions also satisfy equation (13).
In fact, \textit{Polyak-Lojasiewicz condition} is much more general than \textit{convex}.
\citeauthor{10} \shortcite{10} showed that when function $F$ is differentiable and $L$-smooth under $\ell_2$ norm, we have: \\
Strong Convex $\Rightarrow$ Essential Strong Convexity $\Rightarrow$ Weak Strongly Convexity $\Rightarrow$ Restricted Secant Inequality $\Rightarrow$ Polyak-Lojasiewicz Inequality $\Leftrightarrow$ Error Bound

\begin{The}
Suppose that $\ell(\theta,x,y)$ is G-Lipschitz and L-smooth over $\theta$, $L_D(\theta)$ satisfies the Polyak-Lojasiewicz condition and differentiable, $\sigma$ is the same as (3) and $\eta=\frac{1}{L}$. We have:
\begin{equation*}
\mathbb{E}\left[L_D(\theta_T)\right]-L_D^* \leq O\left(\frac{pG^2\ln(1/\delta)\log(n)}{n^2\epsilon^2}\right),
\end{equation*}
where $T=\tilde{O}\left(\log(\frac{n^2\epsilon^2}{pG^2\ln(1/\delta)})\right)$, $L_D^*=\min_{\theta \in \mathbb{R}^p}L_D(\theta)$ and $p$ is the number of parameters.
\end{The}
The proof of Theorem 3 is shown in Appendix A.1.

It can be observed that our excess empirical risk bound over both convex function and non-convex function which satisfies the Polyak-Lojasiewicz condition are tighter than which over convex function in \cite{8} by a factor of $\frac{(mn_{(1)})^2\log(n)}{(\log(mn_{(1)})n)^2}$, where $m$ is the number of parties and $n_{(1)}$ denotes the smallest data set's size.
Under the situation of uneven data scale in real scenarios, the gap between $mn_{(1)}$ and $n$ is huge, our method is extremely superior.

Moreover, by Remark 1, we proof that the excess empirical risk bound can be tighter than which in \cite{11} by a factor of $\log(n)$.

\section{Experiments}
\noindent Experiments are performed on classification task.
We compare our method to the gradient perturbation method proposed by \citeauthor{8} \shortcite{8} and the centralized privacy method proposed by \citeauthor{11} \shortcite{11}.
The comparison between our method and others is represented by accuracy and optimal gap.
Optimal gap is defined as $L(\theta)-L(\theta^*)$, where $\theta^*$ is centralized optimal non-privacy model.
Accuracy represents the performance on test data and optimal gap denotes excess empirical risk on training data.

We use logistic regression method on the data set KDDCup99 \cite{22}, Adult \cite{31}, Bank \cite{32}, Breast Cancer \cite{33} and Credit Card Fraud \cite{34}, the number of total data instances are 70000, 45222, 41188, 699 and 984, respectively.
In all the experiments, $G=1$, the Lipschitz constant of the loss function (the proof is shown in Appendix A.2).
Total local iteration rounds $T$ is set 1000 and learning rate $\eta$ is chosen by cross-validation from 0.01, 0.05, 0.1, 0.5 and 1.

First, we evaluate the influence over differential privacy budget $\epsilon$, $\epsilon$ is set from 0.01 to 0.25 and $\delta=0.001$.
In this setting, we set the number of clients $m=2,4,8,16,32$ and the number of data instances owned by each client is set randomly.
Moreover, considering about the real scenario, we set a threshold of the smallest data set's size, ensuring effective models are trained by clients.
Then, we evaluate the influence caused by the difference on data sets' size owned by different clients.
We define the level of non-average $u$ as $\frac{n_{max}}{n_{min}}$, where $n_{max}$ and $n_{min}$ denote the maximum and minimum data set's size, respectively.
In the experiments, considering the number of total data instances, $u$ is set from 1 to 9 on data set KDDCup99, Adult and Bank, while from 1 to 5 on the rest data sets.
Particularly, $u=1$ means average setting.
For the sake of simplicity, we divide all the clients into 2 groups, clients in the same group have the same data sets's size.

Figure 1 shows the accuracy over privacy budget $\epsilon$.
It can be observed that by considering different weights of different clients, our method WD-DP is better than the method proposed in \cite{8} and is similar to the centralized method proposed in \cite{11}.
Performance is becoming better when $\epsilon$ increases, which is the same as intuition.
Corresponding optimal gap and more experiment results over $\epsilon$ with different number of clients $m$ are shown in Appendix B, leading similar results.

Figure 2 shows the accuracy over the level of non-average $u$.
It can be observed that in average setting, when $u=1$, the accuracy of our proposed method WD-DP is similar to the method proposed in \cite{8}.
However, when $u$ increases, which means data scale is more and more uneven, the accuracy of our method is steady, but the accuracy of which proposed by \citeauthor{8} \shortcite{8} decreases rapidly or fluctuates sharply.
Thus, our method is more reliable, especially in the case of uneven data scale, which is the same as in theoretical analysis.
Corresponding optimal gap and more experiment results over $u$ with different number of clients $m$ and privacy budget $\epsilon$ are also shown in Appendix B, which lead similar results.

\section{Conclusion and Discussion}
\noindent We propose a distributed differential privacy ERM method WD-DP, providing $(\epsilon, \delta)$-differential privacy by gradient perturbation, adding Gaussian noise.
Our work shows that by considering about different weights of different clients, noise bound and excess empirical risk bound can be improved in distributed setting.
Moreover, considering most previous work on differential privacy ERM assumes loss functions are convex and this constraint is not easy to achieve in some situations, we generalize our method to non-convex conditions.
Theoretical analysis and experiment results on real data sets show that we improve the best previous noise bound and excess empirical risk bound for distributed differential privacy ERM, especially under the condition that data scale on clients is uneven, which is common in real scenarios.
In future work, we will focus on non-convex optimization under distributed differential privacy ERM setting (e.g. deep learning), and reducing time complexity of the model, considering most models are synchronous.

\bibliographystyle{aaai}
\bibliography{ref}

\onecolumn
\begin{appendix}

\section{A. More Details and Proofs}
\subsection{A.1 Proof of Theorem 3}
\begin{proof}
The proof is similar to Theorem 2.

According to updating criteria of gradient descent:
\begin{equation*}
\theta_{t+1}-\theta_{t}=-\eta(\nabla L_D(\theta_t)+z_t)=-\frac{1}{L}(\nabla L_D(\theta_t)+z_t).
\end{equation*}

Function $\ell$ is $L$-smooth (denoted as $L$ below) and $L_D(\theta)$ is differentiable (denoted as $d$ below), we have:
\begin{align}
\begin{split}
&\mathbb{E}_{z_t}\left[L_D(\theta_{t+1})-L_D(\theta_t)\right] \\
&\overset{(L,d)}{\leq}\mathbb{E}_{z_t}\left[\left< \nabla L_D(\theta_t), \theta_{t+1}-\theta_t \right>+\frac{L}{2}\left\|\theta_{t+1}-\theta_t \right\|^2\right] \\
&\leq -\frac{1}{L}\left\|\nabla L_D(\theta_t) \right\|^2-\frac{1}{L}\left<\nabla L_D(\theta_t),z_t\right>+\frac{1}{2L}\left\|\nabla L_D(\theta_t) \right\|^2+\frac{1}{2L}\mathbb{E}_{z_t}\left\|z_t \right\|^2+\frac{1}{L}\left<\nabla L_D(\theta_t),z_t\right> \\
&= -\frac{1}{2L}\left\|\nabla L_D(\theta_t) \right\|^2+\frac{1}{2L}\mathbb{E}_{z_t}\left\|z_t \right\|^2.
\end{split}
\end{align}

Note that $L_D(\theta)$ satisfies the Polyak-Lojasiewicz condition, then we have:
\begin{equation}
\left\|\nabla L_D(\theta_t) \right\|^2 \geq 2\mu(L_D(\theta_t)-L_D^*).
\end{equation}

For random variable $X$, we have:
\begin{align}
\begin{split}
\mathbb{E}(X^2)=\mathbb{E}^2(X)+v(X),
\end{split}
\end{align}
where $v(X)$ denotes variance of $X$.

So, by (15) and (16), inequality (14) can be transferred to:
\begin{align*}
\begin{split}
&\mathbb{E}_{z_t}\left[L_D(\theta_{t+1})-L_D(\theta_t)\right] \\
&\leq -\frac{\mu}{L}(L_D(\theta_t)-L_D^*)+\frac{1}{2L}\left(\mathbb{E}_{z_t}^2\left\|z_t\right\|+v(\left\|z_t\right\|)\right) \\
&= -\frac{\mu}{L}(L_D(\theta_t)-L_D^*)+\frac{p\sigma^2}{2L}.
\end{split}
\end{align*}

Then, summing over $T$ iterations, we have:
\begin{equation*}
\begin{aligned}
&\mathbb{E}\left[L_D(\theta_T)\right]-L_D^* \\
&\leq (1-\frac{\mu}{L})^T(L_D(\theta_0)-L_D^*)+\frac{p\sigma^2}{2L}\left((1-\frac{\mu}{L})^0+
(1-\frac{\mu}{L})^1+...+(1-\frac{\mu}{L})^{T-1}\right) \\
&=(1-\frac{\mu}{L})^T(L_D(\theta_0)-L_D^*)+\frac{p\sigma^2}{2L}\frac{L}{\mu}\left(1-(1-\frac{\mu}{L})^T\right) \\
&\leq(1-\frac{\mu}{L})^T(L_D(\theta_0)-L_D^*)+\frac{p\sigma^2}{2\mu}.
\end{aligned}
\end{equation*}

Taking $T=\tilde{O}\left(\log(\frac{n^2\epsilon^2}{pG^2\ln(1/\delta)})\right)$, we have:
\begin{equation*}
\mathbb{E}\left[L_D(\theta_T)\right]-L_D^* \leq O\left(\frac{pG^2\ln(1/\delta)\log(n)}{n^2\epsilon^2}\right).
\end{equation*}

\end{proof}

\subsection{A.2 The Lipschitz Constant of Logistic Regression when Using Cross-Entropy}
When using cross-entropy, the loss function of logistic regression is:
\begin{equation}
J(\theta)=-\frac{1}{n}\sum_{i=1}^{n}\left[y_i\log(h(z_i))+(1-y_i)\log(1-h(z_i))\right],
\end{equation}
where $h(z_i)=\frac{1}{1+e^{-z_i}}$ and $z_i=x_i\theta$.
\begin{proof}
From (17), we have:
\begin{equation*}
\ell=y\log(h(z))+(1-y)\log(1-h(z)).
\end{equation*}
Then, note that $\nabla h(z)=h(z)(1-h(z))$, we have:
\begin{equation*}
\begin{split}
\nabla \ell&=y\frac{1}{h(z)}\nabla h(\theta)+(1-y)\frac{-1}{1-h(z)}\nabla h(\theta) \\
&=xy\frac{1}{h(z)}h(z)(1-h(z))+x(y-1)\frac{1}{1-h(z)}h(z)(1-h(z)) \\
&=x(y-h(z)).
\end{split}
\end{equation*}
Note that $0 < h(z) < 1$, $\left\|y\right\| \leq 1$ and $x$ is normalized, we have $\left\|\nabla \ell\right\| \leq 1$, which means the Lipschitz constant of $\ell$ is $1$.
\end{proof}

\section{B. More Experimental Results}
We give the accuracy and optimal gap (defined as $L(\theta)-L(\theta^*)$, $\theta^*$ is centralized optimal non-privacy model) over privacy budget $\epsilon$ from figure 3 to figure 11.
Figure 3, 5, 7, 9 and 11 show that by considering different weights held by clients, the optimal gap of our method is better than the distributed method proposed in \cite{8} and is similar to the centralized method proposed in \cite{11}, which means the excess empirical risk of our method is similar to centralized methods.
With the increasing of $\epsilon$, the optimal gap decreases, which is the same as intuition.
Figure 4, 6, 8 and 10 show that the accuracy of our method is better than the method proposed by \citeauthor{8} \shortcite{8} by considering weights of parties.
Experimental results show that the performance of our method is straight up to centralized methods, which is similar to the theoretical analysis in Section 5.

Figure 12 to figure 16 show the accuracy and optimal gap over the level of non-average $u$.
In this setting, the value $m$ and $\epsilon$ are chosen randomly.
Figure 12, 14 and 16 show that with the increasing of $u$, which means data scale on different clients is more and more uneven, the optimal gap of our method is steady.
However, the optimal gap of the method proposed by \citeauthor{8} \shortcite{8} increases rapidly or fluctuates sharply.
Figure 13 and 15 show that the accuracy of our method is steady with the increasing of $u$, while the accuracy of the method proposed in \cite{8} decreases rapidly or fluctuates sharply.
Thus, our method is more reliable than the method proposed by \citeauthor{8} \shortcite{8}, especially in the case that data scale is uneven, similar to the theoretical analysis in Section 5.

\begin{figure*}[htb]
\centering
\subfigure[KDDCup99]{\includegraphics[width=0.3\textwidth]{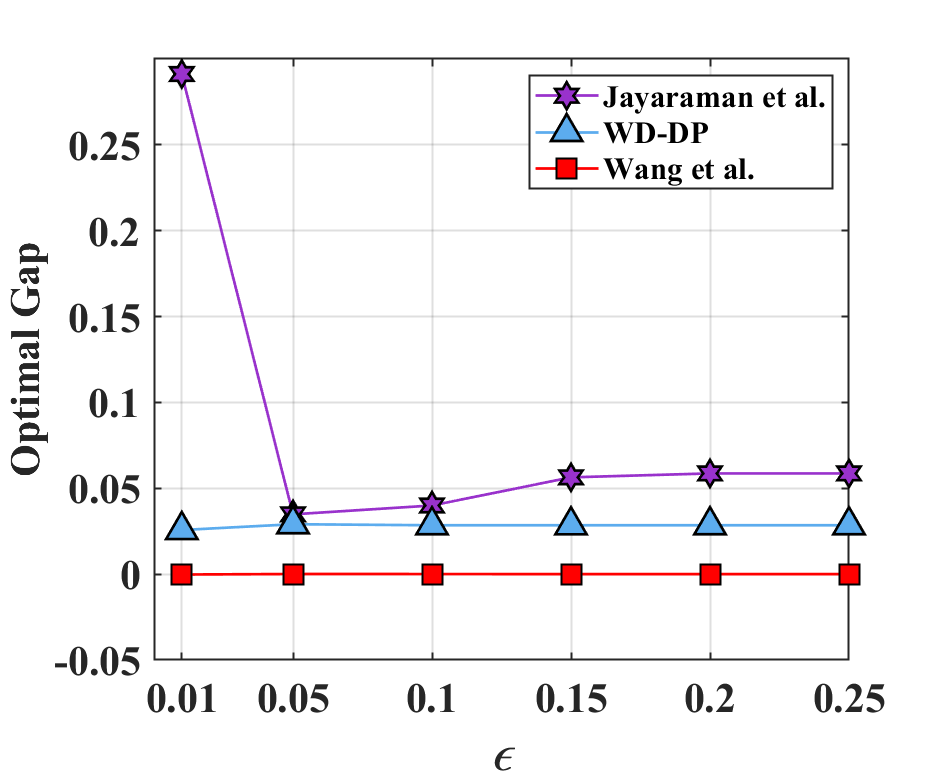}}
\subfigure[Adult]{\includegraphics[width=0.3\textwidth]{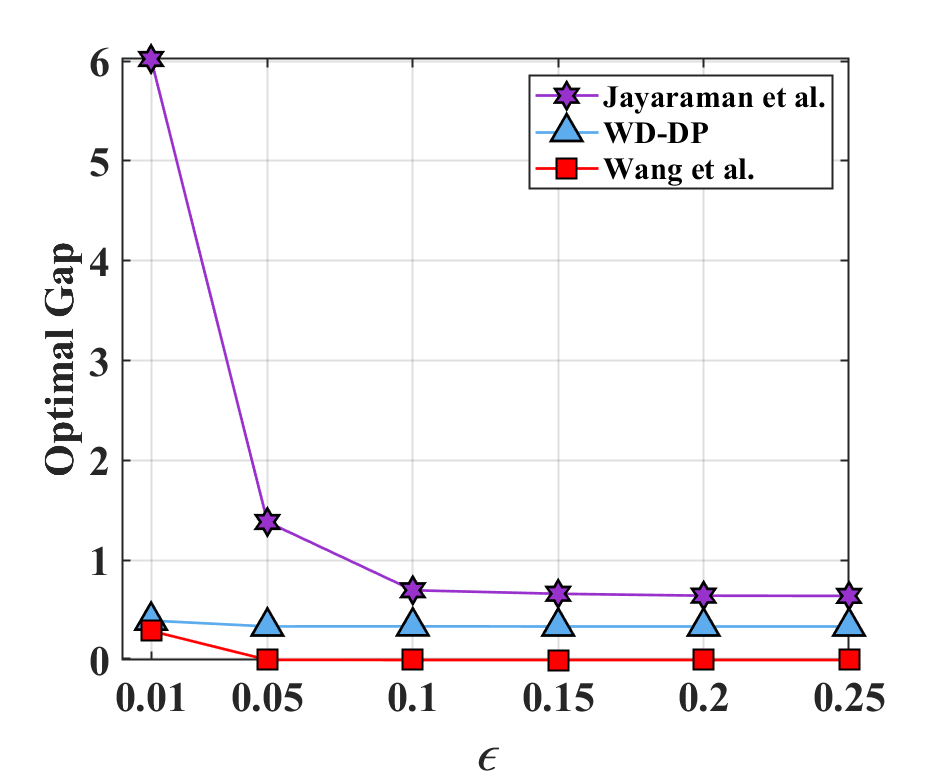}}
\subfigure[Bank]{\includegraphics[width=0.3\textwidth]{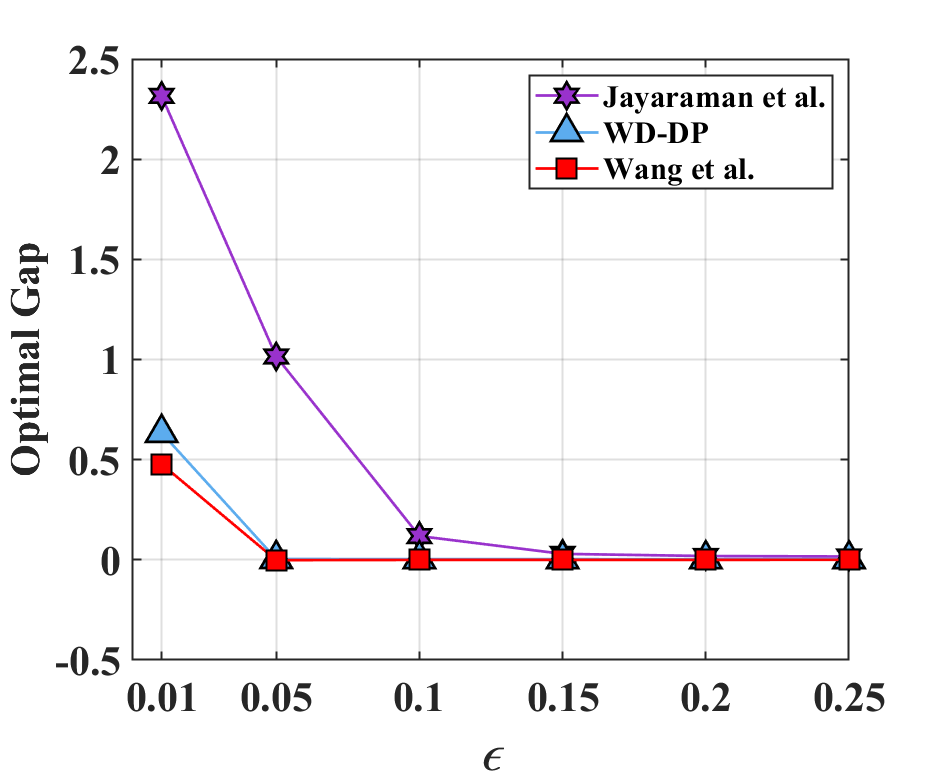}}
\subfigure[Breast Cancer]{\includegraphics[width=0.3\textwidth]{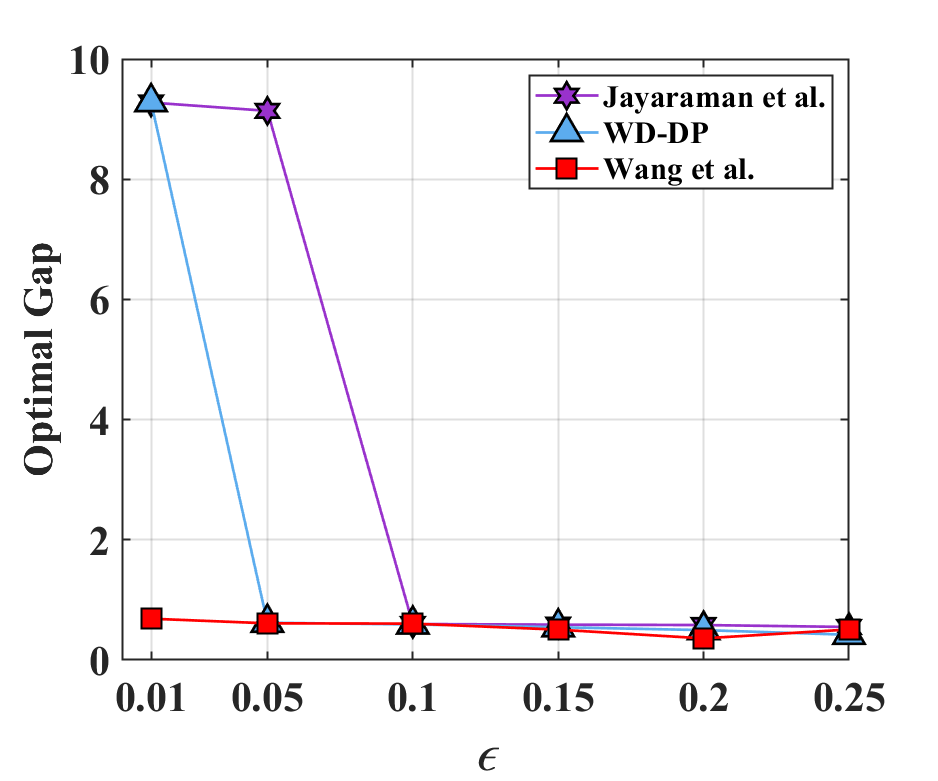}}
\subfigure[Credit Card Fraud]{\includegraphics[width=0.3\textwidth]{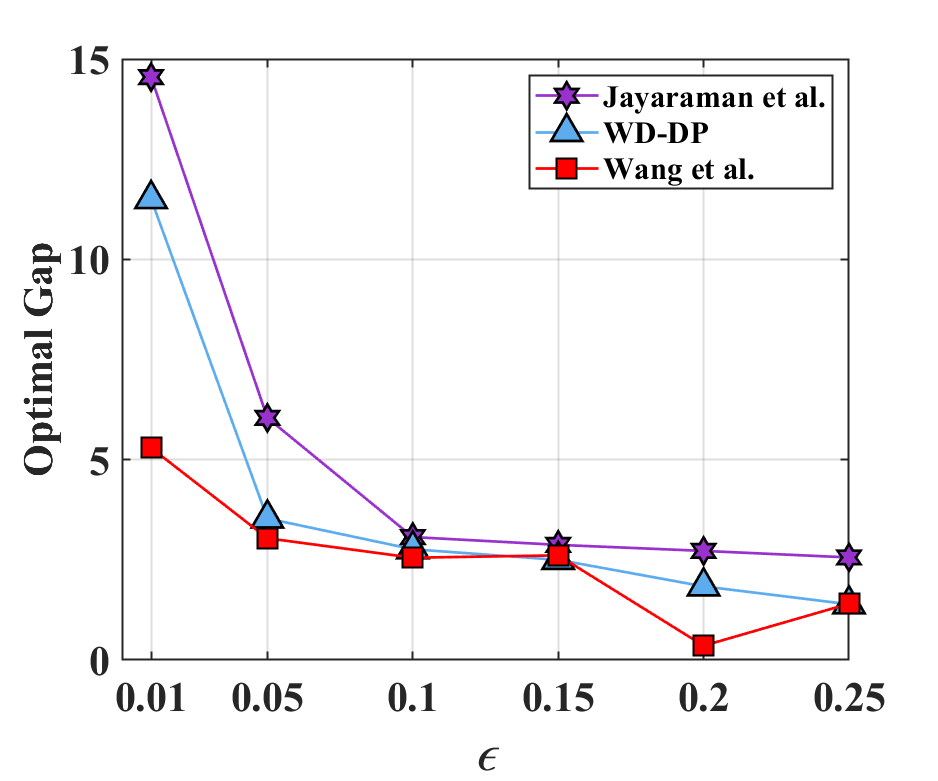}}
\caption{Optimal gap on data sets over privacy budget $\epsilon$. $m=32$, data instances owned by each client is not the same.}
\end{figure*}

\begin{figure*}[htb]
\centering
\subfigure[KDDCup99]{\includegraphics[width=0.3\textwidth]{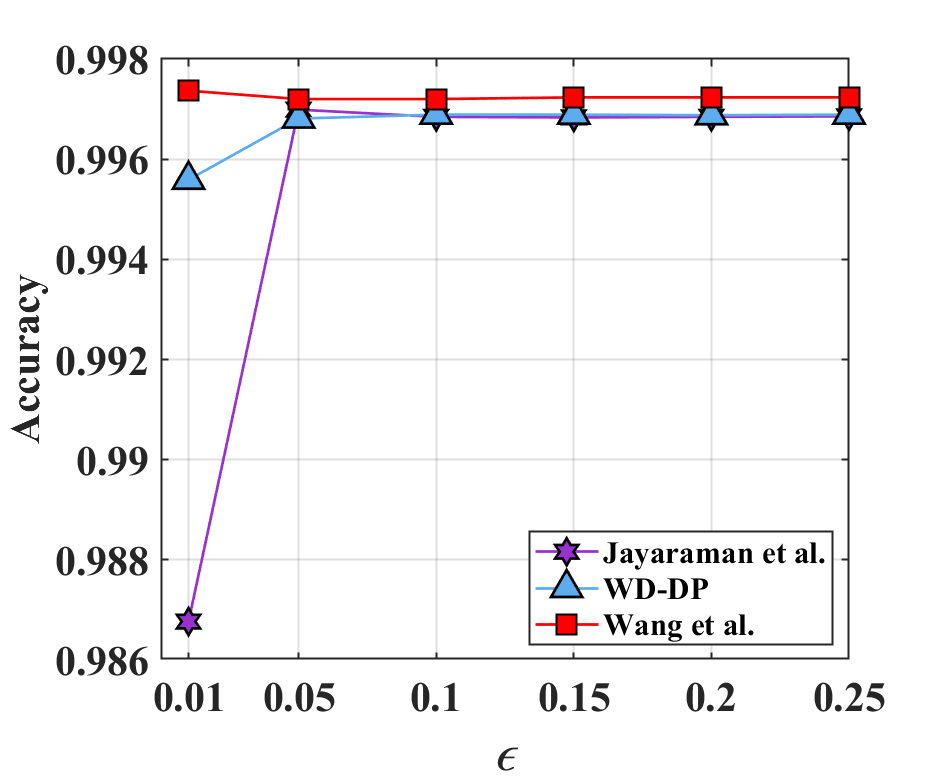}}
\subfigure[Adult]{\includegraphics[width=0.3\textwidth]{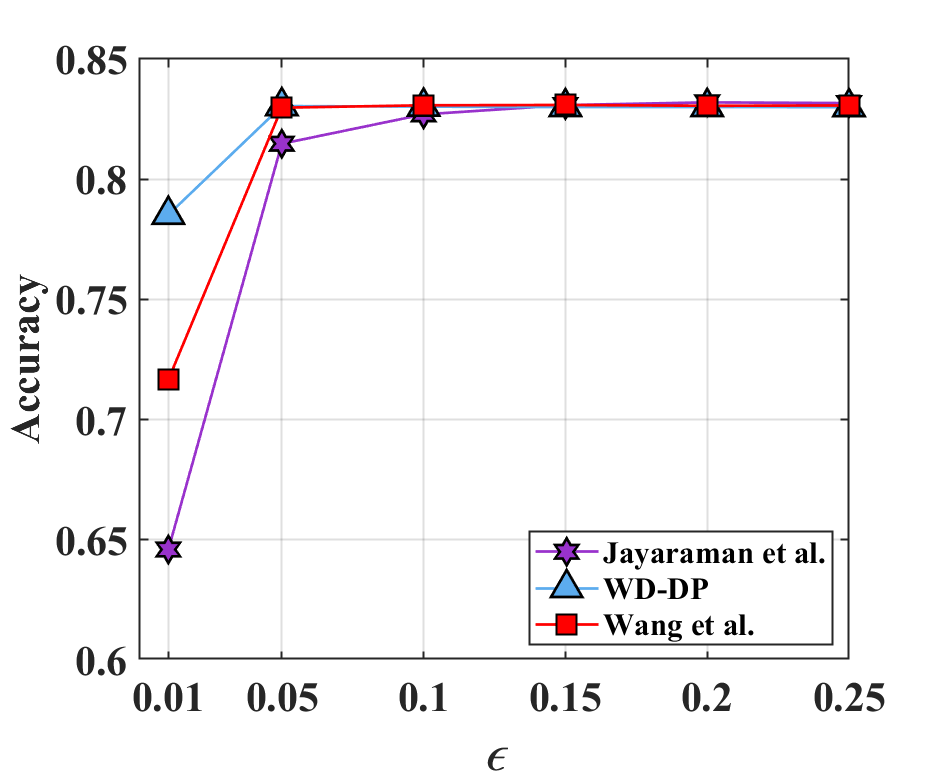}}
\subfigure[Bank]{\includegraphics[width=0.3\textwidth]{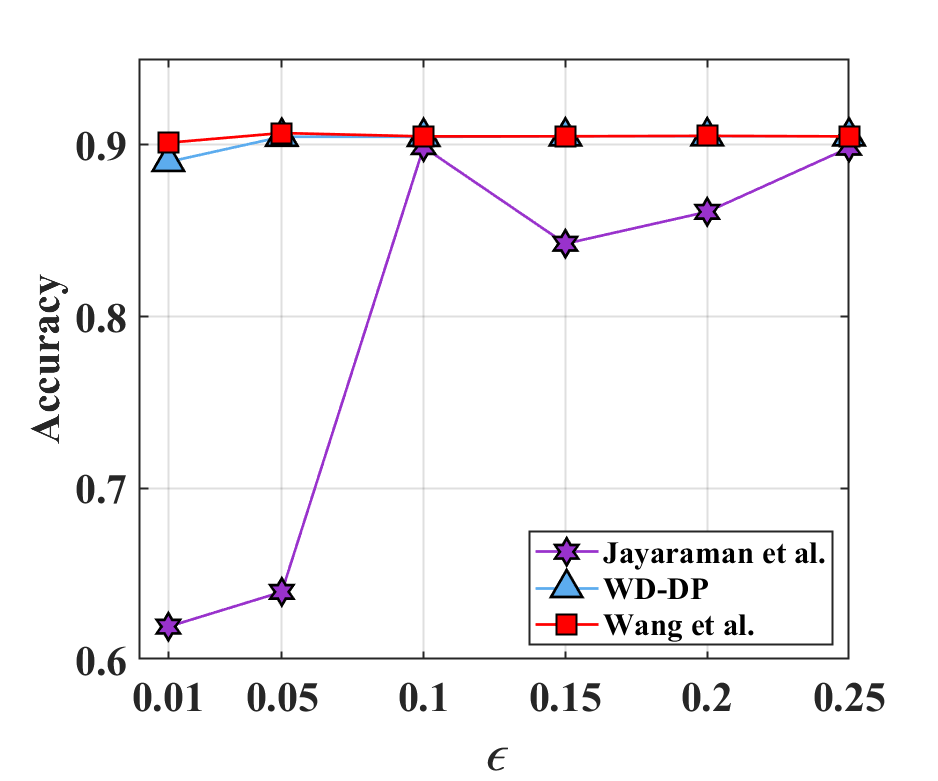}}
\subfigure[Breast Cancer]{\includegraphics[width=0.3\textwidth]{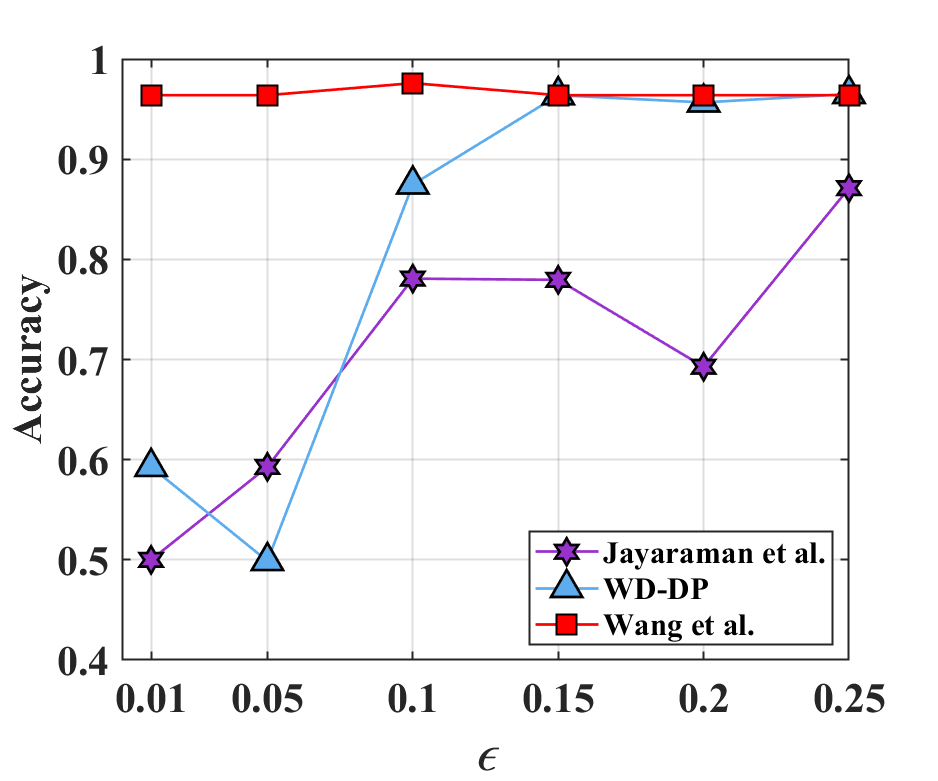}}
\subfigure[Credit Card Fraud]{\includegraphics[width=0.3\textwidth]{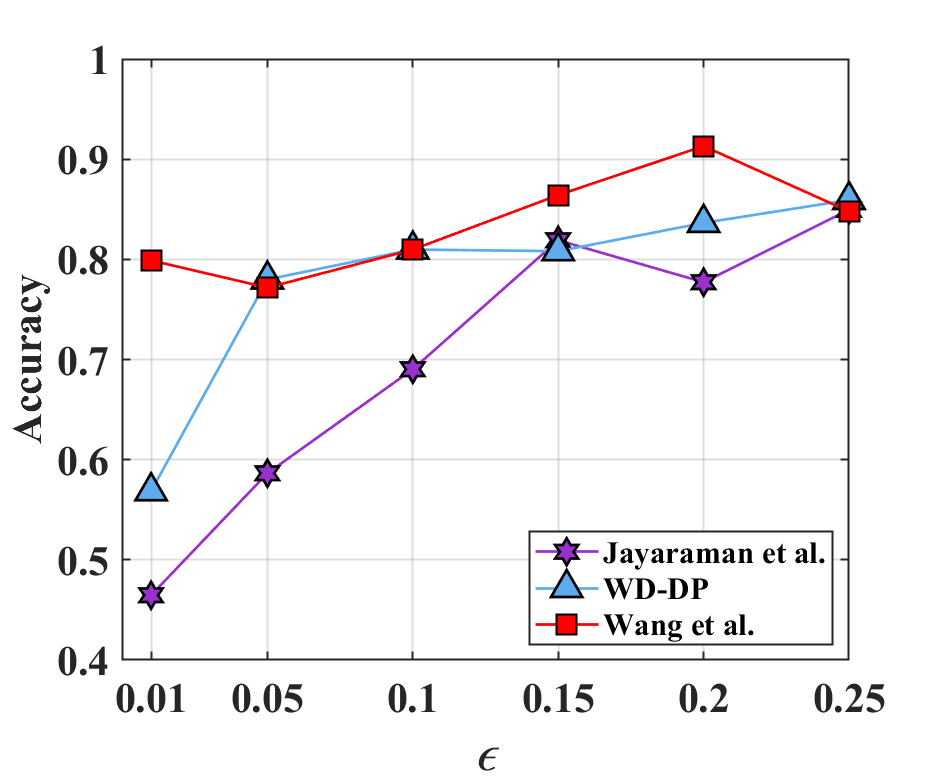}}
\caption{Accuracy on data sets over privacy budget $\epsilon$. $m=2$, data instances owned by each client is not the same.}
\end{figure*}

\begin{figure*}[htb]
\centering
\subfigure[KDDCup99]{\includegraphics[width=0.3\textwidth]{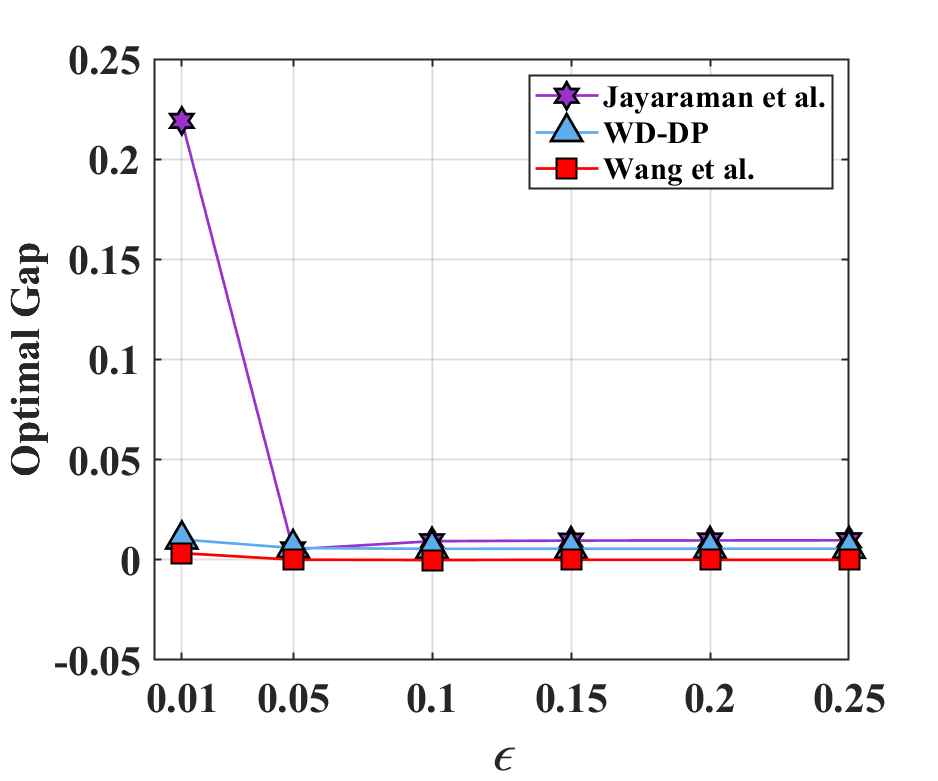}}
\subfigure[Adult]{\includegraphics[width=0.3\textwidth]{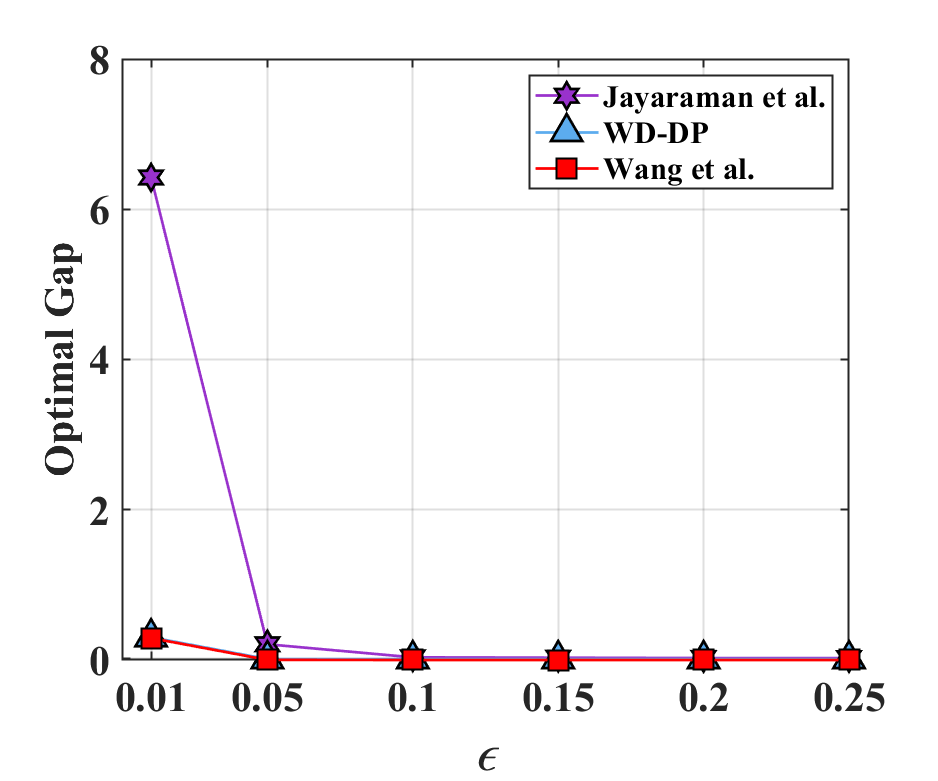}}
\subfigure[Bank]{\includegraphics[width=0.3\textwidth]{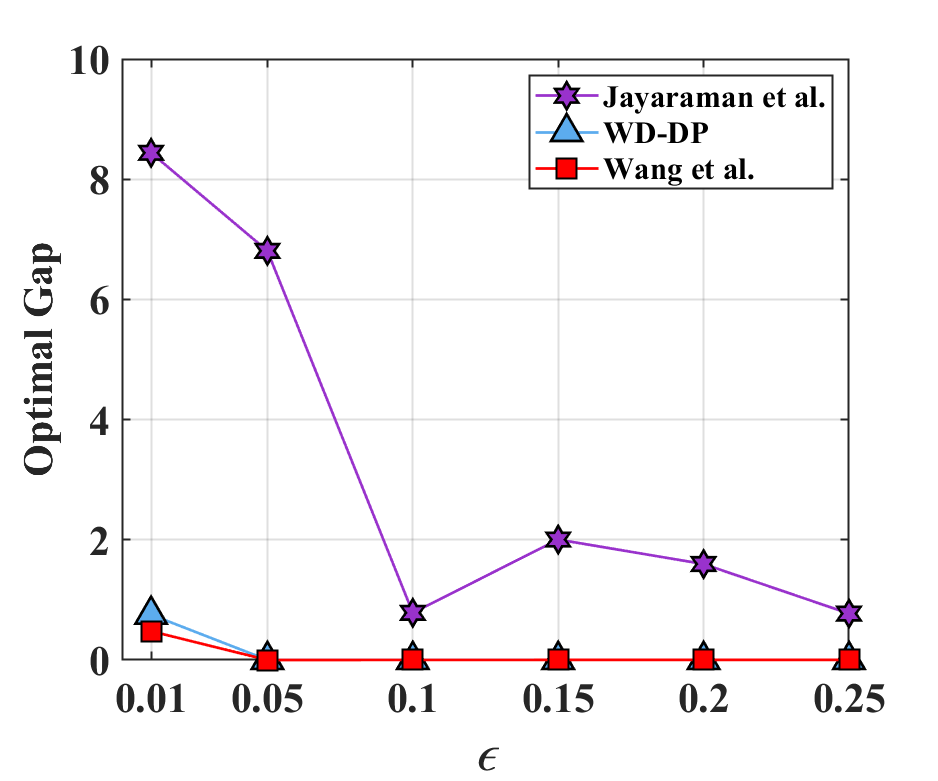}}
\subfigure[Breast Cancer]{\includegraphics[width=0.3\textwidth]{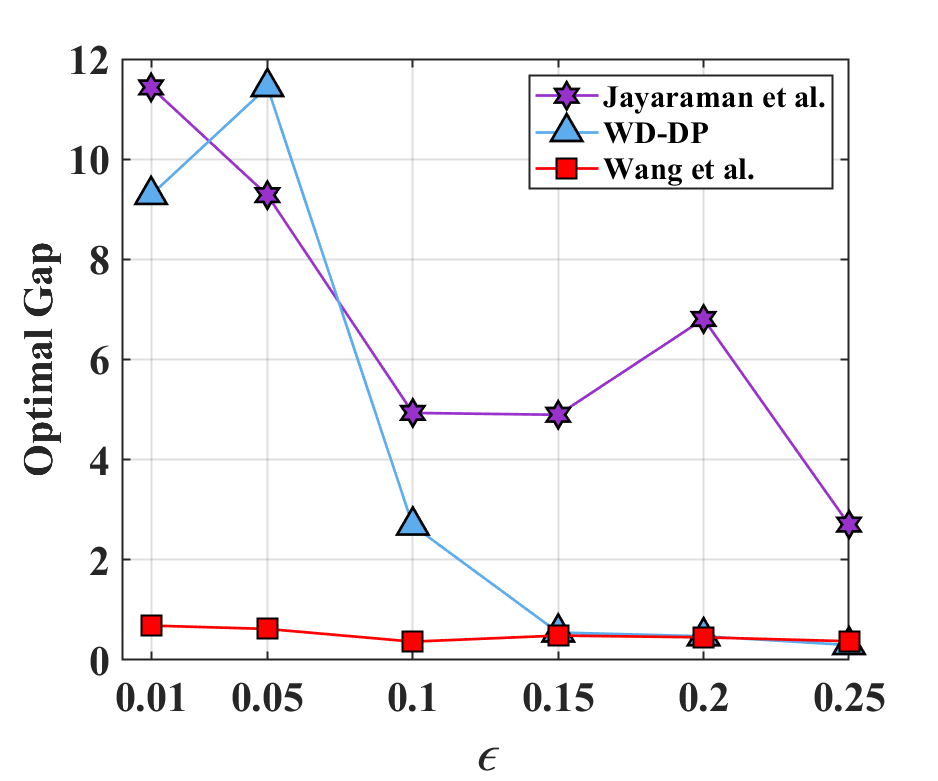}}
\subfigure[Credit Card Fraud]{\includegraphics[width=0.3\textwidth]{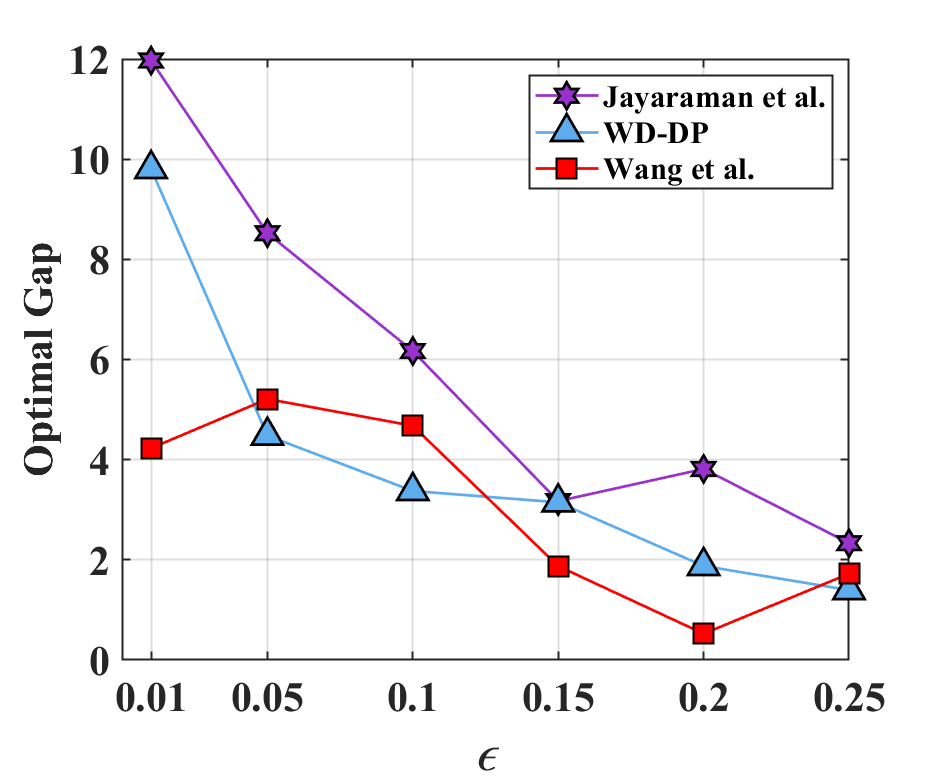}}
\caption{Optimal gap on data sets over privacy budget $\epsilon$. $m=2$, data instances owned by each client is not the same.}
\end{figure*}

\begin{figure*}[htb]
\centering
\subfigure[KDDCup99]{\includegraphics[width=0.3\textwidth]{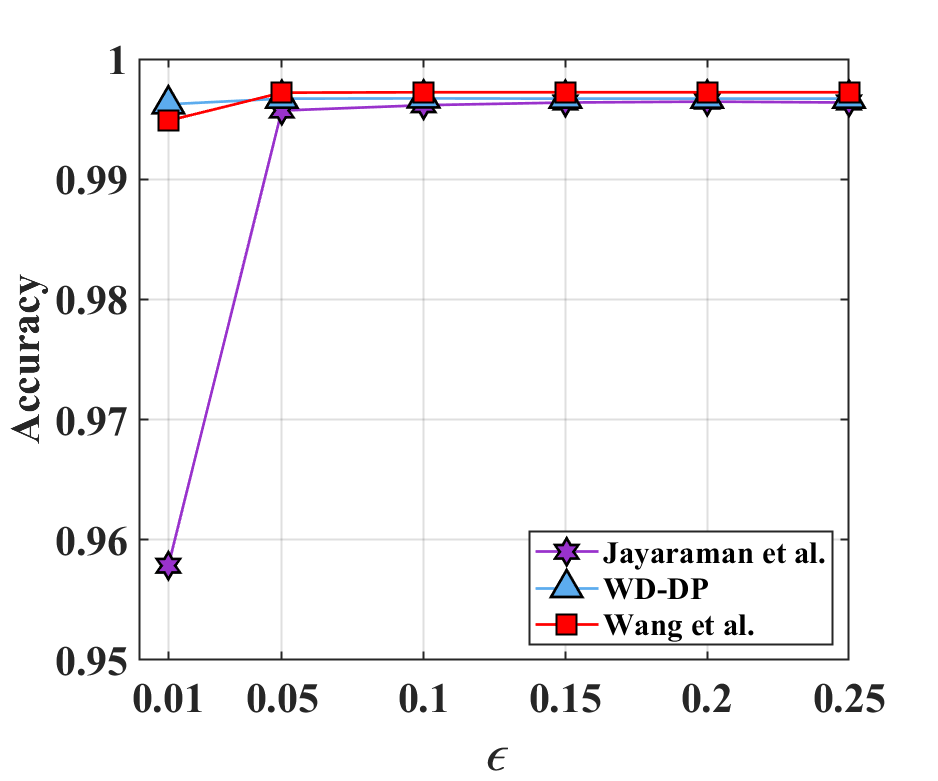}}
\subfigure[Adult]{\includegraphics[width=0.3\textwidth]{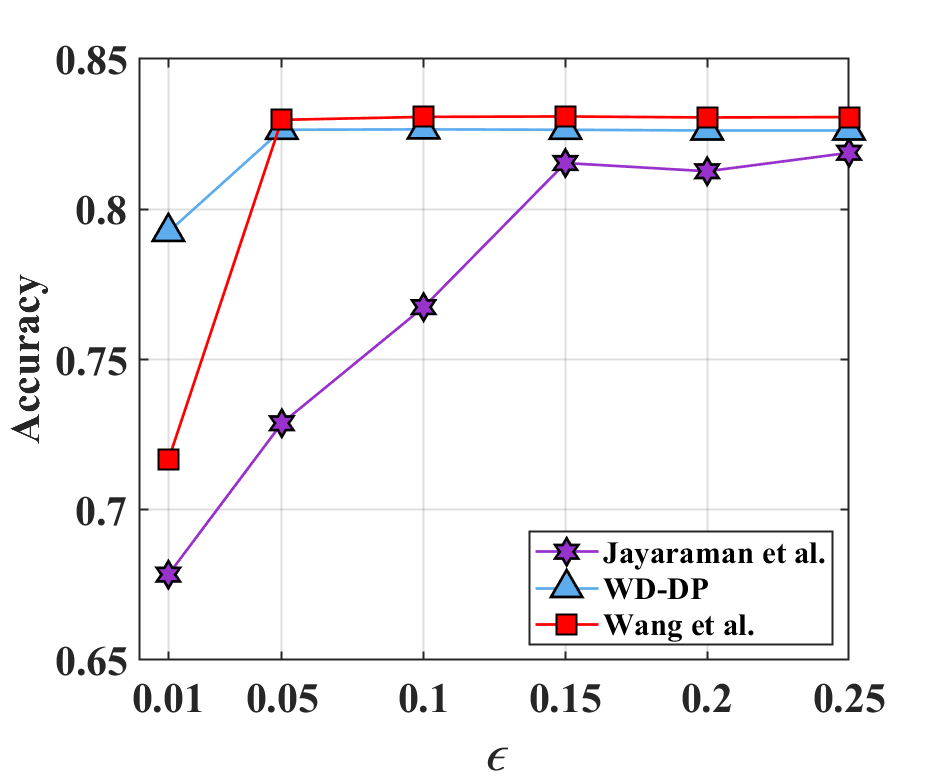}}
\subfigure[Bank]{\includegraphics[width=0.3\textwidth]{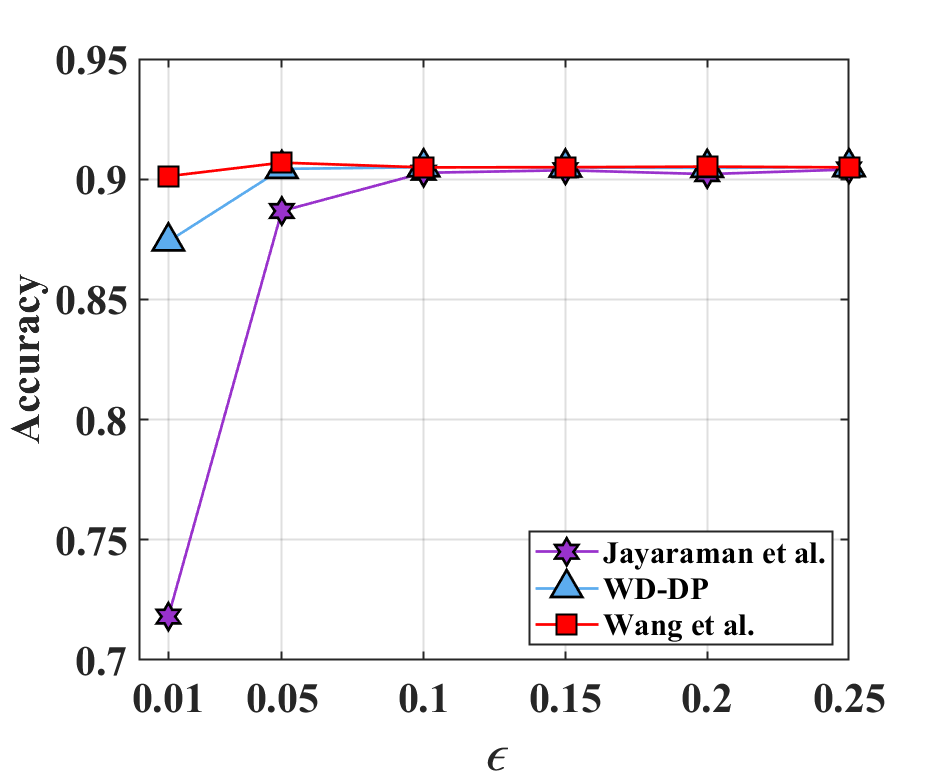}}
\subfigure[Breast Cancer]{\includegraphics[width=0.3\textwidth]{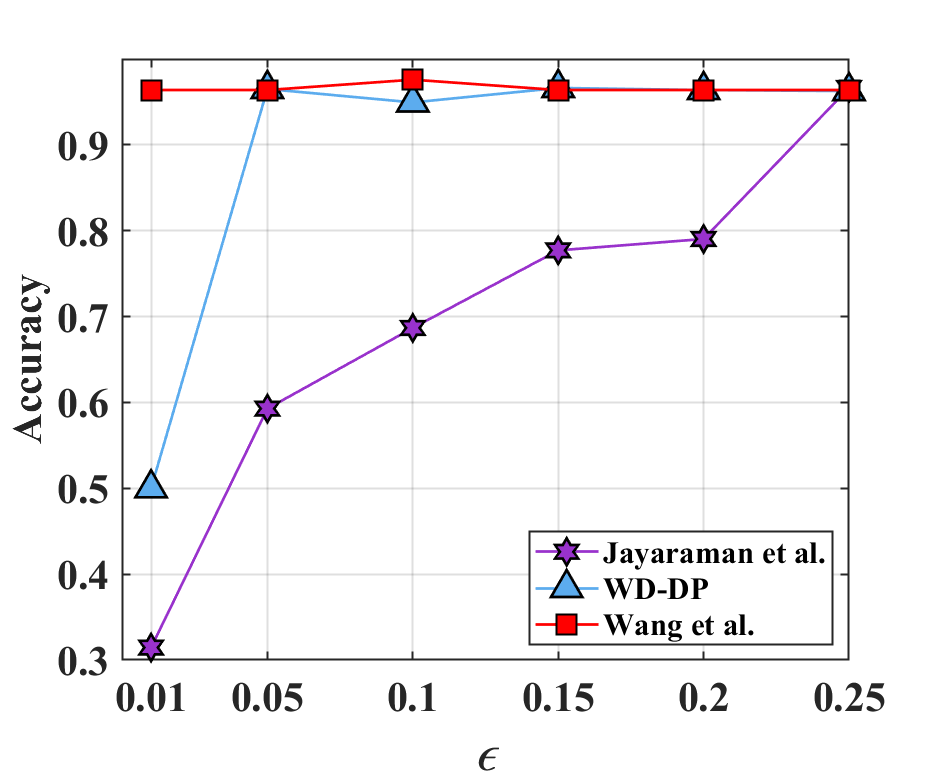}}
\subfigure[Credit Card Fraud]{\includegraphics[width=0.3\textwidth]{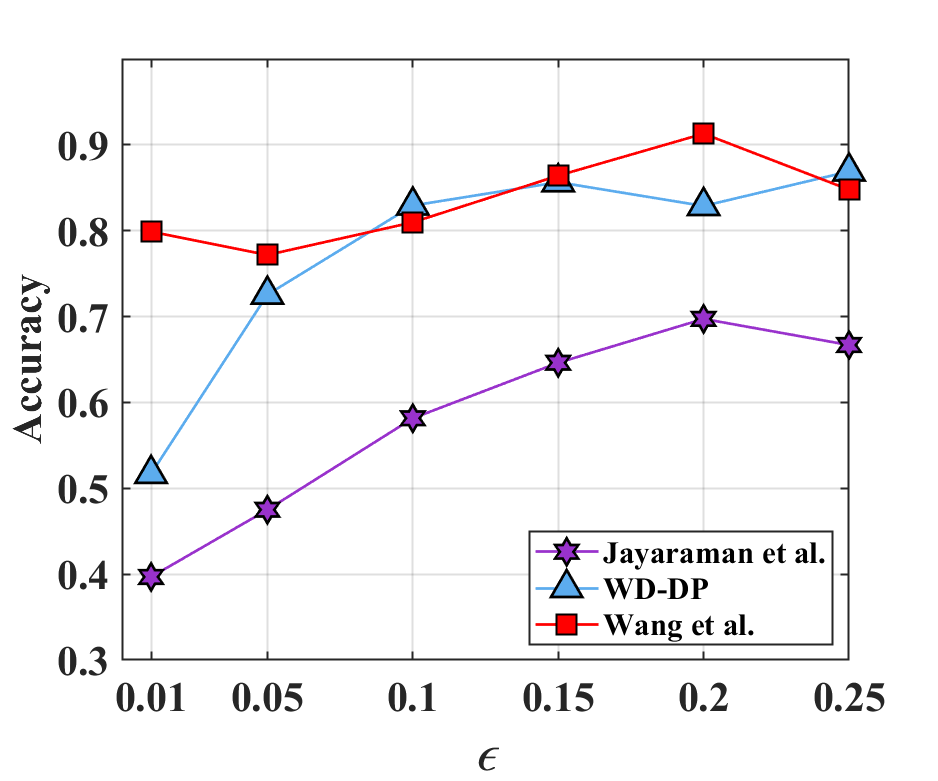}}
\caption{Accuracy on data sets over privacy budget $\epsilon$. $m=4$, data instances owned by each client is not the same.}
\end{figure*}

\begin{figure*}[htb]
\centering
\subfigure[KDDCup99]{\includegraphics[width=0.3\textwidth]{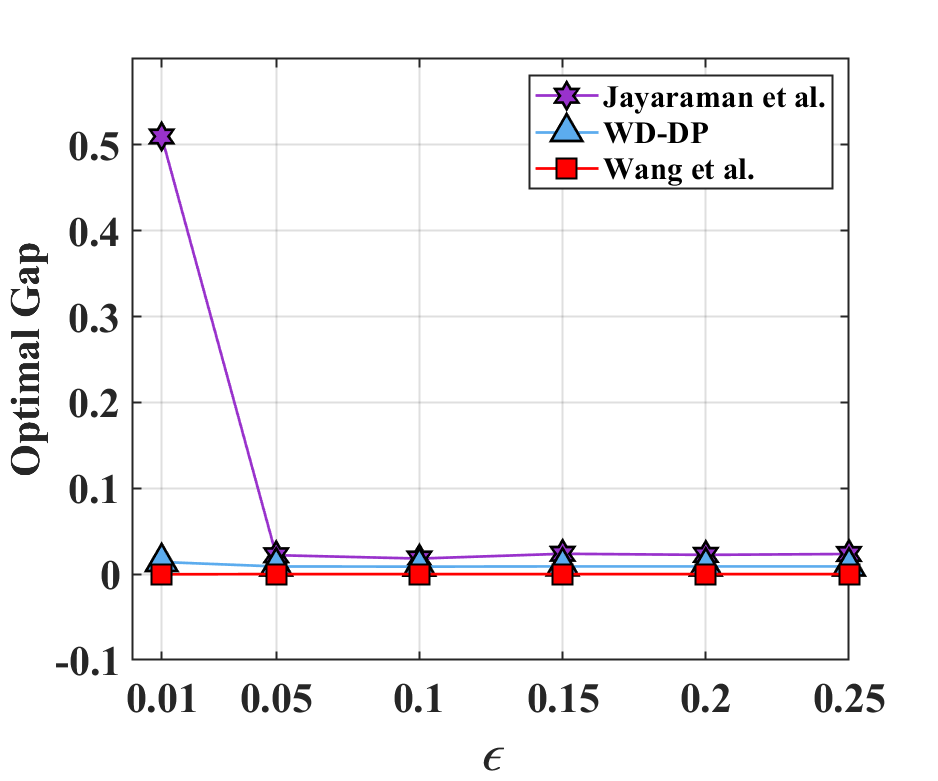}}
\subfigure[Adult]{\includegraphics[width=0.3\textwidth]{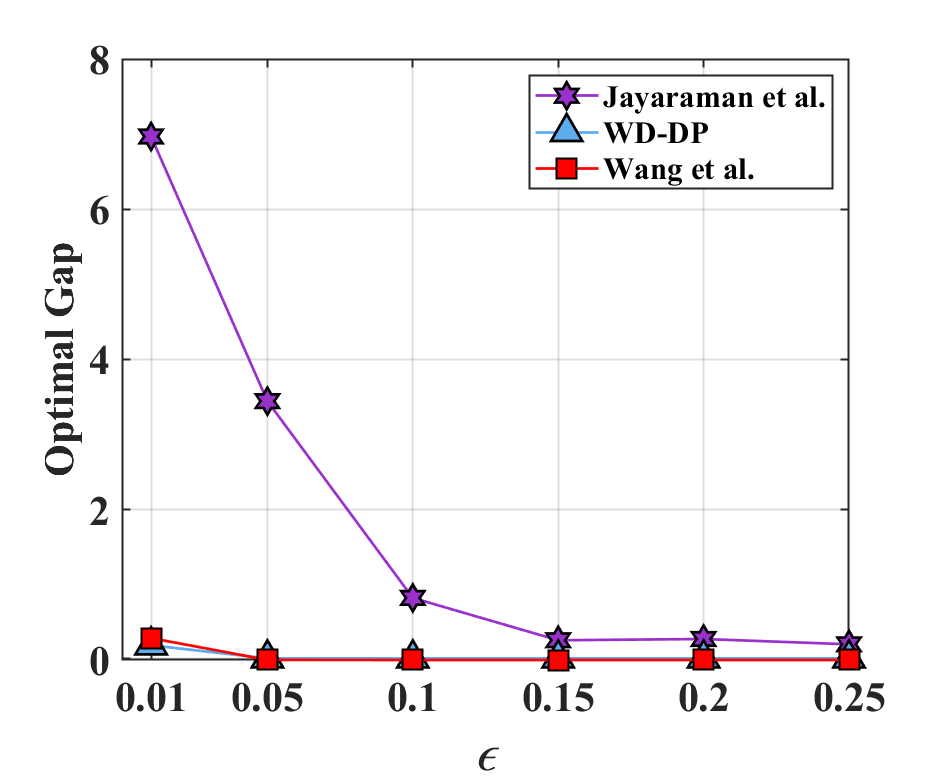}}
\subfigure[Bank]{\includegraphics[width=0.3\textwidth]{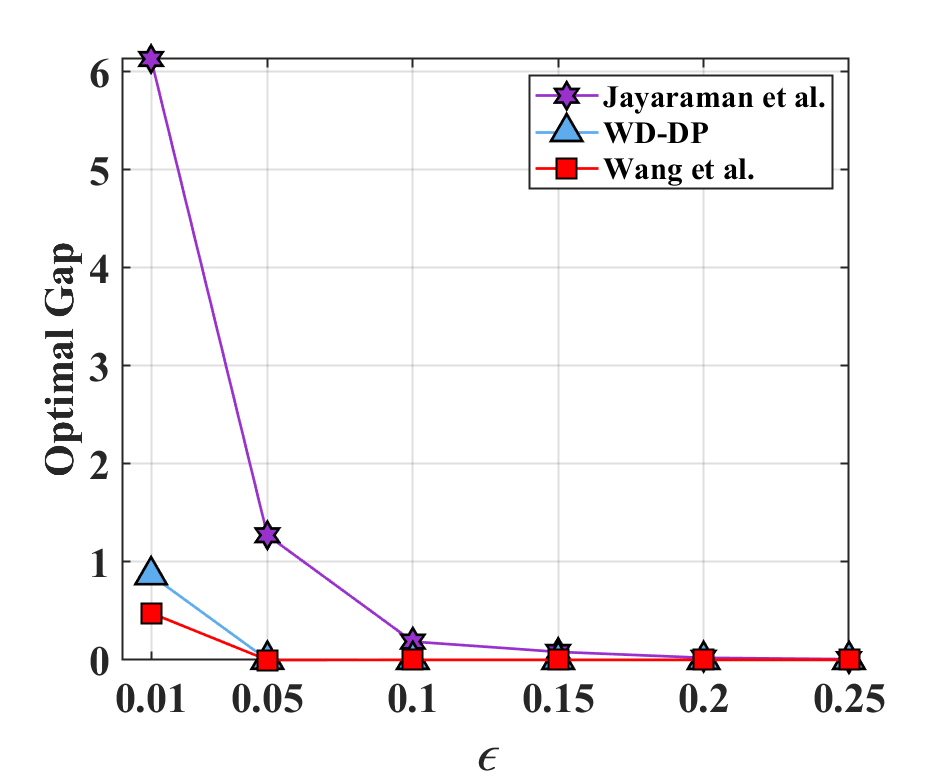}}
\subfigure[Breast Cancer]{\includegraphics[width=0.3\textwidth]{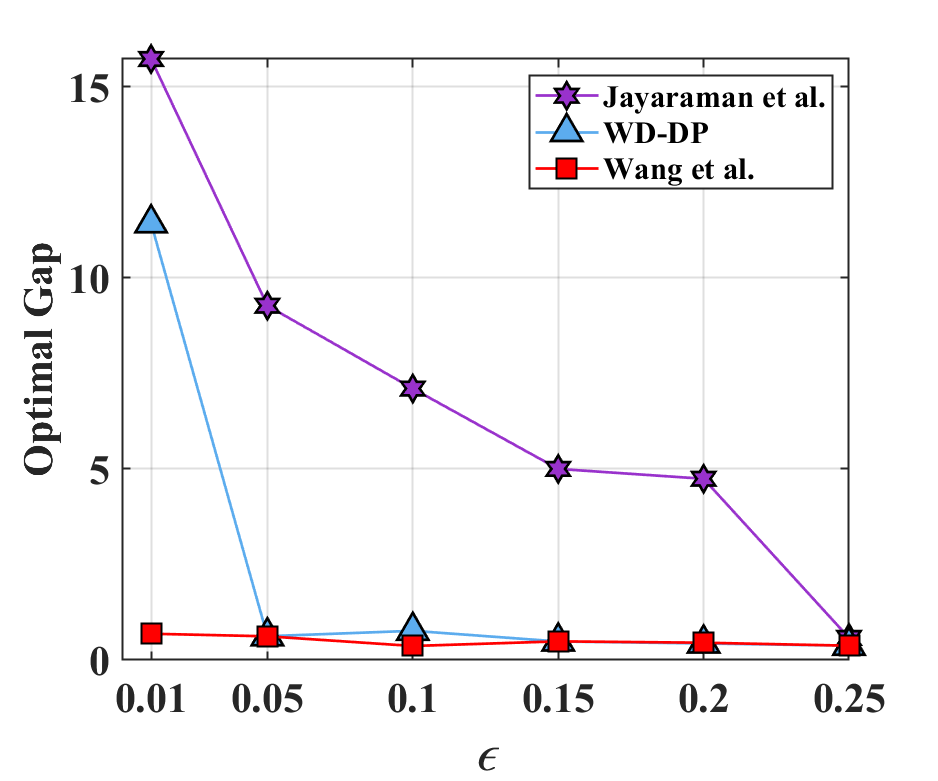}}
\subfigure[Credit Card Fraud]{\includegraphics[width=0.3\textwidth]{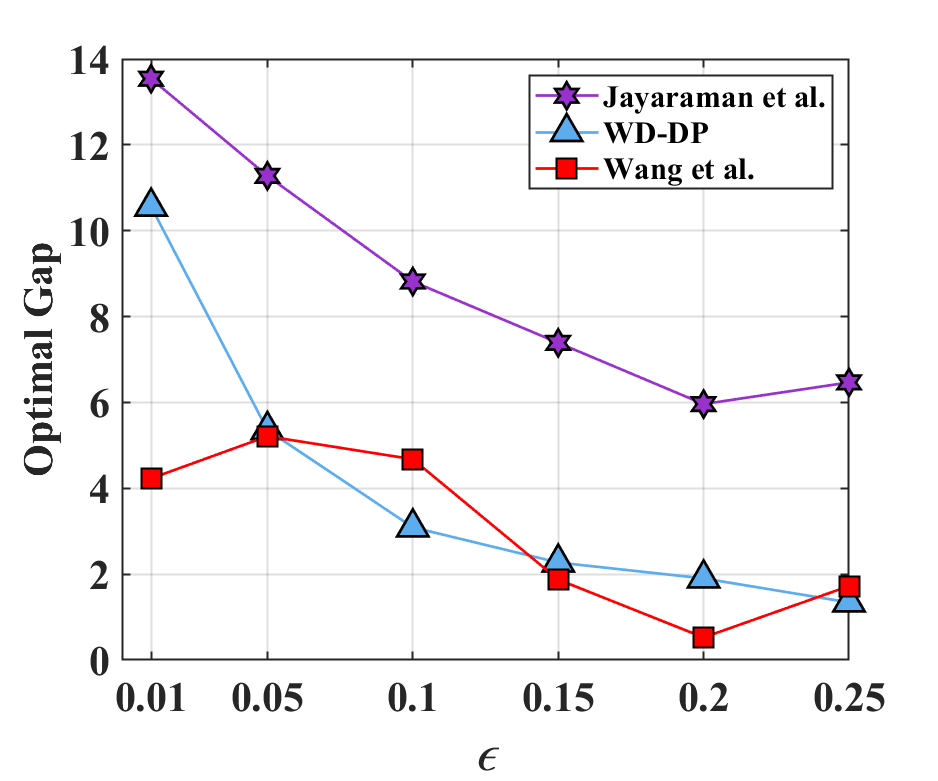}}
\caption{Optimal gap on data sets over privacy budget $\epsilon$. $m=4$, data instances owned by each client is not the same.}
\end{figure*}

\begin{figure*}[htb]
\centering
\subfigure[KDDCup99]{\includegraphics[width=0.3\textwidth]{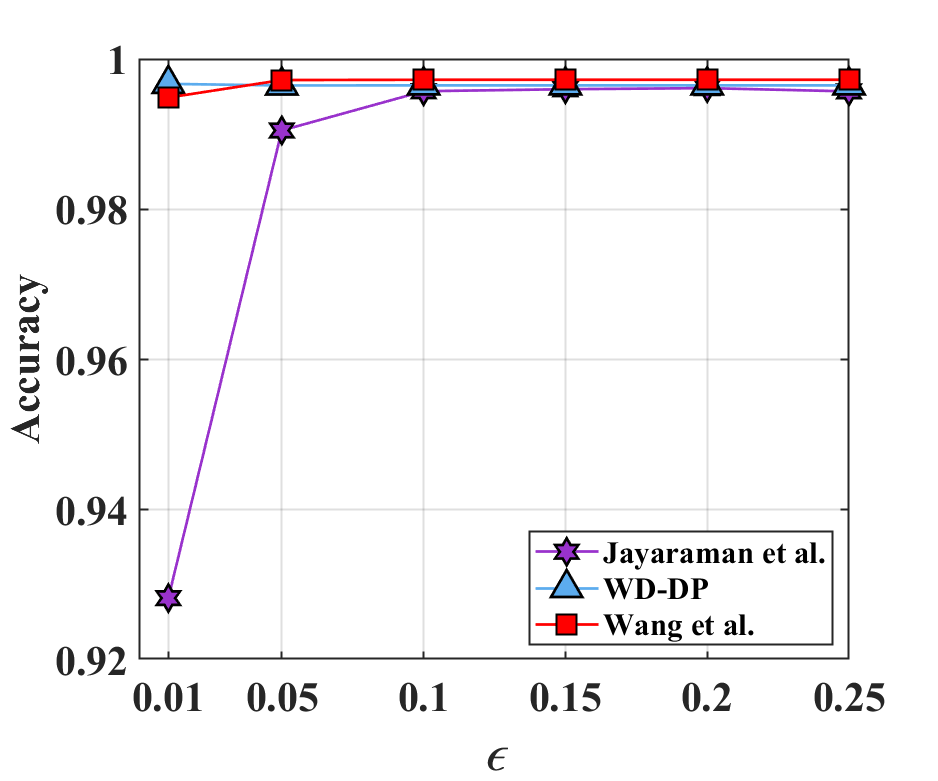}}
\subfigure[Adult]{\includegraphics[width=0.3\textwidth]{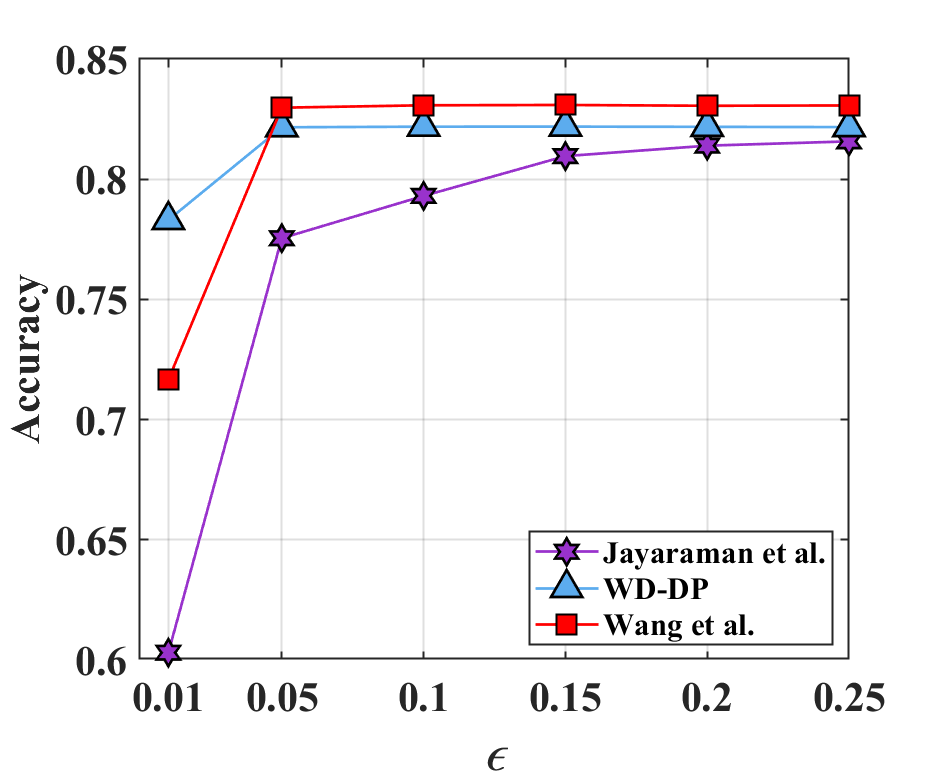}}
\subfigure[Bank]{\includegraphics[width=0.3\textwidth]{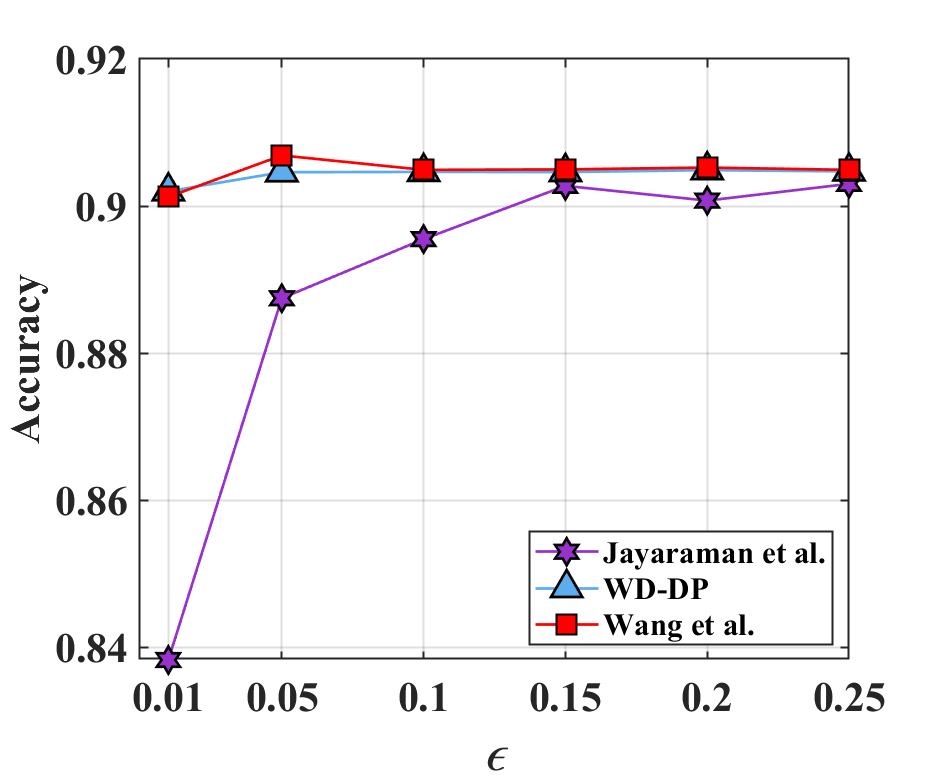}}
\subfigure[Breast Cancer]{\includegraphics[width=0.3\textwidth]{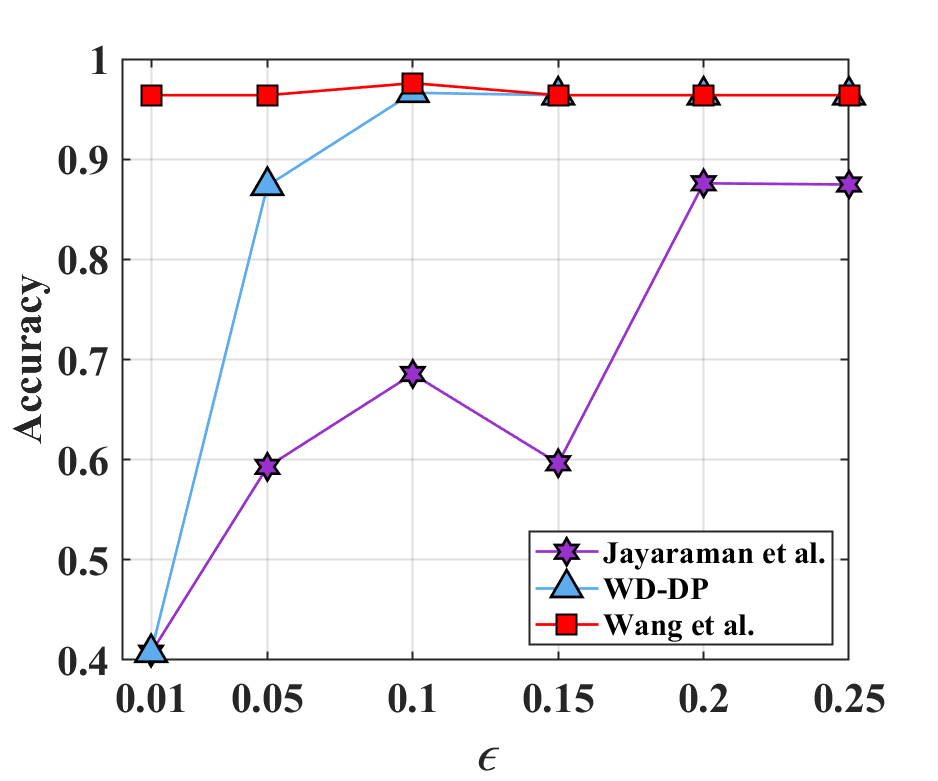}}
\subfigure[Credit Card Fraud]{\includegraphics[width=0.3\textwidth]{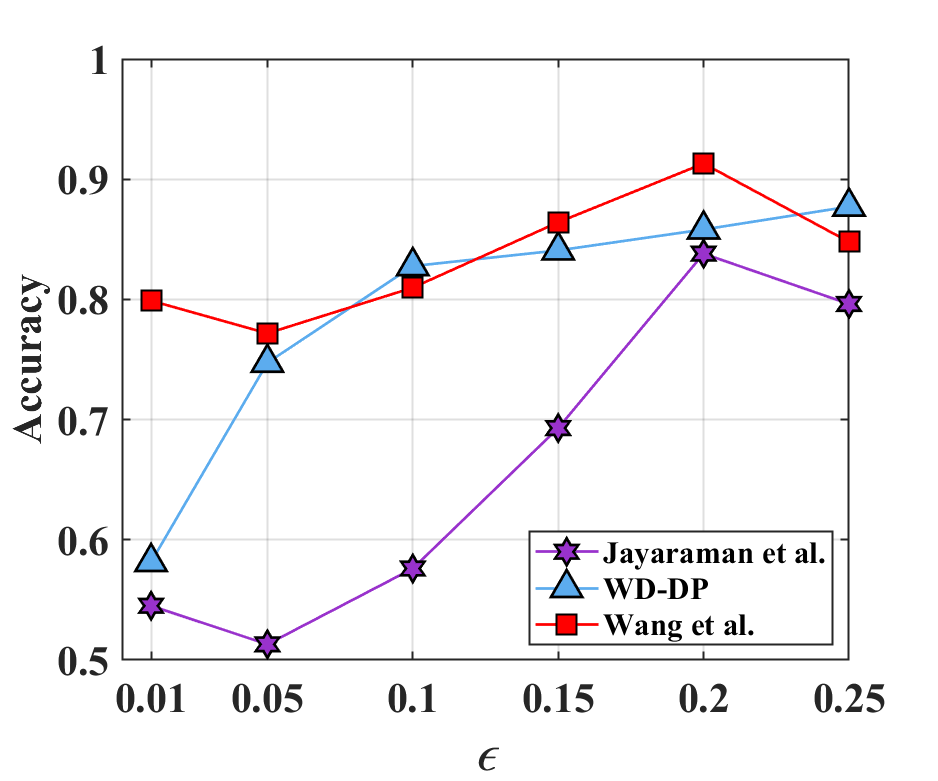}}
\caption{Accuracy on data sets over privacy budget $\epsilon$. $m=8$, data instances owned by each client is not the same.}
\end{figure*}

\begin{figure*}[htb]
\centering
\subfigure[KDDCup99]{\includegraphics[width=0.3\textwidth]{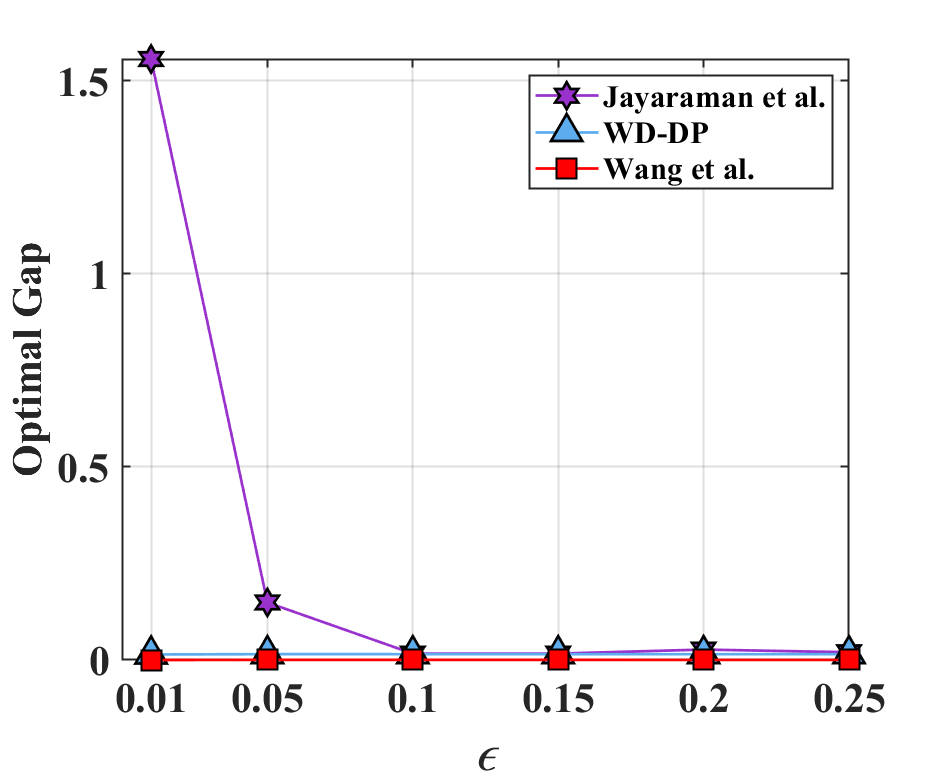}}
\subfigure[Adult]{\includegraphics[width=0.3\textwidth]{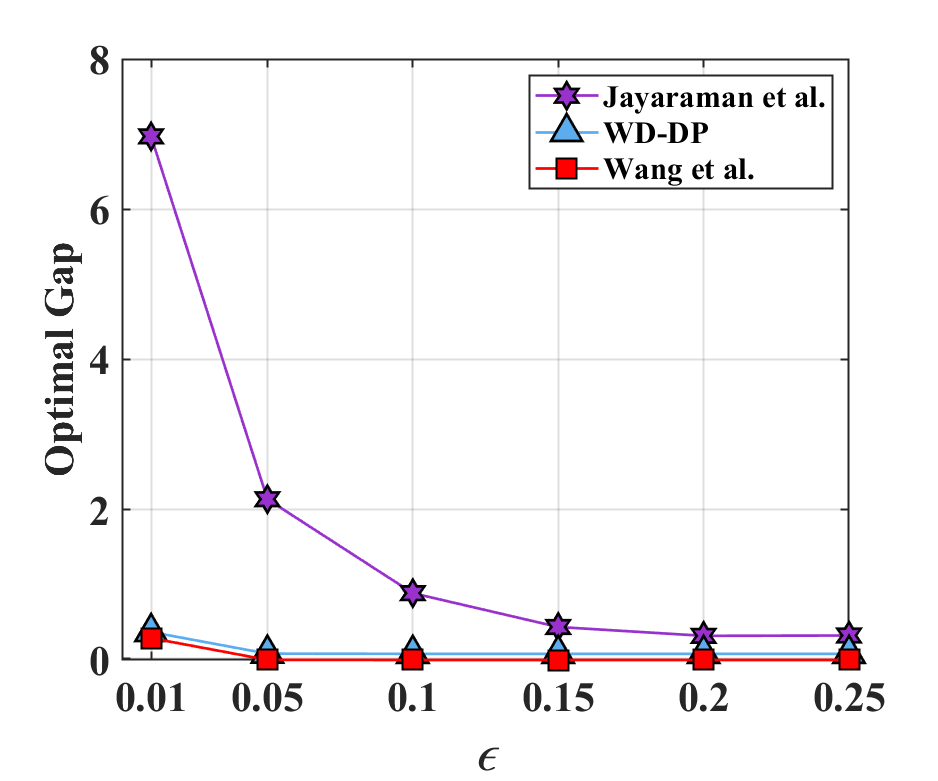}}
\subfigure[Bank]{\includegraphics[width=0.3\textwidth]{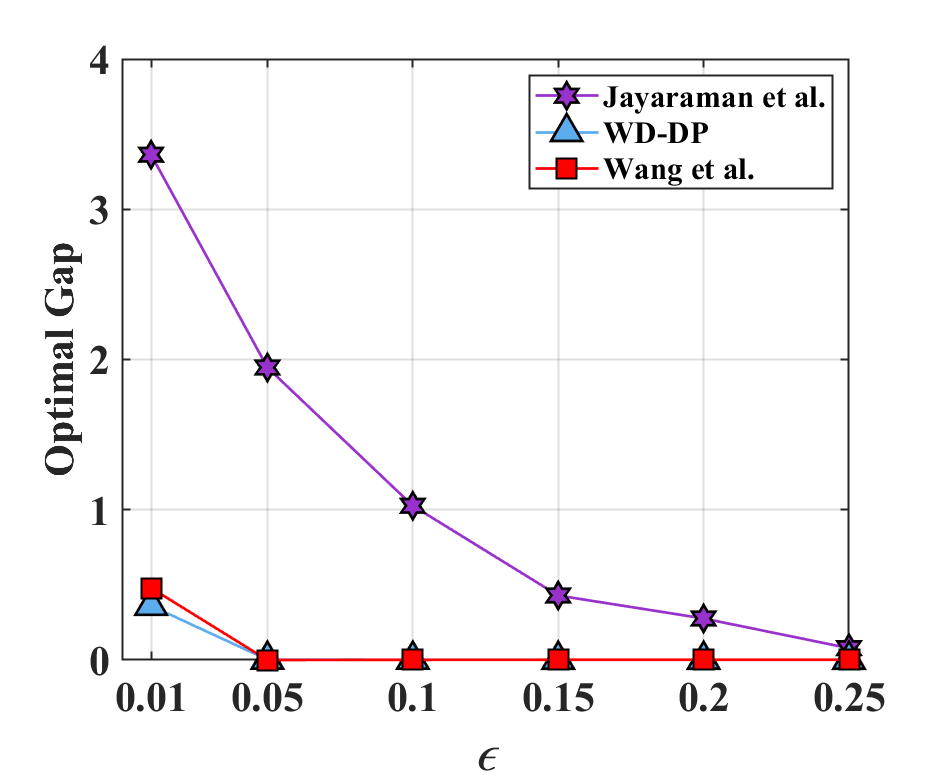}}
\subfigure[Breast Cancer]{\includegraphics[width=0.3\textwidth]{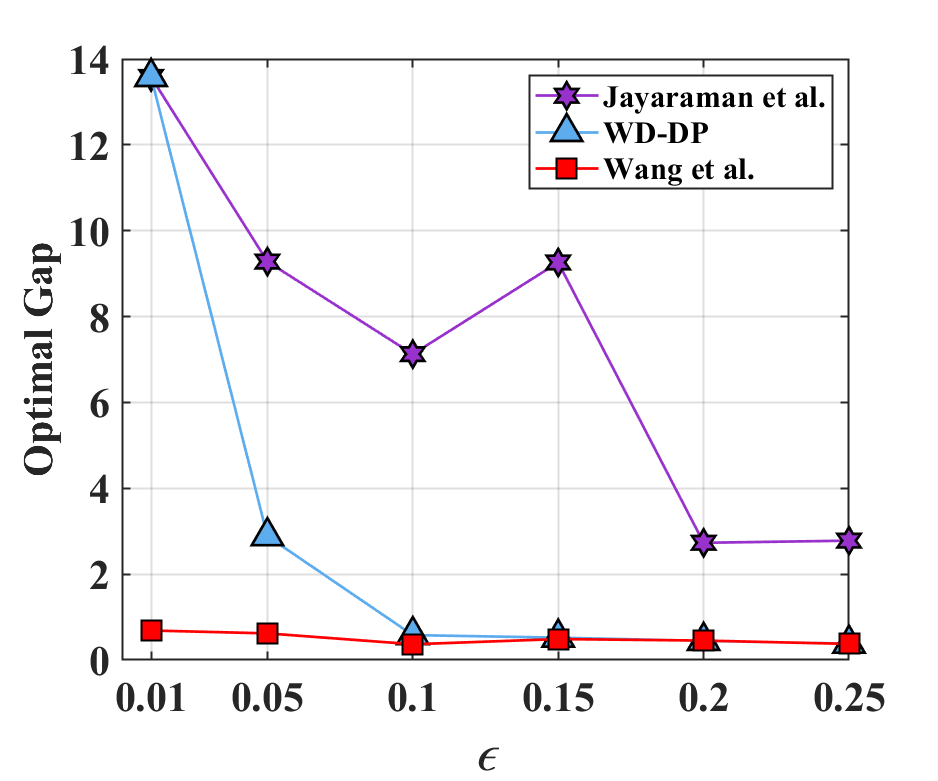}}
\subfigure[Credit Card Fraud]{\includegraphics[width=0.3\textwidth]{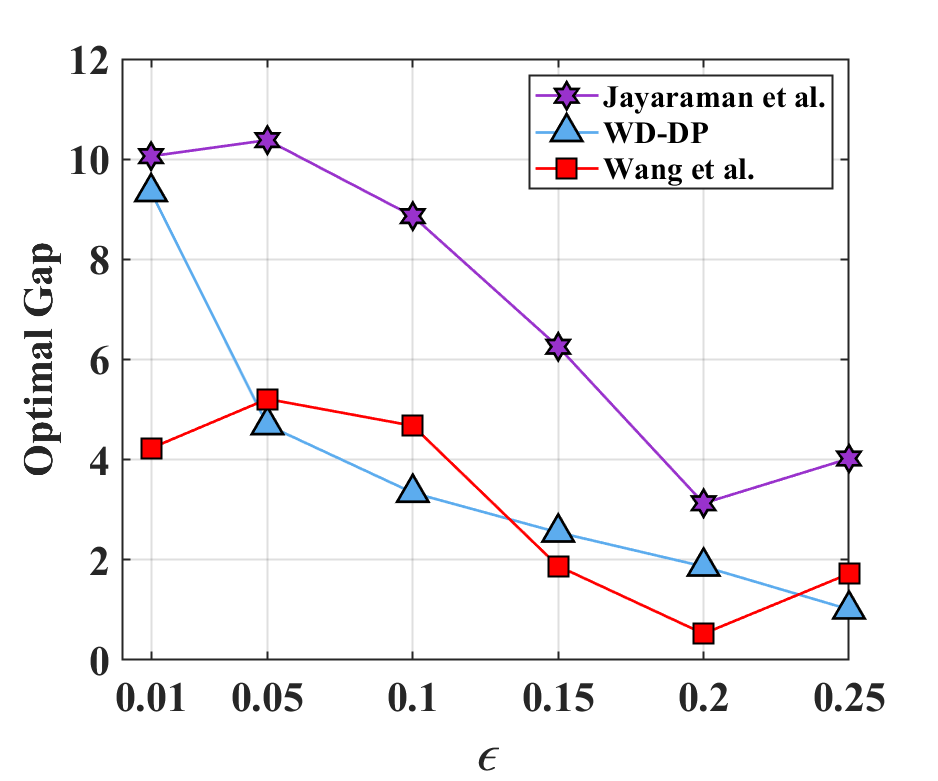}}
\caption{Optimal gap on data sets over privacy budget $\epsilon$. $m=8$, data instances owned by each client is not the same.}
\end{figure*}

\begin{figure*}[htb]
\centering
\subfigure[KDDCup99]{\includegraphics[width=0.3\textwidth]{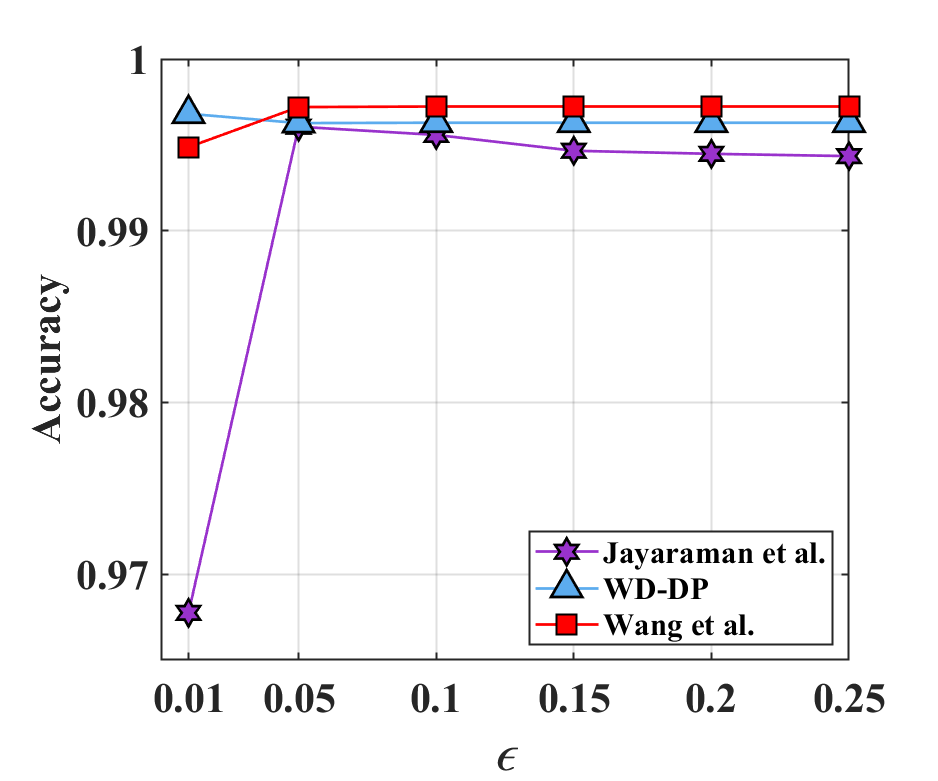}}
\subfigure[Adult]{\includegraphics[width=0.3\textwidth]{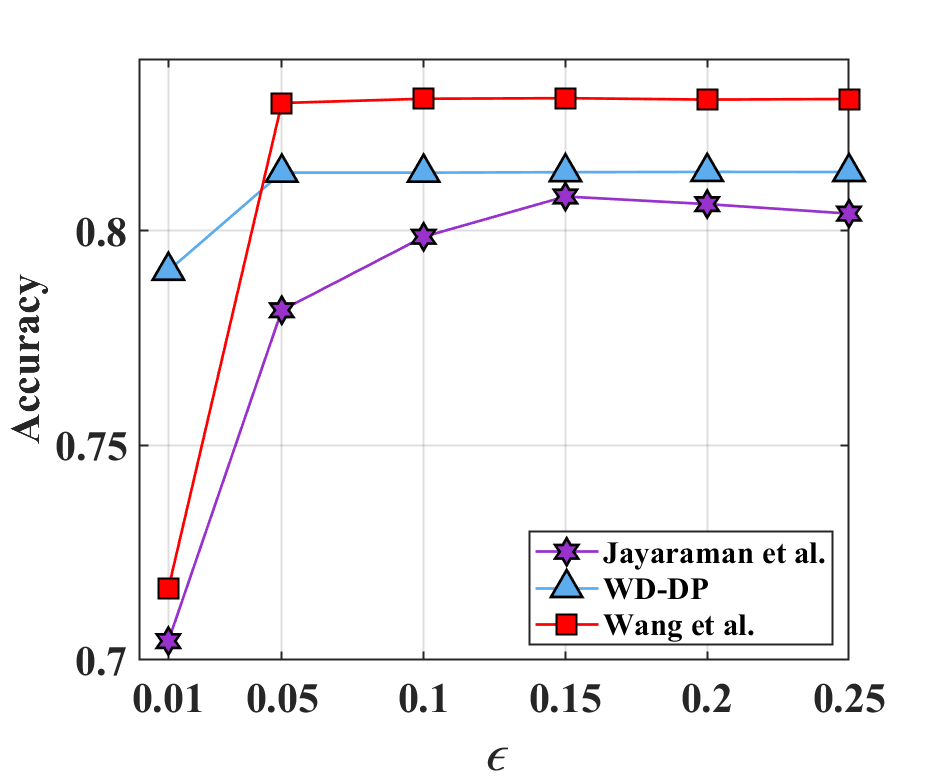}}
\subfigure[Bank]{\includegraphics[width=0.3\textwidth]{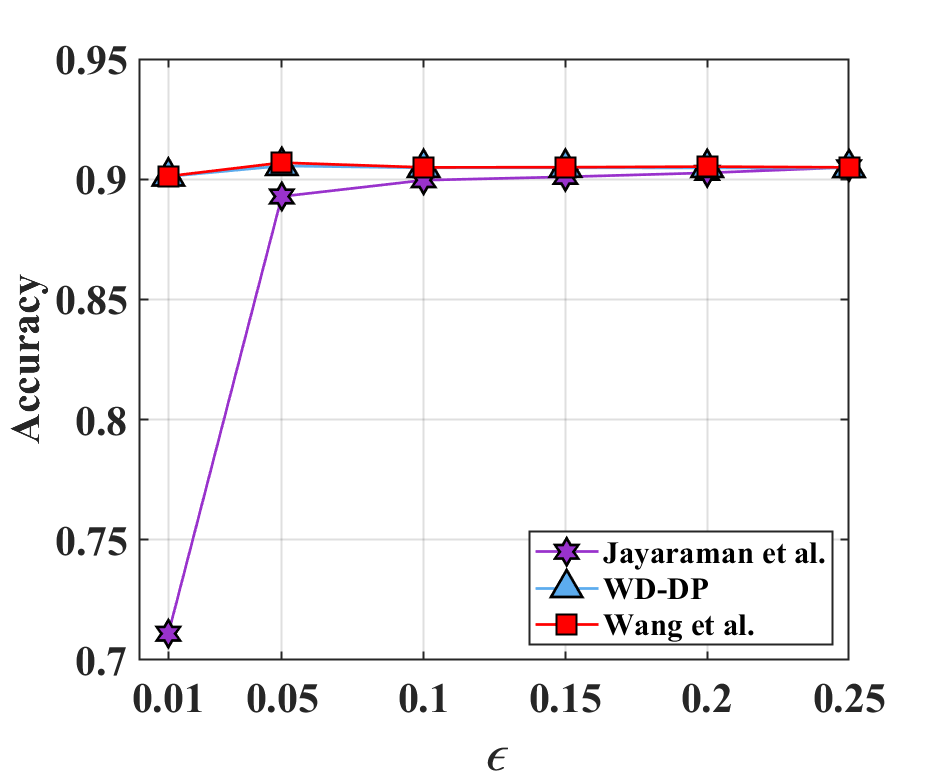}}
\subfigure[Breast Cancer]{\includegraphics[width=0.3\textwidth]{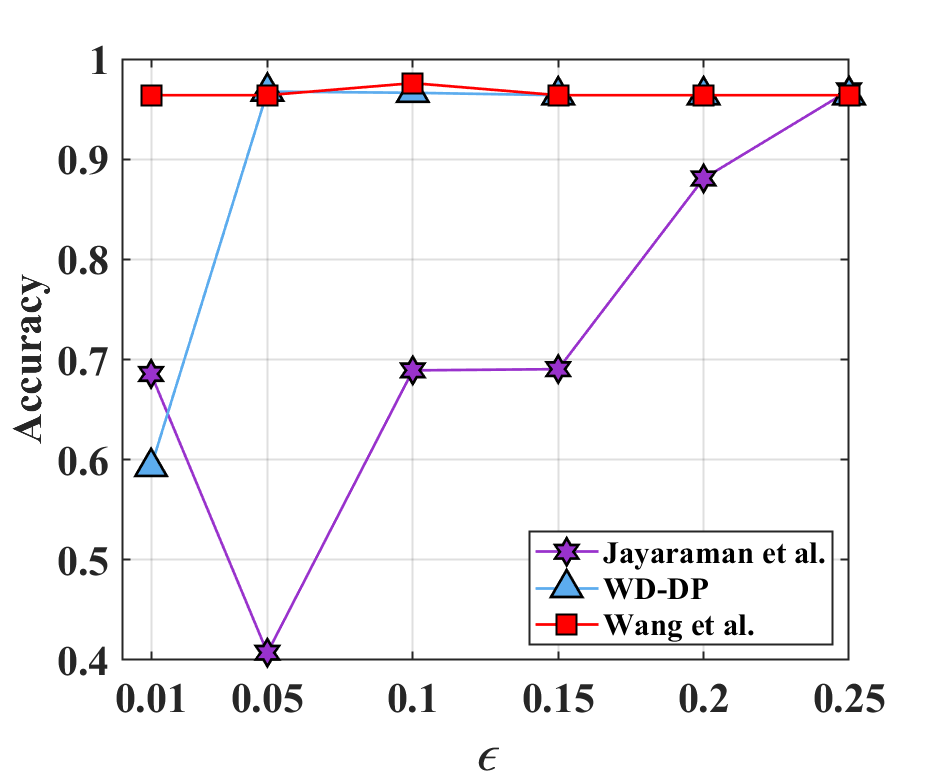}}
\subfigure[Credit Card Fraud]{\includegraphics[width=0.3\textwidth]{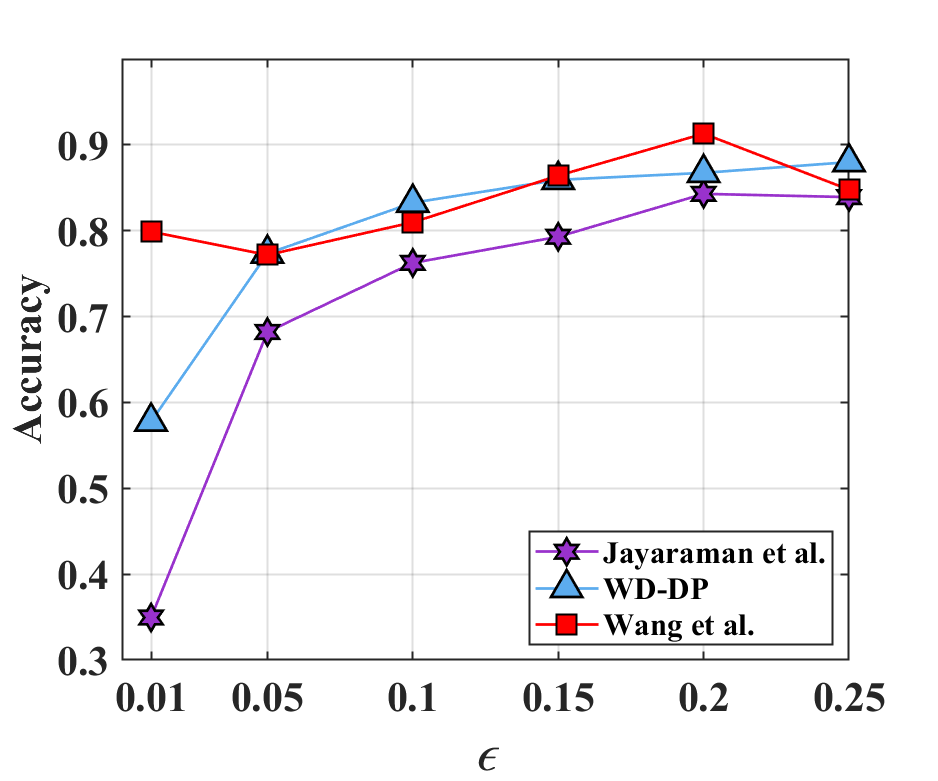}}
\caption{Accuracy on data sets over privacy budget $\epsilon$. $m=16$, data instances owned by each client is not the same.}
\end{figure*}

\begin{figure*}[htb]
\centering
\subfigure[KDDCup99]{\includegraphics[width=0.3\textwidth]{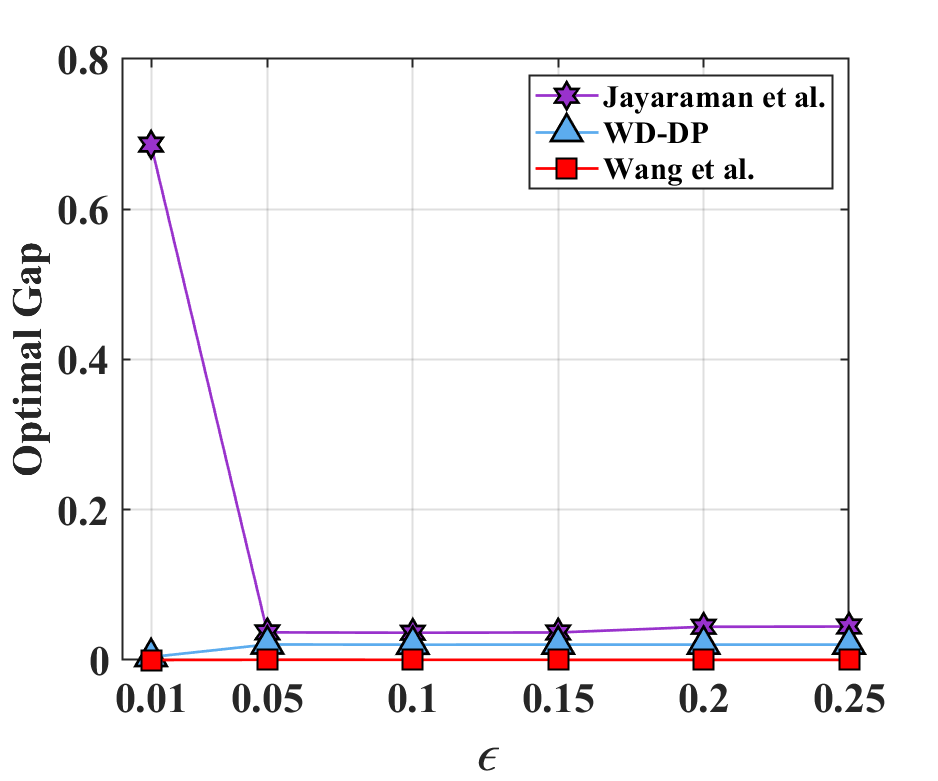}}
\subfigure[Adult]{\includegraphics[width=0.3\textwidth]{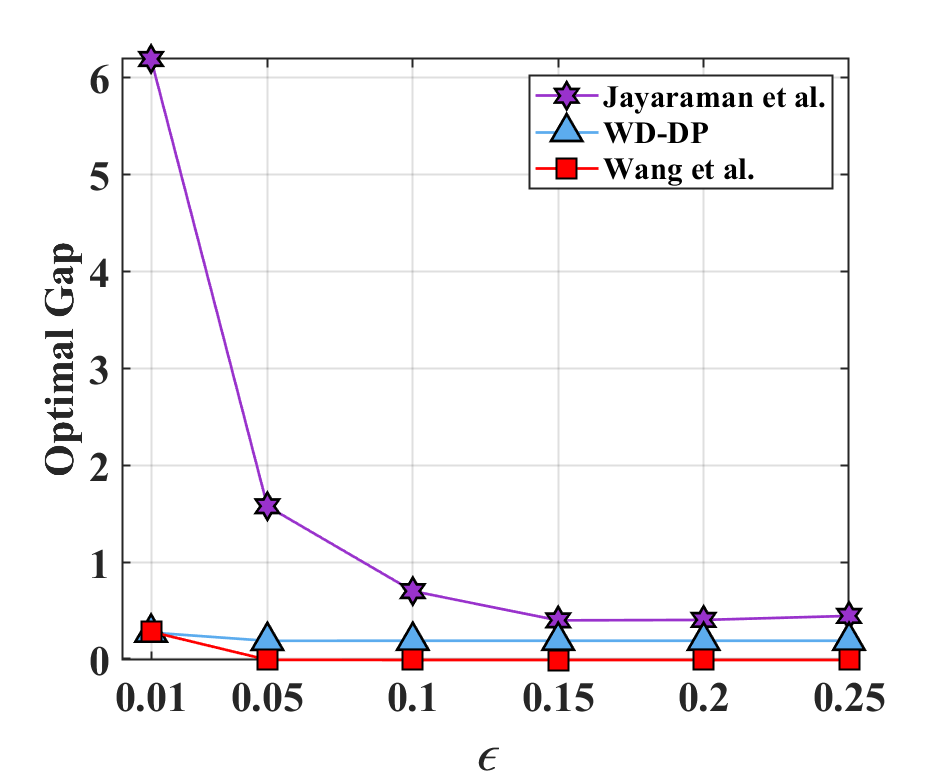}}
\subfigure[Bank]{\includegraphics[width=0.3\textwidth]{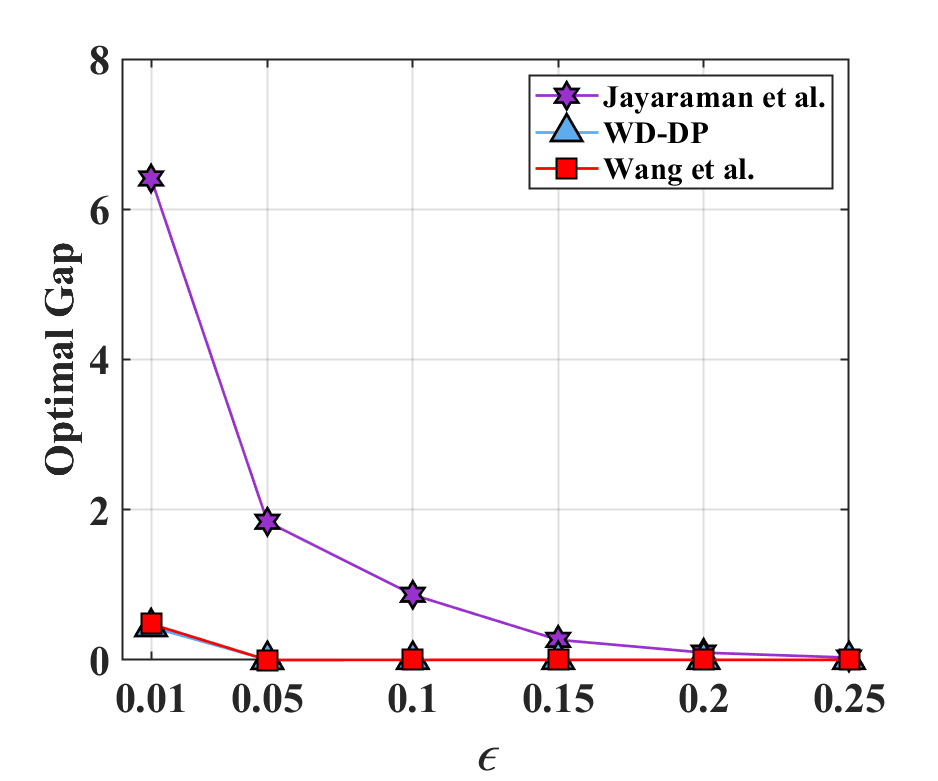}}
\subfigure[Breast Cancer]{\includegraphics[width=0.3\textwidth]{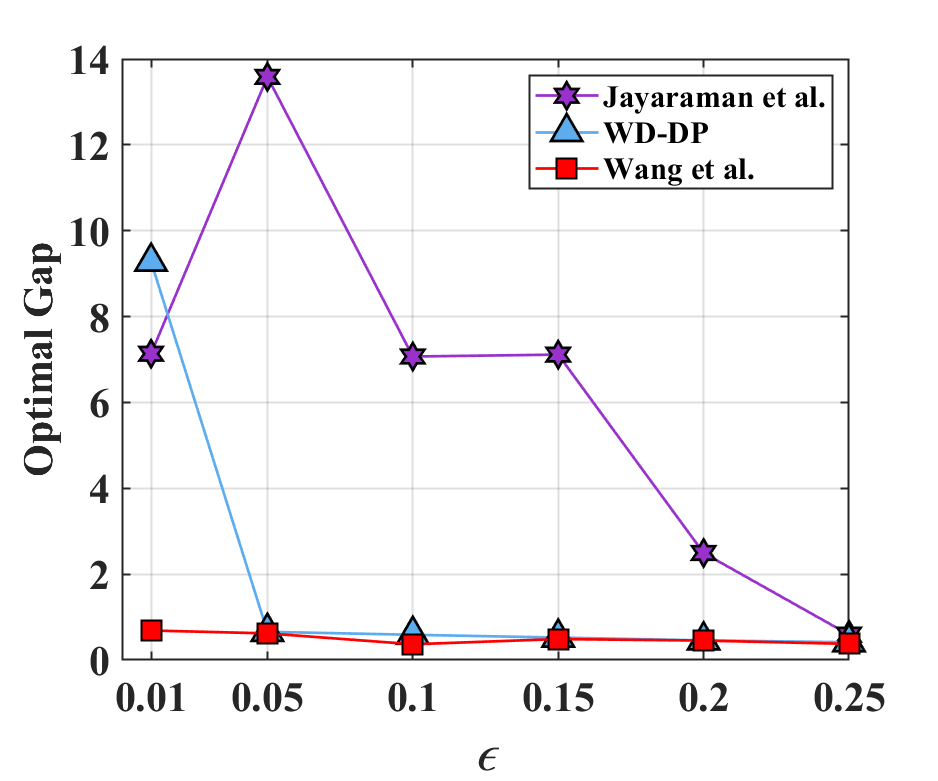}}
\subfigure[Credit Card Fraud]{\includegraphics[width=0.3\textwidth]{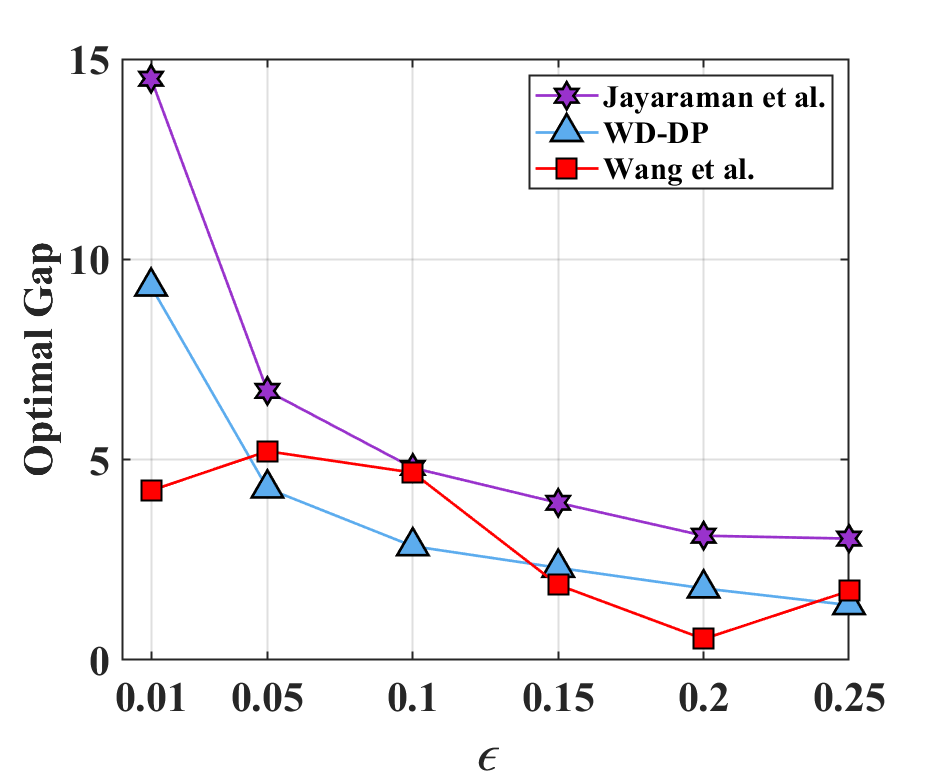}}
\caption{Optimal gap on data sets over privacy budget $\epsilon$. $m=16$, data instances owned by each client is not the same.}
\end{figure*}

\begin{figure*}[htb]
\centering
\subfigure[KDDCup99]{\includegraphics[width=0.3\textwidth]{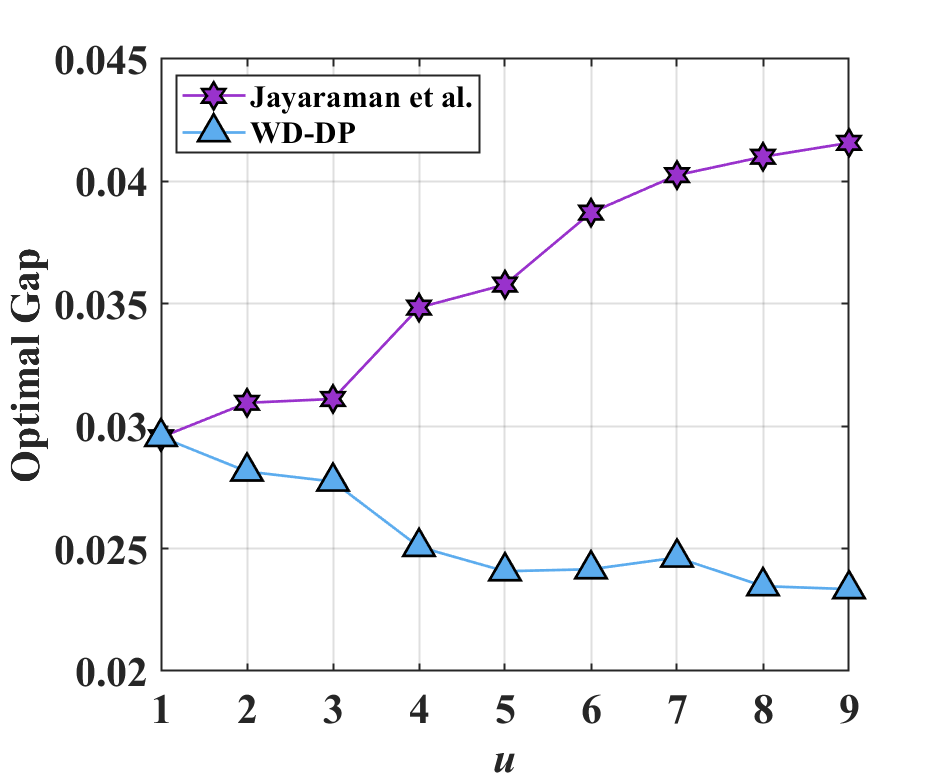}}
\subfigure[Adult]{\includegraphics[width=0.3\textwidth]{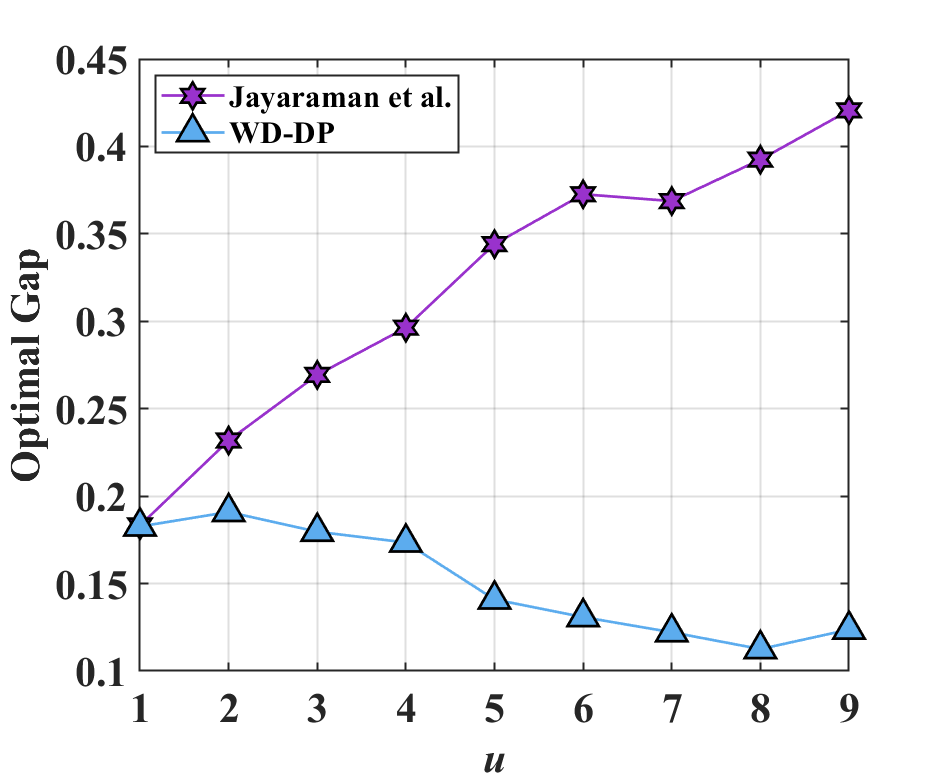}}
\subfigure[Bank]{\includegraphics[width=0.3\textwidth]{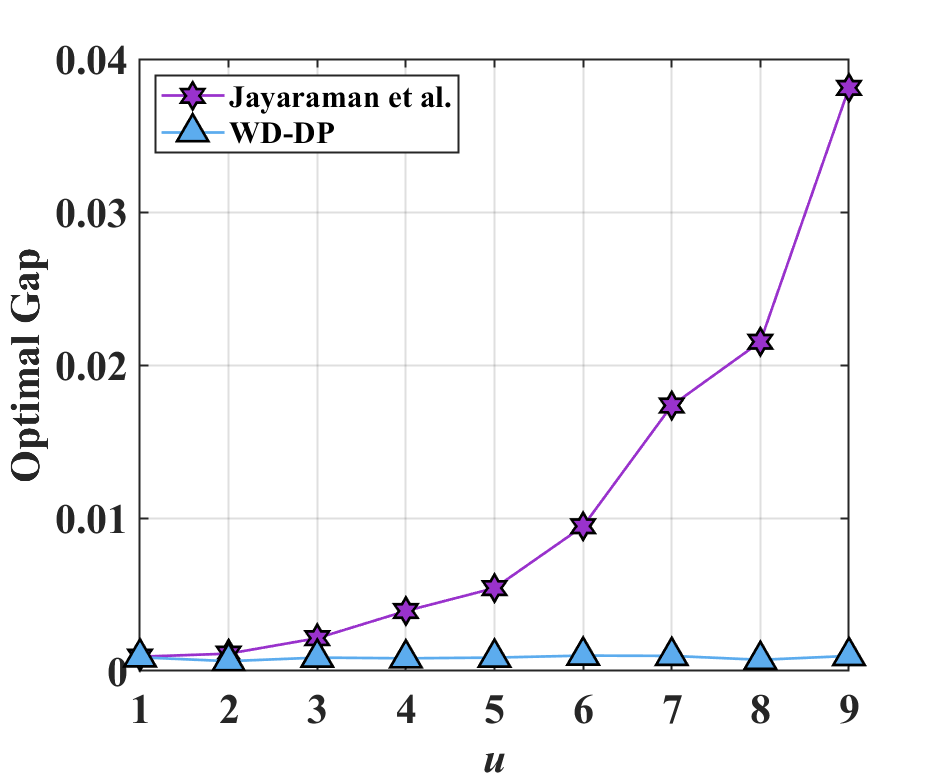}}
\subfigure[Breast Cancer]{\includegraphics[width=0.3\textwidth]{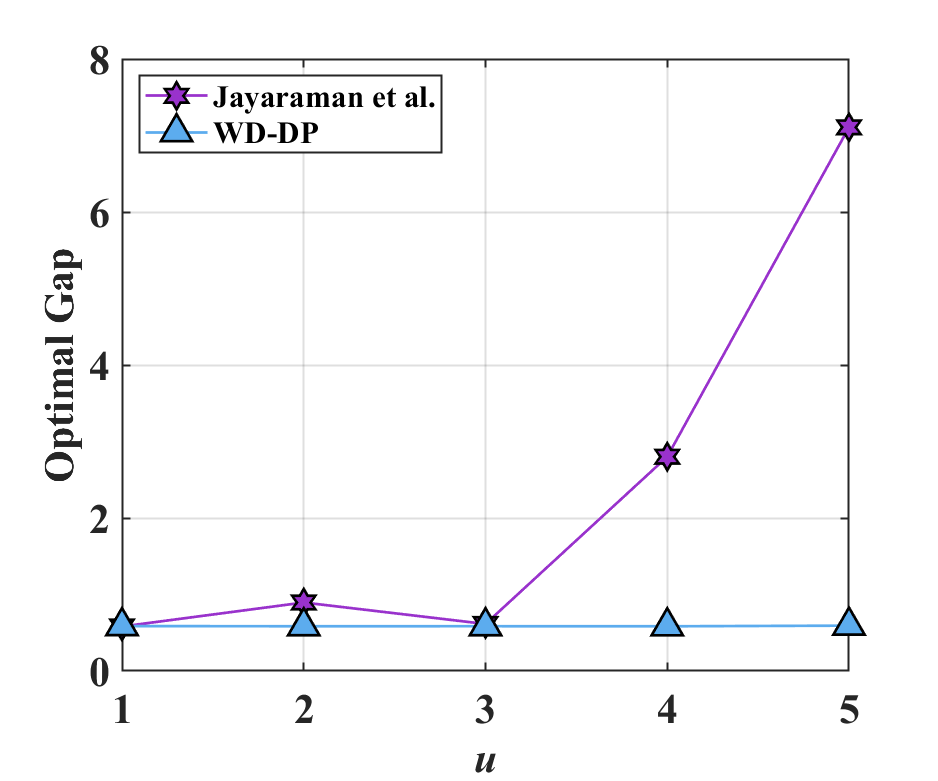}}
\subfigure[Credit Card Fraud]{\includegraphics[width=0.3\textwidth]{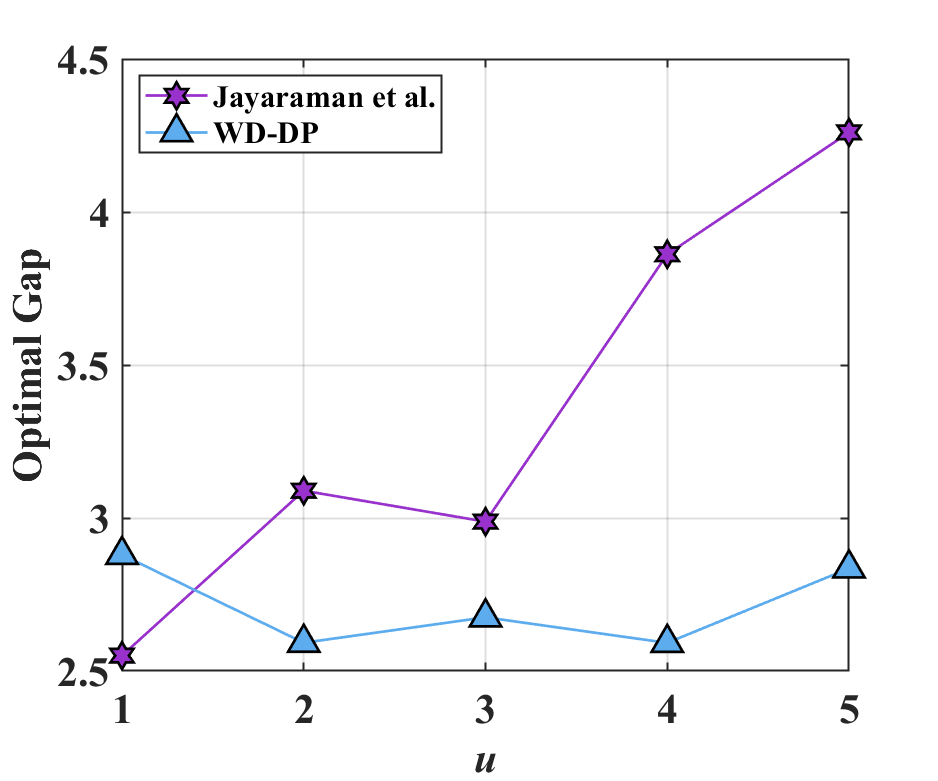}}
\caption{Optimal gap on data sets over the level of non-average $u$, with $m=16$, $\epsilon=0.05$.}
\end{figure*}

\begin{figure*}[htb]
\centering
\subfigure[KDDCup99]{\includegraphics[width=0.3\textwidth]{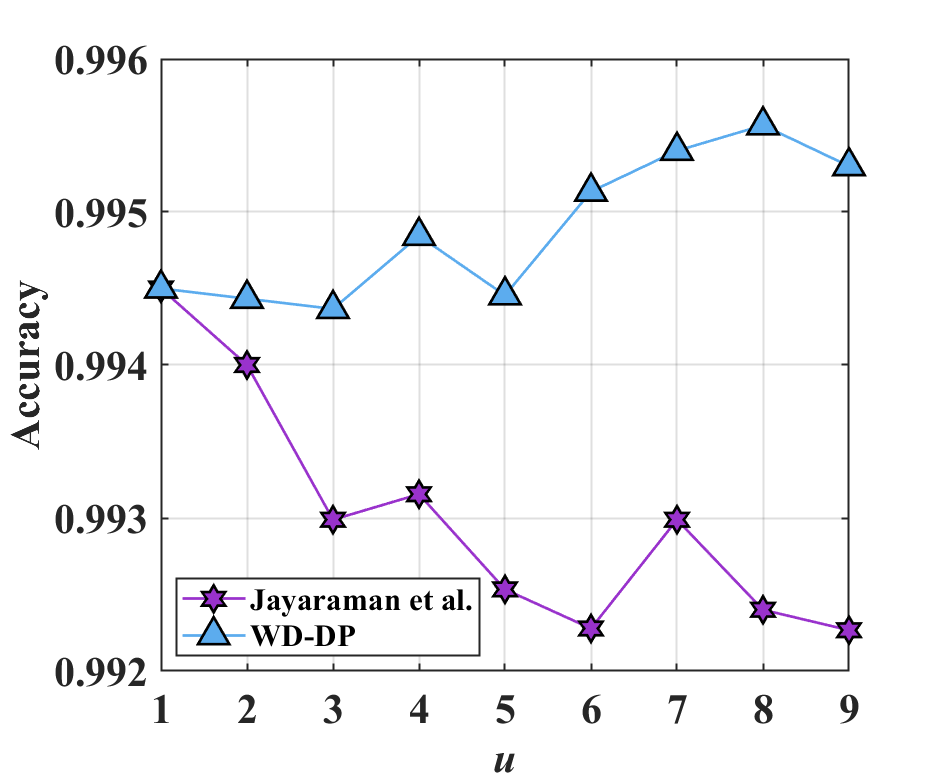}}
\subfigure[Adult]{\includegraphics[width=0.3\textwidth]{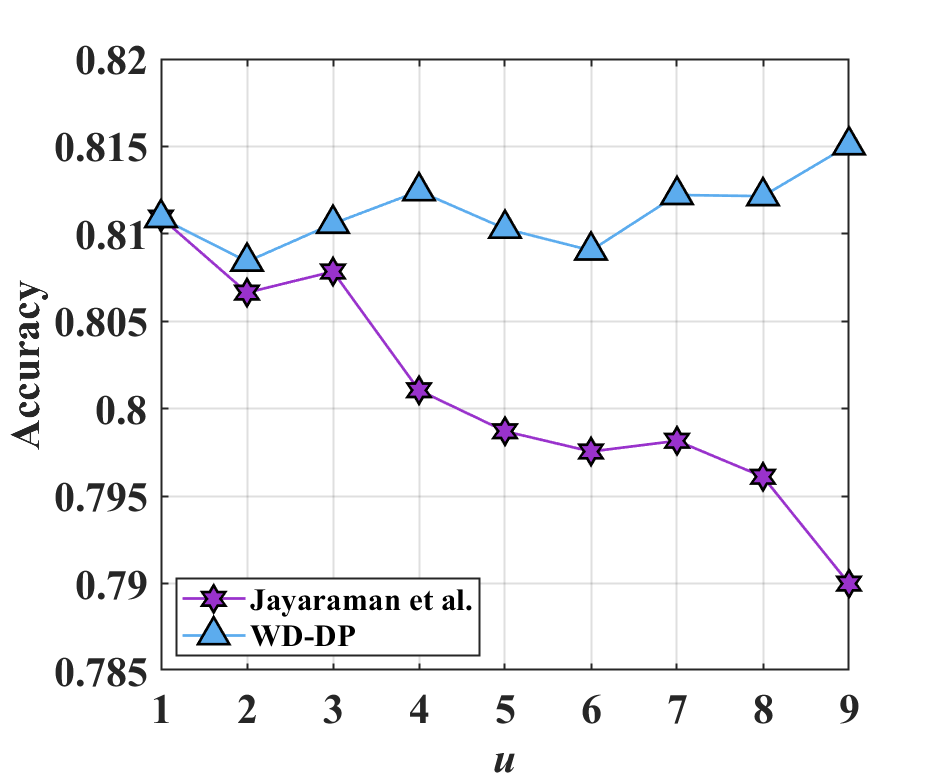}}
\subfigure[Bank]{\includegraphics[width=0.3\textwidth]{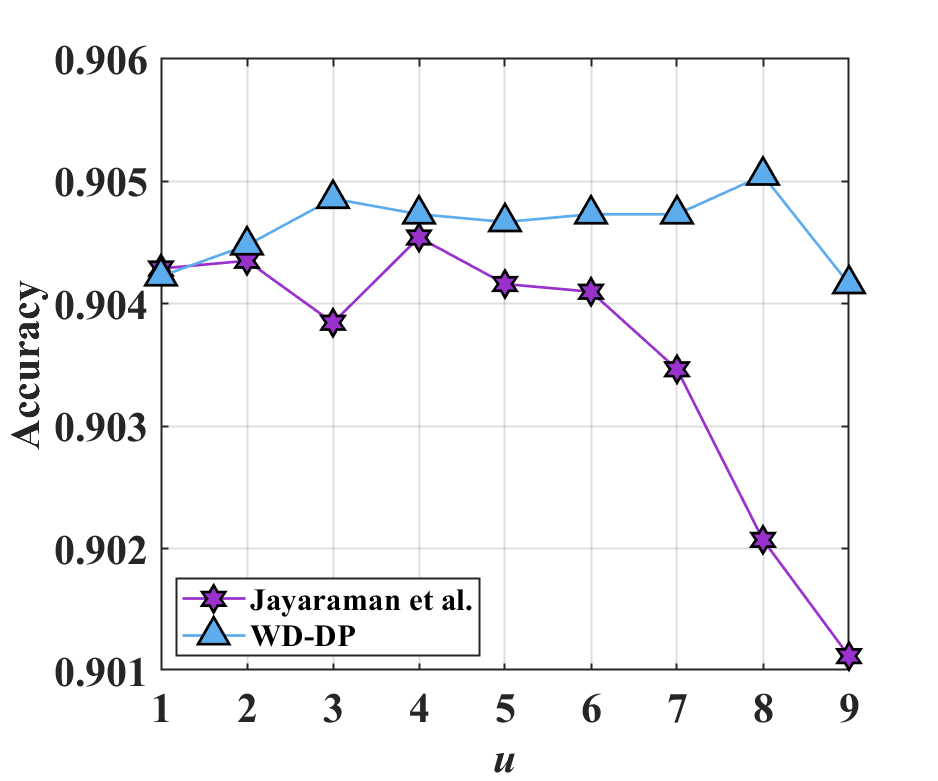}}
\subfigure[Breast Cancer]{\includegraphics[width=0.3\textwidth]{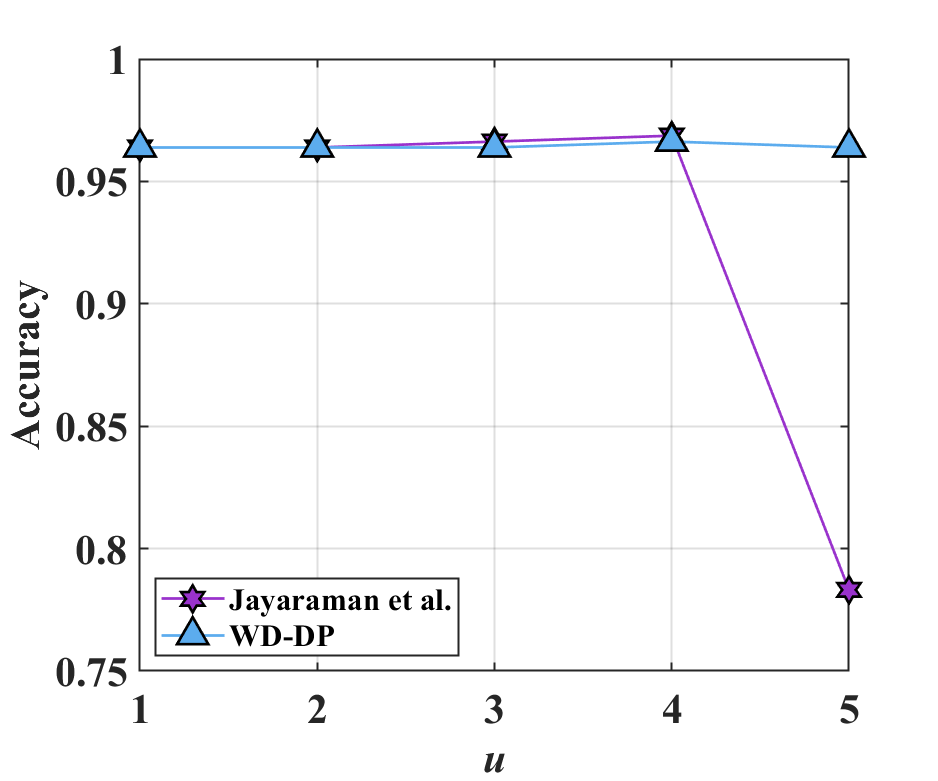}}
\subfigure[Credit Card Fraud]{\includegraphics[width=0.3\textwidth]{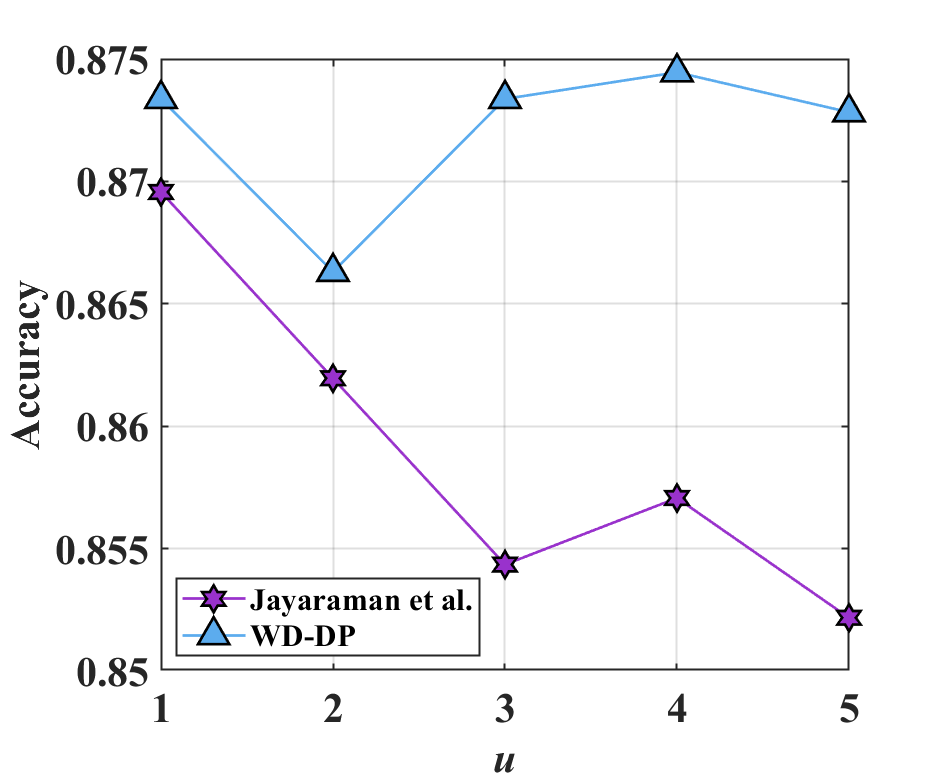}}
\caption{Accuracy on data sets over the level of non-average $u$, with $m=32$, $\epsilon=0.1$.}
\end{figure*}

\begin{figure*}[htb]
\centering
\subfigure[KDDCup99]{\includegraphics[width=0.3\textwidth]{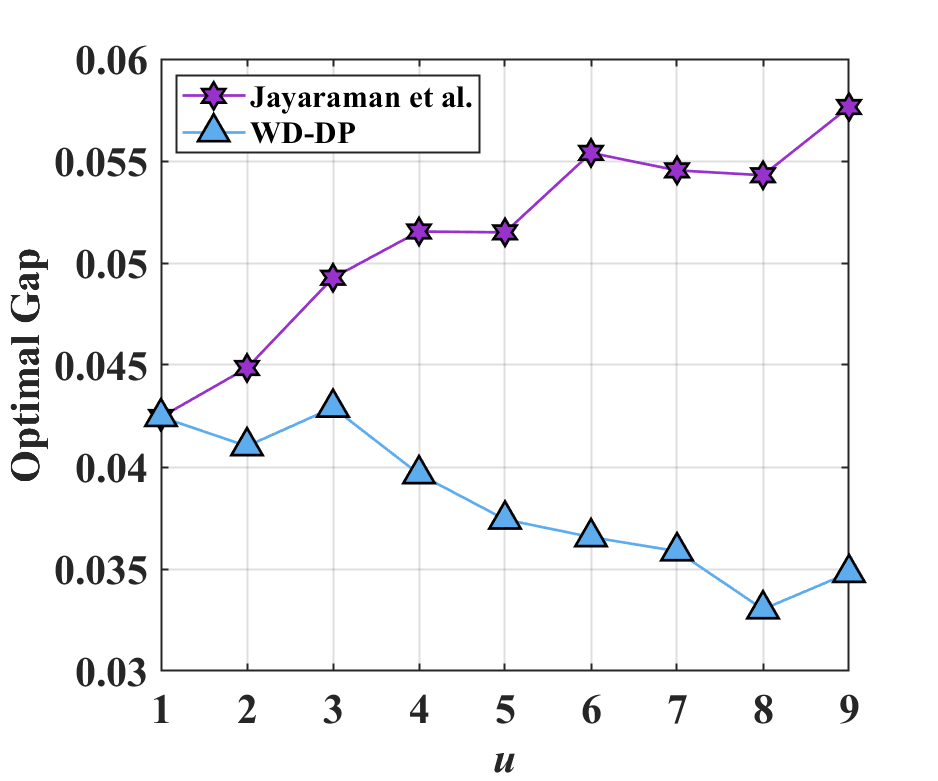}}
\subfigure[Adult]{\includegraphics[width=0.3\textwidth]{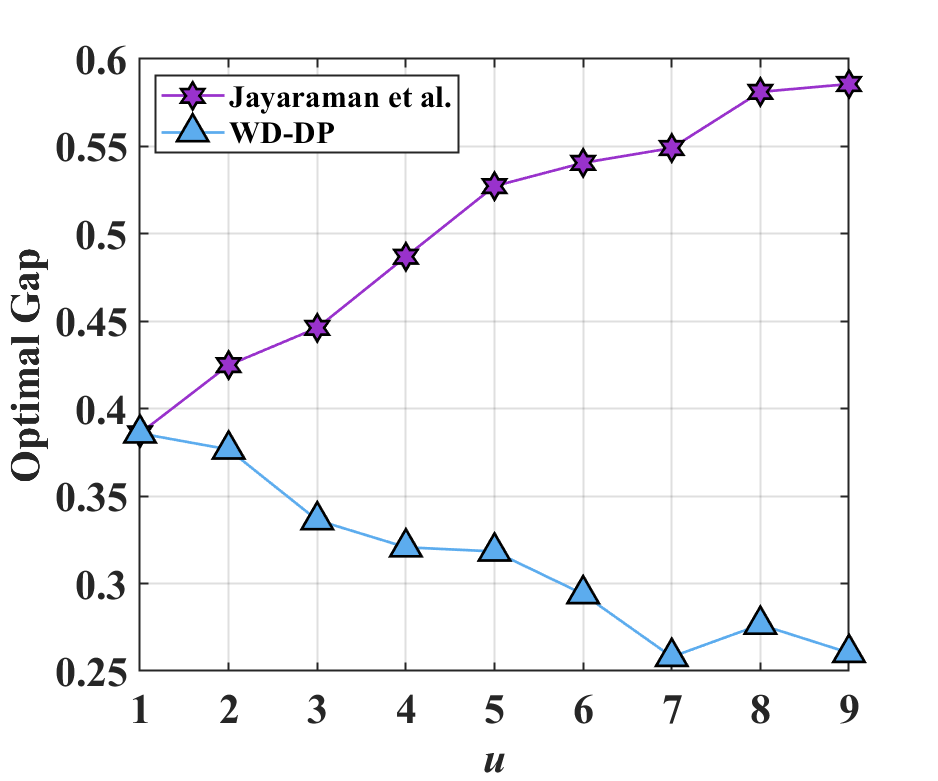}}
\subfigure[Bank]{\includegraphics[width=0.3\textwidth]{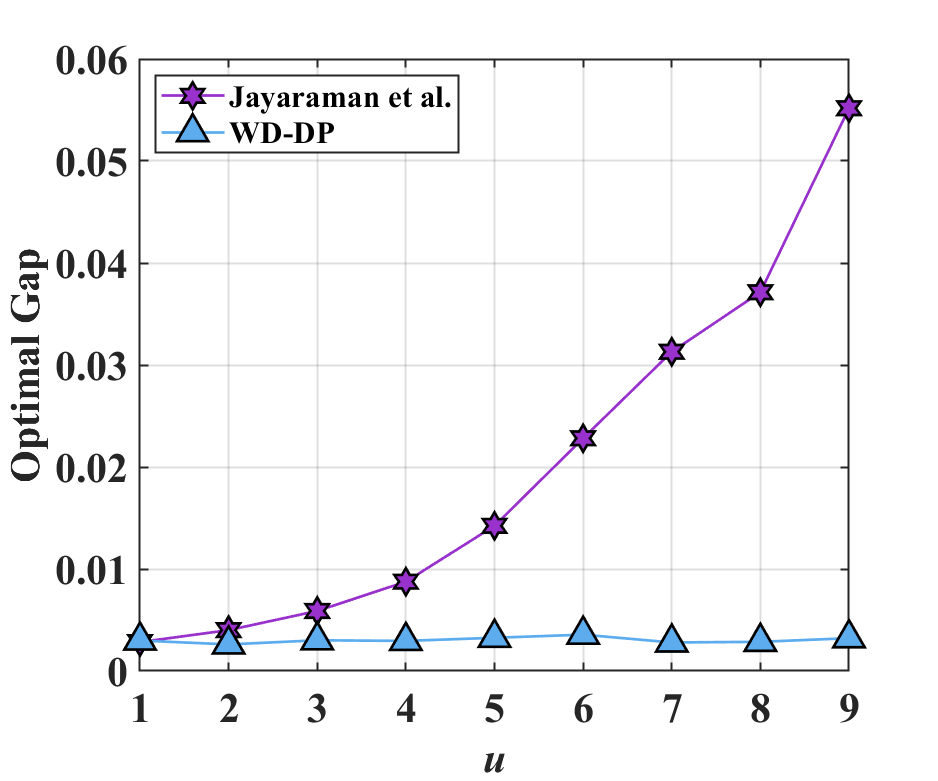}}
\subfigure[Breast Cancer]{\includegraphics[width=0.3\textwidth]{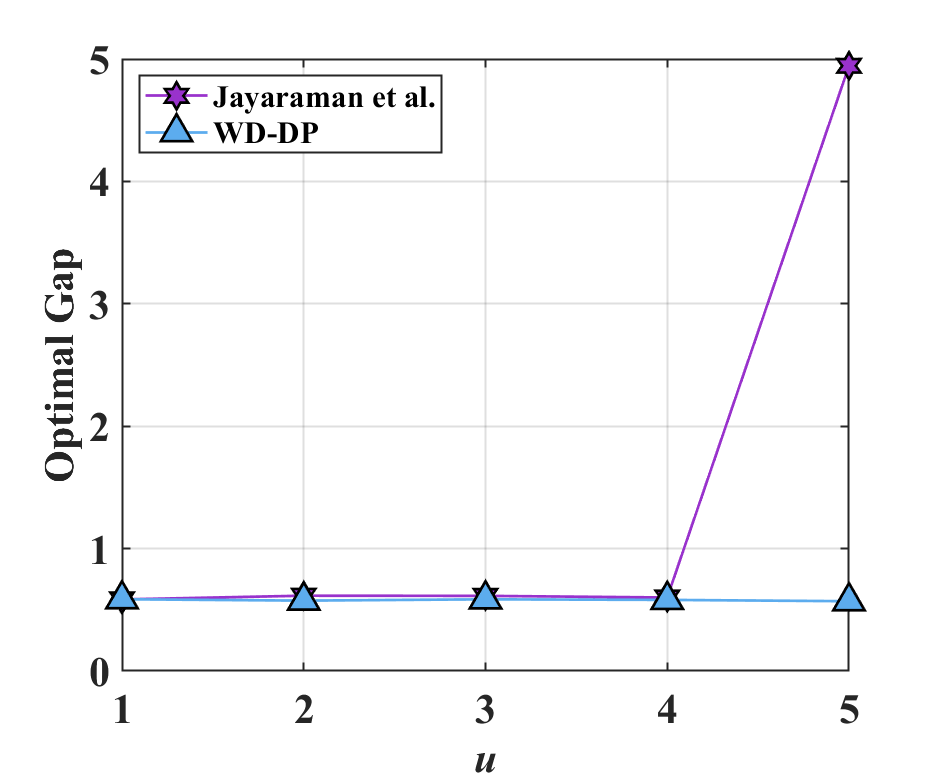}}
\subfigure[Credit Card Fraud]{\includegraphics[width=0.3\textwidth]{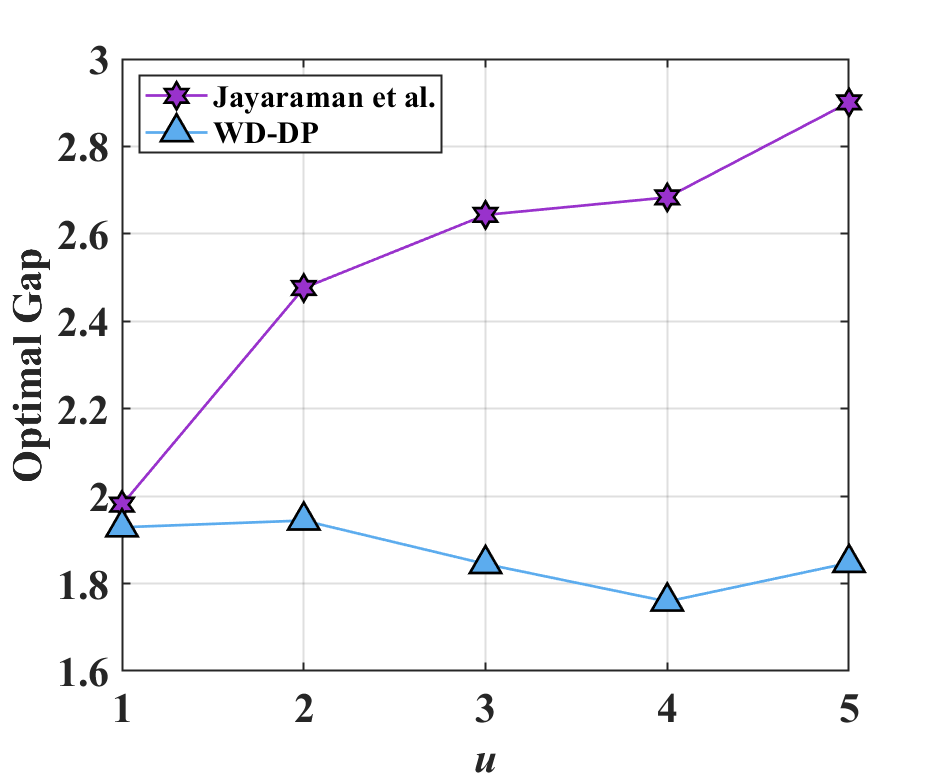}}
\caption{Optimal gap on data sets over the level of non-average $u$, with $m=32$, $\epsilon=0.1$.}
\end{figure*}

\begin{figure*}[htb]
\centering
\subfigure[KDDCup99]{\includegraphics[width=0.3\textwidth]{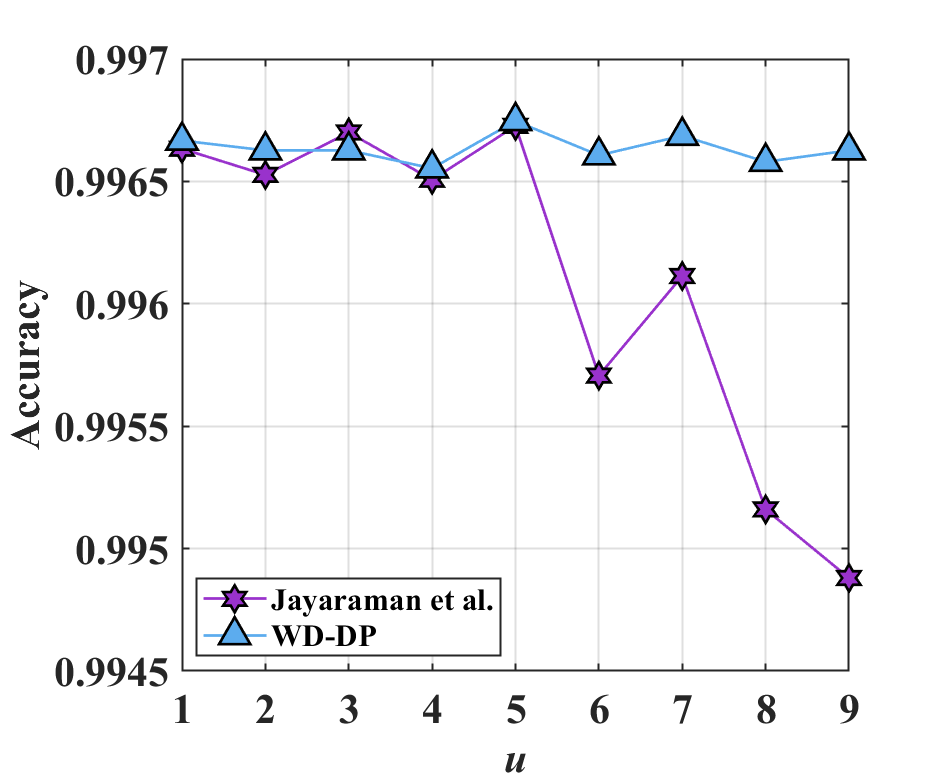}}
\subfigure[Adult]{\includegraphics[width=0.3\textwidth]{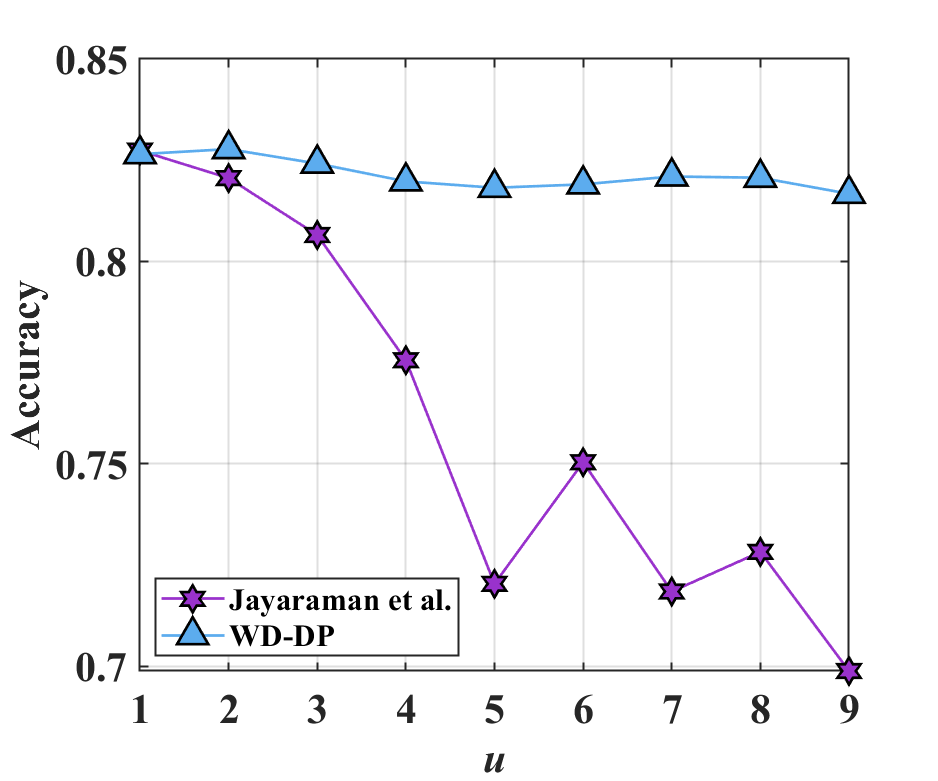}}
\subfigure[Bank]{\includegraphics[width=0.3\textwidth]{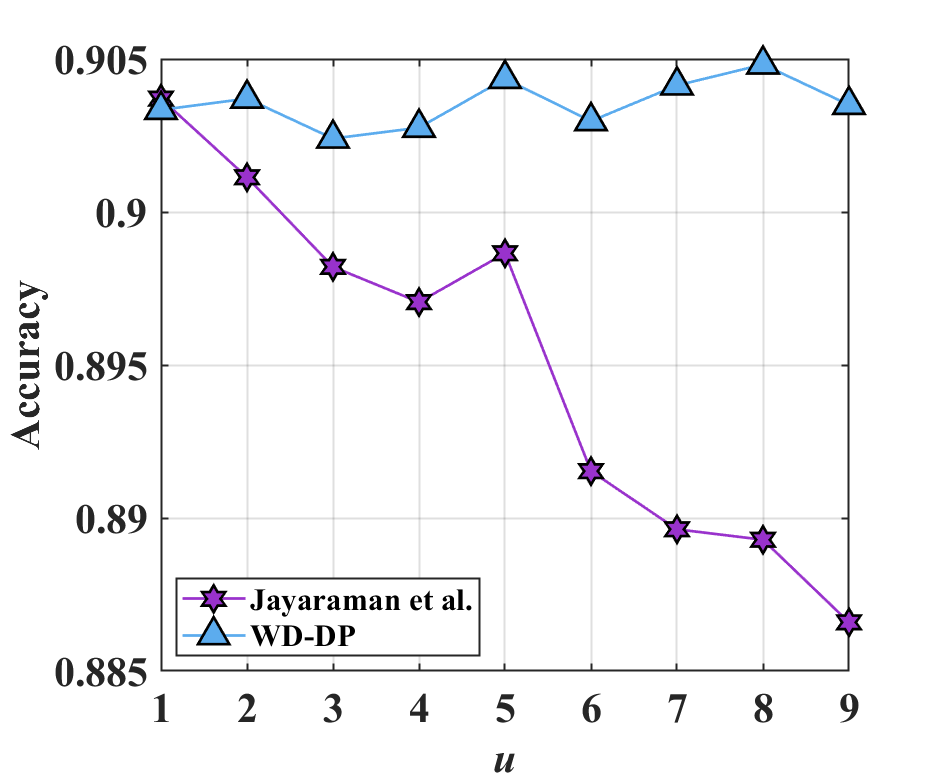}}
\subfigure[Breast Cancer]{\includegraphics[width=0.3\textwidth]{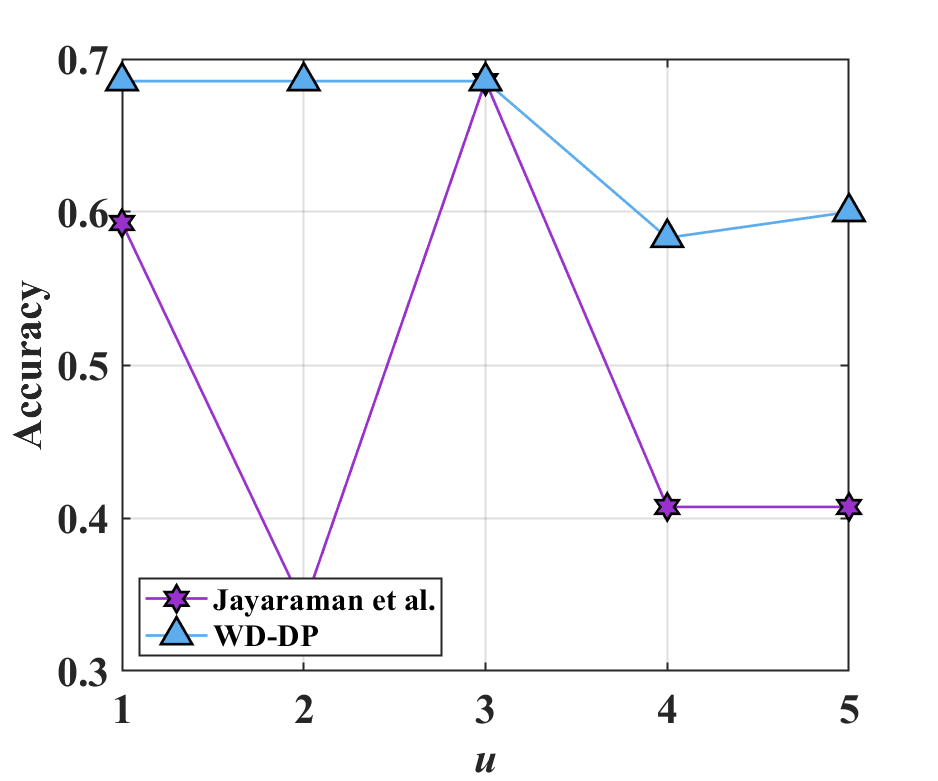}}
\subfigure[Credit Card Fraud]{\includegraphics[width=0.3\textwidth]{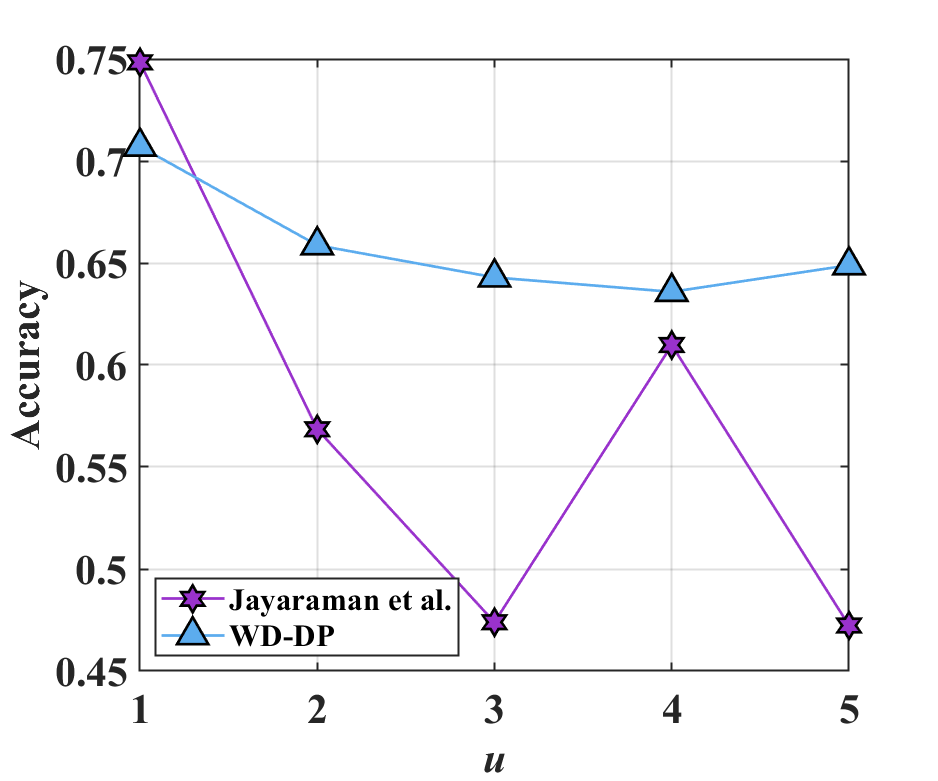}}
\caption{Accuracy on data sets over the level of non-average $u$, with $m=8$, $\epsilon=0.01$.}
\end{figure*}

\begin{figure*}[htb]
\centering
\subfigure[KDDCup99]{\includegraphics[width=0.3\textwidth]{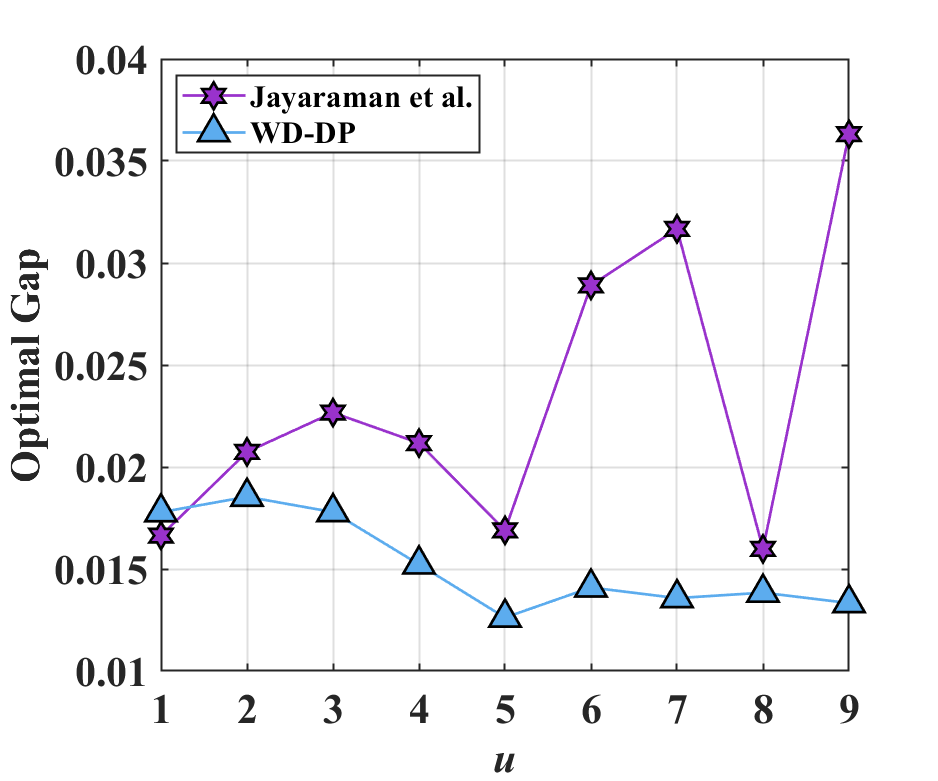}}
\subfigure[Adult]{\includegraphics[width=0.3\textwidth]{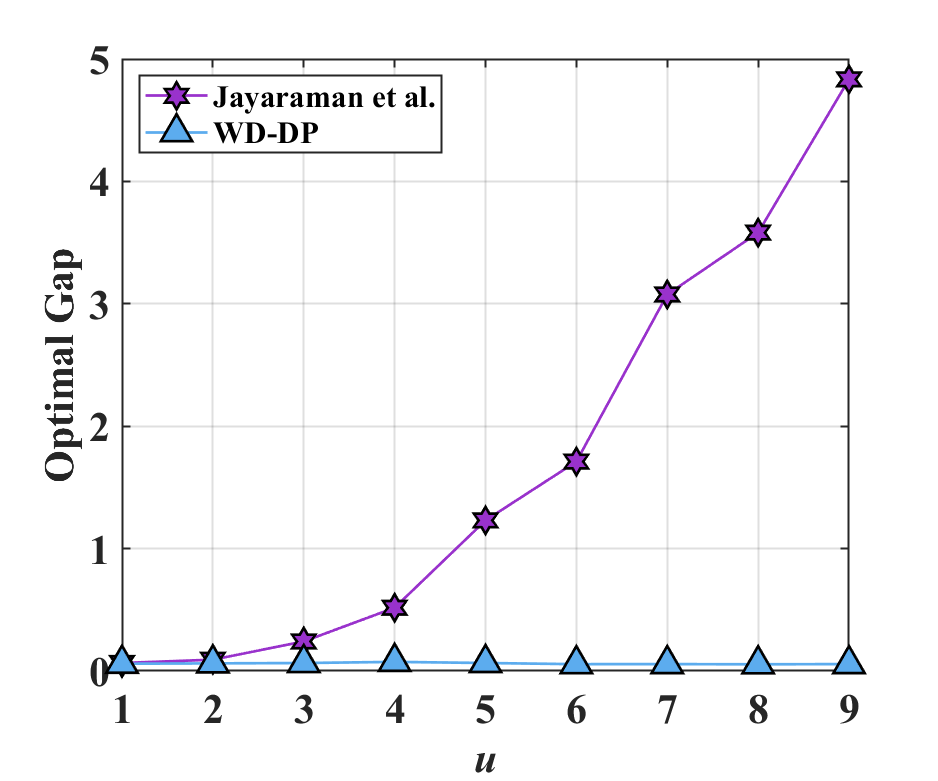}}
\subfigure[Bank]{\includegraphics[width=0.3\textwidth]{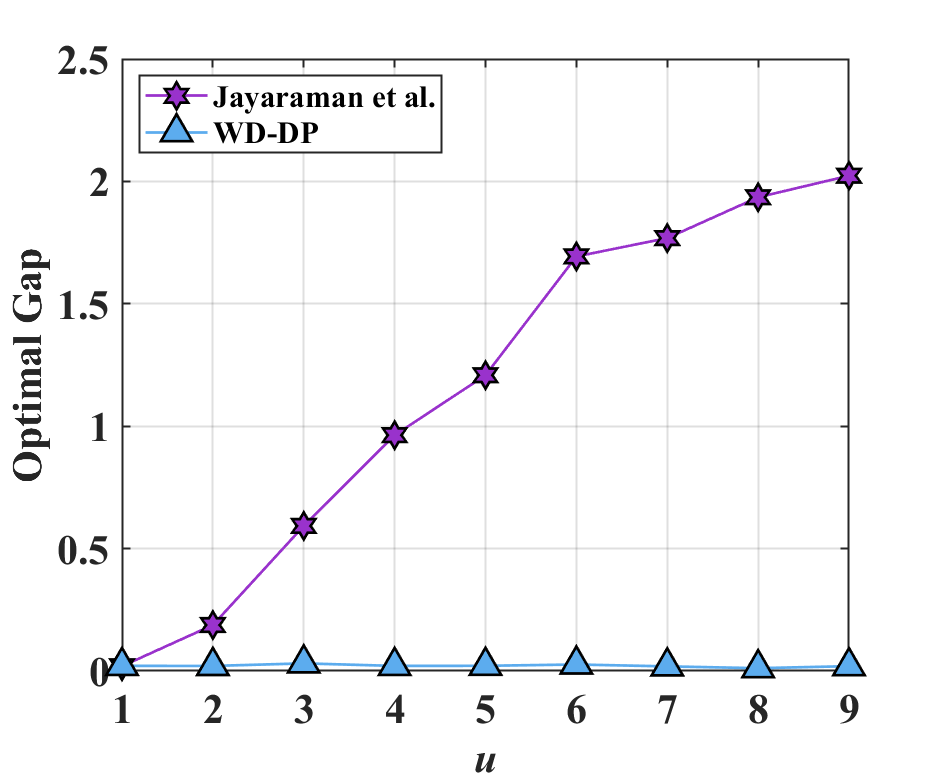}}
\subfigure[Breast Cancer]{\includegraphics[width=0.3\textwidth]{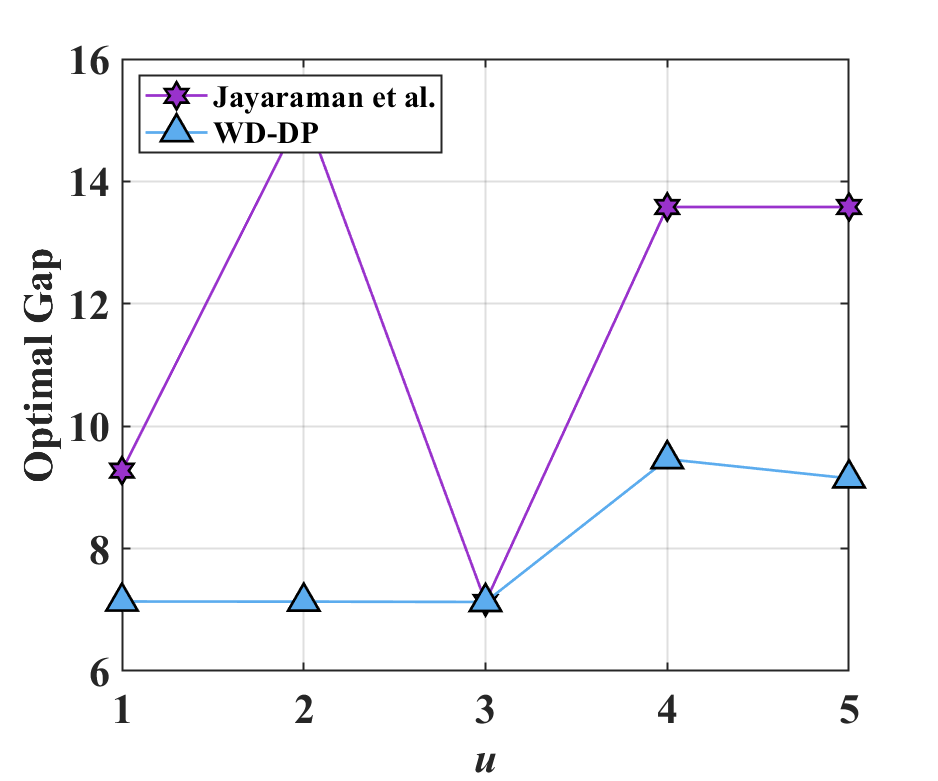}}
\subfigure[Credit Card Fraud]{\includegraphics[width=0.3\textwidth]{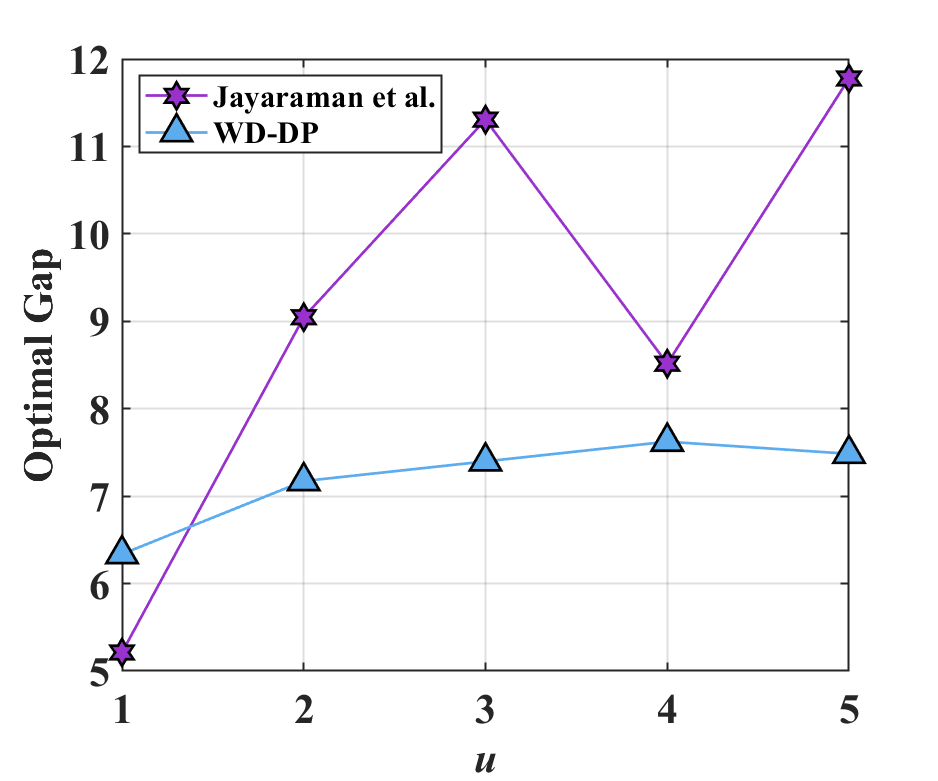}}
\caption{Optimal gap over the level of non-average $u$, with $m=8$, $\epsilon=0.01$.}
\end{figure*}

\end{appendix}

\end{document}